\documentclass[11pt]{article}
\usepackage{amsfonts,amsmath,amsthm,amssymb, epsfig}

\usepackage{color}

\usepackage[a4paper, total={6in, 9in}]{geometry}

\theoremstyle{plain}
\newtheorem{thm}{Theorem}[section]
\newtheorem{lem}[thm]{Lemma}
\newtheorem{prop}[thm]{Proposition}
\newtheorem{cor}[thm]{Corollary}

\theoremstyle{definition}
\newtheorem{dfn}{Definition}[section]

\newtheorem{algorithm}{Algorithm}[section]
\numberwithin{equation}{section}

\theoremstyle{remark}
\newtheorem{rem}{Remark}

\newcommand{\R}{\mathbb{R}}

\newcommand{\N}{\mathbb{N}}

\newcommand{\A}{\mathcal{A}}

\newcommand{\linspan}{\operatornamewithlimits{span}}

\newcommand{\change}[1]{{#1}}
\newcommand{\rechange}[1]{{\color{black}{#1}}}

\title{{Robust and Resource Efficient Identification\\of Shallow Neural Networks by Fewest  Samples}}
\author{
Massimo Fornasier\footnote{Department of Mathematics, TU M\"unchen, Boltzmannstr. 3, D-85748 Garching bei M\"unchen, Germany,
{\tt massimo.fornasier@ma.tum.de}; this author was supported by the DFG Grant FO767/6-2 ``Lern- und Wiederherstellungsalgorithmen fuer Multi-Sensor Datenfusion und spektrales Entmischen in der Erdbeobachtung''.},
Jan Vyb\'\i ral\footnote{Department of Mathematics FNSPE, Czech Technical University in Prague, Trojanova 13, 12000 Prague, Czech Republic,
{\tt jan.vybiral@fjfi.cvut.cz}; this author was supported by the grant P201/18/00580S of the Grant Agency of the Czech Republic,
by the Neuron Fund for Support of Science
and by the European Regional Development Fund-Project ``Center for Advanced Applied Science'' (No. CZ.02.1.01/0.0/0.0/16\_019/0000778)},
and Ingrid Daubechies\footnote{{Department of Mathematics, Duke University, NC, USA, {\tt ingrid.daubechies@duke.edu}}}
}

\begin{document}
\maketitle

\begin{abstract}

We address the structure identification and the uniform approximation of sums of ridge functions
$f(x)=\sum_{i=1}^m g_i(\langle a_i,x\rangle)$ on $\R^d$, representing a general form of a shallow feed-forward neural network,
from a small number of  query samples. 
Higher order differentiation, as used in our constructive approximations, of sums of ridge functions  or of their compositions, as in deeper neural network,
yields a natural connection between neural network weight identification and tensor product decomposition identification. In the case
of the shallowest feed-forward neural network, second order differentiation and tensors of order two (i.e., matrices) suffice as
we prove in this paper. 
We use two sampling schemes to perform approximate differentiation -  active sampling, where the sampling points are universal, actively, and randomly designed,
and  passive sampling, where  sampling points were preselected at random from a distribution
with known density. Based on multiple gathered approximated first and second order differentials, our general approximation strategy is developed as a sequence
of algorithms to perform individual sub-tasks. We first perform an active subspace search by approximating the span 
of the weight vectors $a_1,\dots,a_m$.  Then we use a straightforward substitution, which reduces the dimensionality
of the problem from $d$ to $m$. The core of the construction is then the stable and efficient approximation of weights expressed
in terms of rank-$1$ matrices $a_i \otimes a_i$, realized by formulating their individual identification as
a suitable nonlinear program. 
We prove the successful identification by this program of weight vectors
being close to orthonormal and we also show how 
we can constructively reduce to this case by a whitening procedure, without loss of any generality.
We finally discuss  the implementation and the performance of the proposed algorithmic pipeline with extensive numerical experiments,
which illustrate and confirm the theoretical results.
\end{abstract}


{\bf AMS subject classification (MSC 2010):}  82C32, 92B20, 65D15, 60B20

\tableofcontents


\section{Introduction and main results}

\subsection{Introduction}

In the last decade, deep neural networks (NN) outperformed other pattern recognition methods, 
achieving even superhuman skills in some domains \cite{Intro1, Intro3, Intro2}.
In the meanwhile, the success of NNs has been further confirmed in speech recognition \cite{Intro4}, optical character recognition \cite{Intro5},
games solution \cite{Intro7,Intro6} and many other areas. Unfortunately, training a neural network usually involves a non-convex optimization
and the process may get stuck at one of the many of its local minimizers. Furthermore, due to the huge number of parameters of multi-layer NNs
and the multitude of local minimizers, the performance of a neural network is often difficult to explain and interpret.
This black-box feature makes often NNs not the first-choice machine learning method in those areas,
where interpretability is a crucial issue (like security, cf. \cite{Intro11}) or for those applications where one wants to extract new insights from data \cite{Intro12}.

It is therefore of interest to know, which neural networks can be uniquely determined in a stable way by finitely many training points.
In fact, the unique identifiability is clearly a form of interpretability. The motivating  problem of this paper is the robust and resource efficient training of feed forward neural networks \cite{HOT06,HS06}.
Unfortunately, it is known that training a very simple (but general enough) neural network is indeed NP-hard \cite{BR92,Judd}.
Even without invoking fully connected neural networks, recent work \cite{FSV,maulvyXX} showed that even the training of one single
neuron (ridge function or single index model) can show any possible degree of intractability. Recent results \cite{ba17,Kaw16,memimo19,SoCa16,rova18},
on the other hand, are more encouraging, and  show that minimizing a square loss of a (deep) neural network does not have in general
or asymptotically (for large number of neurons)
poor local minima, although it may retain the presence of critical saddle points.

In this paper we present conditions for a shallow neural network to be provably and constructively identifiable with a number of samples,
which is polynomially depending  on the dimension of the network. Moreover, we prove that our procedure is robust to perturbations.
Our results hold with uniform approximation.
For the implementation we do not require high dimensional optimization methods and no concerns about complex energy loss landscapes
need to be addressed, but only classical and relatively simple calculus and linear algebra tools are used
(mostly function differentiation and singular value decompositions).

\change{The notation used throughout the paper is rather standard. For $0<p<\infty$, we denote by $\|x\|_p=\Bigl(\sum_{j=1}^d |x_j|^p\Bigr)^{1/p}$
the $p$-(quasi)-norm of a vector $x\in\R^d$. This notation is complemented by setting $\|x\|_\infty=\max_{j=1,\dots,d}|x_j|$. 
If $M\in\R^{m\times d}$ is an $m\times d$ matrix, 
we denote by $\|M\|_F$ 
the Frobenius norm and by $\|M\|$ 
the spectral norm of $M$. 
The inner product of two vectors $x,y\in\R^d$ is denoted by $\langle x,y\rangle=x^Ty.$
Their tensor product is a rank-1 matrix denoted by $x\otimes y=xy^T.$
For any subspace $A$ of vectors or matrices, we denote $P_A$ the orthogonal projection onto $A$.}
More specific notation is introduced along the way, when needed.

\subsection{Ridge functions and tensor decompositions}

We focus in particular on shallow feed forward neural networks
\begin{equation}\label{NNone}
\sum_{i=1}^m \alpha_i \sigma \left (\sum_{j=1}^m w_{i j} x_j + \theta_i \right),
\end{equation}
which constitute the main building blocks of deeper nets. To approach this problem, we study the more general task of the identification from minimal point queries of sums of ridge functions of the type
\begin{equation}\label{sumsridge}
\sum_{i=1}^m g_i(\langle a_i,x\rangle),\quad x\in\R^d,
\end{equation}
for some functions $g_i:\R\to\R$ and some non-zero vectors $a_i\in\R^d$.
We assume that the functions $g_i$ and the weights (or ridge directions)  $a_i$'s are both unknown. Ridge function approximation has been  extensively studied in mathematical statistics under the name of {\it projection pursuit},
see for instance \cite{dojo89,hu85,kli92}. The identification of sums of ridge functions has also been thoroughly considered in the approximation theory \cite{ca03,codadekepi12,co15,co14,deospe97,FSV,li02,li92,losh75,maulvyXX,pe99,pi97}, in particular we mention the work \cite{BP}, where  higher order differentiation was used to ``extract''  from the function $f$ and ``test'' its principal directions $a_i$'s against some given vectors $c_j$'s: 
\begin{equation*}
D^{\alpha_1}_{c_1}\dots D^{\alpha_k}_{c_k}f(x)=\sum_{i=1}^m g_i^{(\alpha_1+\dots+\alpha_k)}(\langle a_i,x\rangle)\langle a_i,c_1\rangle^{\alpha_1}\dots
\langle a_i,c_k\rangle^{\alpha_k},
\end{equation*}
where $k\in\N$, $c_i\in\R^d$, $\alpha_i\in\N$ for all $i=1,\dots,k$ and $D^{\alpha_i}_{c_i}$
is the $\alpha_i$-th derivative in the direction $c_i$. Hence, differentiation establishes a direct link between identification of the weights $a_i$'s and tensor decompositions \cite{HA12}. Interestingly, \cite{momo} shows that learning the weights of a simple neural network (which essentially coincides with \eqref{sumsridge})
is as hard as the problem of decomposition of a tensor built up from these weights.

In order to avoid instability due to numerical differentiation and active sampling (active choice of point queries),  ``weak differentiation'' approaches have been proposed. Let us describe the main ideas: Given an empirical sampling of points $x_k \sim \mu$ according to a probability distribution $\mu$, several methods, such as Stein's lemma and differentiation by parts with respect
to known density $p(x)$ of $\mu$ \cite{jasean,kli92}, have been considered to build from point queries empirical approximations to  tensors  corresponding to the expected value of higher order derivatives,
for instance 
\begin{eqnarray}
\Delta^k_N(f)=\frac{1}{N} \sum_{l=1}^N f(x_{l})(-1)^k  \frac{\nabla^k   p(x_{l})}{p(x_{l})} &\approx& \int_{\mathbb R^d}  f(x) (-1)^k \frac{\nabla^k p(x)}{p(x)} p(x) dx \nonumber \\
&=& \int_{\mathbb R^d} \nabla^k f(x) d \mu(x) =\mathbb E_{x\sim \mu} [\nabla^k f] \label{passivelearn} \\
&=& \sum_{i=1}^m  \left ( \int_{\mathbb R^d} g^{(k)}(\langle a_i,x\rangle) d \mu(x)\right ) \underbrace{a_i \otimes \dots \otimes a_i}_{\mbox{$k$-times}}.\nonumber 
\end{eqnarray}
In the case of second order tensors, i.e., $k=2$ this approach comes under the name of {\it principal Hessian directions} \cite{kli92}. This case is particularly relevant, because it deals with empirical approximation of matrices of the type
$$
\Delta^2_N(f)\approx \sum_{i=1}^m  \int_{\mathbb R^d} g''(\langle a_i,x\rangle) d \mu(x) a_i \otimes a_i.
$$
In case of orthogonal weights $a_i$'s and $m\leq d$, the identification of the weights is in principle solvable by spectral decomposition. However, this method leaves open the issue of dealing with non-orthogonal weights and the overdetermined case of $m > d$. In order to tackle both these issues the idea has been extended to third order tensors ($k=3$) and tensor decompositions. Using precisely the approximation \eqref{passivelearn}, in the recent paper \cite{jasean} the authors proposed and analyzed the algorithm NN-LIFT, which learns a two-layer feed-forward neural network, where the second layer has a linear activation function.
These results build upon the work \cite{angeja} where symmetric non-orthogonal tensor decompositions are shown to be tractably computable.  The non-orthogonal case is in fact addressed by reducing it via an orthogonalization procedure, called {\it whitening}, to the symmetric orthogonal tensor decomposition, which is known to be tractable \cite{kolda,rob14}.

The approaches based on the approximation \eqref{passivelearn}, e.g., both the principal Hessian directions \cite{kli92} and the recent one in \cite{jasean}
for third order tensors,  {may} suffer from a significant drawback: they are based on the decomposition of {\it one single instance matrix/tensor}, which is the empirical approximation to the expected value of (higher order) weak derivatives. In fact, it is well known that spectral and tensor decompositions are in general unstable processes, unless spectral gaps  and well-conditioning are guaranteed \cite{st90}. The error estimates appearing in \cite{kli92,jasean} look very similar, see, e.g., \cite[Theorem 4.1]{kli92} and \cite[Theorem 3, formula (12) or Lemma 9, Lemma 10]{jasean}, and contain inverse proportional terms with respect to eigenvalues or tensor coefficients {(with higher order power in the case of tensors)}, on which no control can be  provided, unless one assumes well-conditioning a priori.
In other words, if the one matrix/tensor at hand happens to have unstable decomposition, then one is simply left with bad luck. One may argue that this situation  may generically not occur, but no proof is provided so far.

Another drawback of \cite{kli92}, as mentioned above, is that the approach via principal Hessian directions cannot deal with non-orthogonal weights. In the paper \cite{jasean} the authors claim that, while a matrix decomposition is only identifiable up to orthogonal components, tensors can have identifiable non-orthogonal components and use this argument to motivate the necessity of the complexity and potential higher instability of third order tensors. Unfortunately, the low-rank tensor decomposition problem is usually not well-posed \cite{deli08}. (This is
based on the result that there are rank-$(r+1)$ tensors 
in the closure of rank-$r$ tensors.) 
Nonetheless, in certain regimes and under certain assumptions, also for $m$-rank tensors with $m >d$ and without assumptions of near-orthonormality, it is
uniquely solvable. A very helpful characterization of the regime where tensor decomposition is a
well-posed problem can be given in terms of the generic rank $r_{gen}(d,s)$: if 
$m \leq r_{gen}(d,s) := \frac{1}{d+1}\change{{d+s\choose s}}$ 
then with the exception of nongeneric cases, every $m$-rank $s$-tensor in $\mathbb R^d$ has a unique rank-$m$ decomposition (up to rescaling); moreover there are algorithms to find such a decomposition, mainly based on algebraic methods \cite{chci02,me06,oeot13}, whose stability under perturbation is presently not known.

In this paper we approach the problem of the weight identification by using a {robust} procedure, which does not suffer from the potential instabilities of being based on a single matrix/tensor instance as in \cite{jasean,kli92}. Moreover, we disprove the claim  that it is necessary to use higher order tensors in order to deal with non-orthogonal weights:
in fact, by developing an appropriate whitening procedure, we will exclusively build our identification procedure on matrices, making our approach  resource efficient and {potentially} more stable than tensor technology, which is in general more susceptible to intractability and instabilities \cite{deli08,hastad,hilim}. 
{Also in our error estimates, see, e.g., \eqref{eq:estnorm1}, \eqref{eq:estnorm2-1}, we require inverse proportional terms with respect to conditioning of the problem encoded by constants $\alpha, \alpha_2$, which we will introduce below; however, as we use a lower order differentiation and matrices (not tensors), the power magnitude of these terms is smaller than for higher order differentiations and tensors as in, e.g., \cite[Theorem 3, formula (12) or Lemma 9, Lemma 10]{jasean}, where such terms appear even at the sixth power.}
In this paper we focus on the case of $m\leq d$, \change{i.e., when the number of neurons $m$ is at most equal to the underlying dimension $d$}.
In \cite{FKRV} we are addressing the overdetermined case of $m>d$ and of two hidden layer\footnote{In view of a certain ambiguity in the literature, we clarify that, for two hidden layers, we mean here one more fully nonlinear layer with respect to \eqref{NNone}.} feed-forward neural networks. 

\change{
\subsection{Outline of the approach}

The aim of this paper is the structure identification and uniform approximation of sums of ridge functions
\begin{equation}\label{eq:sec2_sum}
f(x)=\sum_{i=1}^m g_i(\langle a_i,x\rangle),\quad x\in B_1^d=\{x\in\R^d: \|x\|_2\le 1\}.
\end{equation}
We assume throughout that the vectors $a_1,\dots,a_m\in\R^d$ are linearly independent and, therefore, $m\le d.$
Nevertheless, the typical setting we have in mind is that the number $d\gg 1$ of variables is very large
and the number $m$ of summands in \eqref{eq:sec2_sum} might be \rechange{much} smaller than $d$, i.e. $m\ll d.$

\rechange{Sections \ref{dimred}-\ref{sec:funct} address the identification of \eqref{eq:sec2_sum} under the assumption that $\{a_1,\dots,a_m\}$
are close to an orthonormal basis. In Section \ref{whitening} (see also Remark \ref{rem1}) we show how this assumption can be removed
without any loss of generality.
Our approach is based on the following fundamental steps,
which will be realized in a robust constructive/algorithmic way:}

\begin{itemize}
\item[1.] ({\bf Active subspace}) By using pointwise evaluations of the network $f$ we approximate (strong or weak) gradients, e.g., 
$$
\nabla f(x) = \sum^m_{i=1} g_i'(\langle a_i, x\rangle)a_i \in  A = \operatorname{span} \left\lbrace a_1, \dots, a_m \right\rbrace,
$$ 
at different points $x$, and we construct by  Algorithm \ref{alg1} an approximating space $\tilde A \approx A$, see Theorem \ref{thm3} or Theorem \ref{thm3-1} in Section \ref{dimred}.

\item[2. ] ({\bf Dimensionality reduction}) We recall that in this paper $m \leq d$. First of all, we show that we can reduce the problem to the case of $d=m$. Let us choose any orthonormal basis of $\tilde A$ and 
arrange it as the columns of a matrix $\tilde A \in \R^{d \times m}$ with some abuse of notation. Then
\[
f(x)  \approx 
f(P_{\tilde A} x) = f(\tilde A \tilde A^T x).
\]
We define the lower dimensional network
$$
\tilde f(y) := f(\tilde A y ) : \R^m \to \R,
$$
which has weights $\alpha_1 = \tilde A^T a_1, \dots, \alpha_m = \tilde A^T a_m$. 
Note that $\tilde A \alpha_i = P_{\tilde A} a_i \approx a_i$, and therefore $a_i$ can be approximately recovered from $\alpha_i$.
In summary, if the active subspace $A$ of $f$ is approximately known, then we can construct $\tilde f$, such that
the identification of $f$ and $\tilde f$ are equivalent. Hence, as we show in more details in Theorem \ref{initthm} of Section \ref{dimred2}, without loss of generality we can assume that $f$ maps $\R^m$ to $\R$ and that $a_i \in \R^m$.

\item [3. ] ({\bf Principal Hessian space}) While first order differentiation provides information about the active subspace
$A = \operatorname{span} \left\lbrace a_1, \dots, a_m \right\rbrace$, we need to query higher order derivatives in order
to access the individual weights $a_1, \dots, a_m$. Again by pointwise evaluations of the network $f$ we approximate
(strong or weak) Hessians and we construct by Algorithm \ref{alg3} an  approximating space $\widetilde {\mathcal A}$
of $ {\mathcal A}=\operatorname{span}\{a_i\otimes a_i, i=1,\dots,m\}\subset \R^{m\times m}$. The approximation results
are collected in Theorem \ref{thm:projapp} and Theorem \ref{thm:projapp2} in Section \ref{princhess}.

\item[4. ] ({\bf Individual weight recovery})
\rechange{Once the space $\widetilde {\mathcal A} \approx {\mathcal A}$ is constructed, the robust approximation of the weights $a_1, \dots, a_m$ is reduced
to the problem of identifying near rank-$1$ matrices in $\widetilde {\mathcal A}$.
In Section \ref{locmaxima} we solve the problem under the assumption that the weights $\{a_i:i=1,\dots,m\}$
are close to orthonormal.
In that case,}
the geometry of the space $\widetilde {\mathcal A} \approx {\mathcal A}=\operatorname{span}\{a_i\otimes a_i, i=1,\dots,m\}\subset \R^{m\times m}$ can be described by the following Euclidean representation:
\begin{figure}[h!]
  \centering
\includegraphics[width=0.4\textwidth]{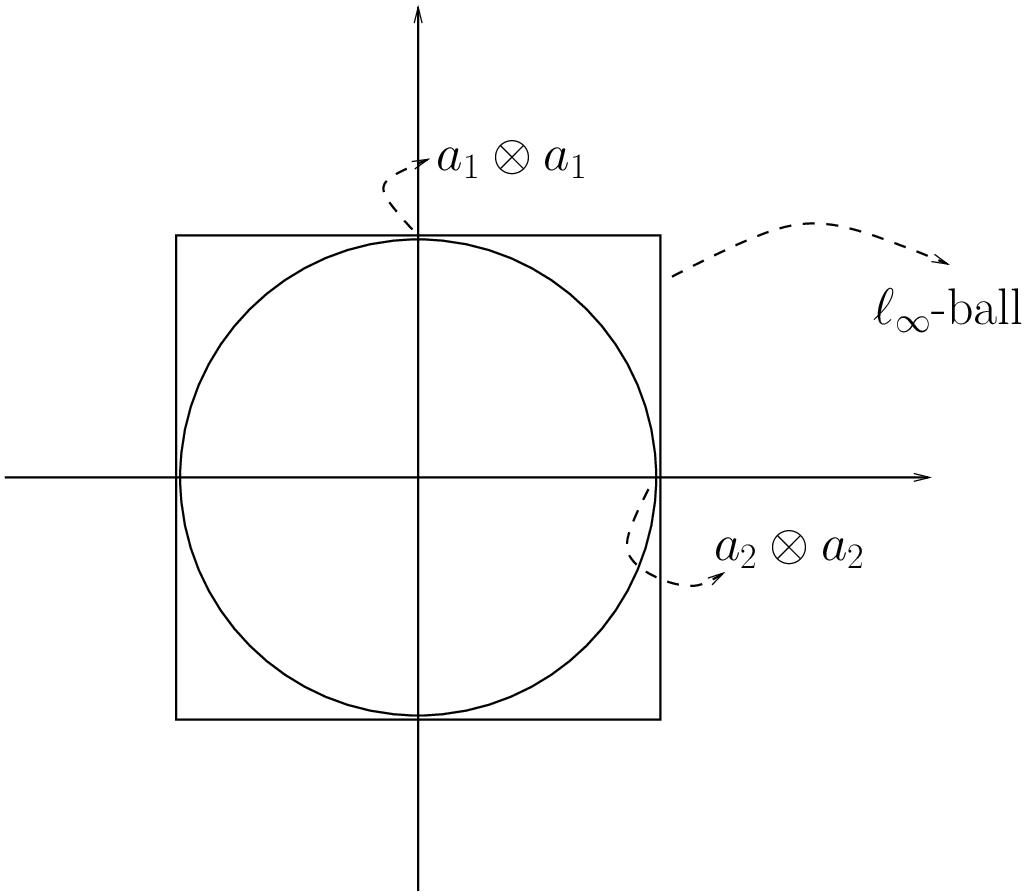}
\caption{If $a_i\otimes a_i$ were orthonormal then they would also be the extremal points of the matrix operator norm in the Frobenius ball of $\A$.}
\label{fig:00}
\end{figure}

\rechange{We therefore consider} the following robust nonconvex program
$$
{\rm arg\ max} \ \|M\|,\quad {\rm s.t.} \quad M\in\widetilde {\mathcal A}, \|M\|_F\le 1,
$$
searching for extremal matrices of the operator norm within the Frobenius ball, 
and we prove with Theorem \ref{thm:recoveridge} that its (local) maximizing solutions have first singular vectors $\hat a_i$ which are approximations to any one $a_i$ up to sign. Conversely, we also show in Proposition \ref{converse} that every $a_i \otimes a_i$ can be approximated by a local maximizer.
The solution of this minimization problem is approached by an iterative  gradient ascent-type  algorithm, Algorithm \ref{alg7},
with Theorem \ref{thm:recovery2} establishing its local convergence.

\item[5. ] ({\bf Identification of activation functions or ridge profiles}) Once the weights are recovered (up to sign) it is not hard to build approximating ridge profiles $\hat g_i$ and construct an approximation $\hat f$ of the full network. The construction is summarized as follows: Let $(\hat b_j)_{j=1}^m$ be the dual basis to $(\hat a_j)_{j=1}^m$, then one can define $\hat g_j(t):=f(t\hat b_j)$, $t\in (-1/\|\hat b_j\|_2,1/\|\hat b_j\|_2)$ and $\displaystyle\hat f(x):=\sum_{j=1}^m \hat g_j(\langle\hat a_j,x\rangle)$.
We show in the concluding Theorem \ref{identact} that $\hat f$ is a good uniform approximation to the original network.
\end{itemize}

\begin{rem}\label{rem1}({\bf Whitening}) As we have seen in the step 4, the problem of the identification of near rank-$1$ matrices in
$\widetilde {\mathcal A} \approx {\mathcal A}=\operatorname{span}\{a_i\otimes a_i, i=1,\dots,m\}\subset \R^{m\times m}$
is greatly simplified if one can assume that $\{a_1, \dots, a_m \}$
are nearly orthonormal vectors.
\rechange{If this condition is not fulfilled, we propose an additional step called \emph{whitening}, which may be evoked between
the steps 3 and 4 above. Indeed, we prove in Section \ref{whitening} that}, without loss of generality, we can always assume that the weights are nearly orthonormal.
For this sake, we consider any positive definite matrix $\widetilde G \in \widetilde {\mathcal A}$ of $\|\widetilde G\|_F=1$ with maximal smallest eigenvalue
and compute its singular value or spectral decomposition
$
\widetilde G= \widetilde  U \widetilde D  \widetilde U^T.$
If we now denote $ \widetilde W =  \widetilde D^{-\frac{1}{2}}  \widetilde U^T$ the so-called \emph{whitening matrix}, then 
we show in Theorem \ref{thm:review1} that the system of vectors $\{ \widetilde W a_i: i=1,\dots,m \}$ defines a near-orthogonal basis.
In view of the simple reformulation
$$
f(\widetilde W^T x)= \sum_{i=1}^m g_i(\langle a_i,  \widetilde W^T x\rangle)=\sum_{i=1}^m \tilde g_i(\langle \widetilde W a_i/\|\widetilde W a_i\|_2, x\rangle)=\tilde f(x),
$$
for $\tilde g_i(t) = g_i(\|\widetilde W a_i\|_2 t)$, we can further assume without loss of generality that the vectors
$\{a_i : i = 1, \dots, m\}$ are nearly orthonormal in first place.
\end{rem}

\subsection{Main result of the paper}

At this point, it is worth to summarize all the construction through the different algorithms and intermediate results outlined above in a single higher level result, which we specify for the case of {\it active sampling}, exclusively for the sake of simplicity. A similar statement would hold also for the case of {\it passive sampling}, which allows for more generic distributions.

We introduce two matrices. First we shall assume that the matrix 
\begin{equation}\label{eq:Jf_1}
J[f]:=\int_{{\mathbb S}^{d-1}}\nabla f(x) \nabla f(x)^T d \mu_{{\mathbb S}^{d-1}}(x)
\end{equation}
has maximal rank $m$, being $\mu_{{\mathbb S}^{d-1}}$\footnote{The use of the uniform measure on the sphere for active sampling is by no means a restriction,
more general distributions could be used with no significant difference in the results.} the uniform measure on the sphere ${\mathbb S}^{d-1}$. The matrix $J[f]$ ensures that sampling approximate gradients of the network $f$ uniformly on the sphere is gathering enough information for the active subspace $A = \operatorname{span}\{a_1,\dots,a_m\}$ to be fully, robustly, and constructively identifiable by simple PCA. 
Then we consider also the matrix 
$$
H_2[f] := \int_{\mathbb S^{m-1}} \operatorname{vec}(\nabla^2 f (x)) \otimes   \operatorname{vec}(\nabla^2 f(x)) d \mu_{\mathbb S^{m-1}}(x)
$$
of rank $m$ (the symbol $\operatorname{vec}(\nabla^2 f (x))$ stands for the vectorization of the Hessian of $f$, see \eqref{eq:vec'} for its precise definition). The matrix $H_2[f]$ also ensures that sampling approximate Hessians of the network $f$ uniformly on the sphere is gathering enough information for the active subspace $\mathcal A = \operatorname{span}\{a_1\otimes a_1 ,\dots,a_m \otimes a_m\}$ to be fully, robustly, and constructively identifiable by simple PCA.
}
\begin{thm}\label{mainresult}
Let $m\le d$ and let 
$f$ be a real-valued function defined on the neighborhood of $B_1^d$, which takes the form
$$
f(x) = \sum_{i=1}^m g_i(\langle a_i,x\rangle),
$$
where $g_i,i=1,\dots,m,$ are three times continuously differentiable on a neighborhood of $[-1,1]$
and $\{a_1,\dots,a_m\}\subset\R^d$ are linearly independent. \rechange{We additionally assume both $J[f]$ and $H_2[f]$ of maximal rank $m$ and well-conditioned.}
Let $\epsilon>0$. Then Algorithms \ref{alg1}-\ref{alg6}
use at most $m_{\mathcal X} [(d+1)+ (m+1)(m+2)/2]$ random exact point evaluations of $f$,
which correspond to numerical differentiation of $f$ with step-size $\epsilon$,
and they construct approximations $\{\hat a_1,\dots,\hat a_m\}$ of the ridge directions $\{a_1,\dots,a_m\}$
up to a sign change for which
\begin{equation}\label{mainres1}
\bigg ( \sum_{i=1}^m \|\hat a_i-a_i\|_2^2 \bigg )^{1/2} \lesssim \epsilon,
\end{equation}
with probability at least $1 - m \exp\Bigl(-\frac{m_{\mathcal X}C}{2 m^2}\Bigr)$ for suitable $C>0$.
Moreover, Algorithm \ref{alg8} constructs an approximating function $\hat f:B_1^d \to \mathbb R$  of the form
$$
\hat f(x) = \sum_{i=1}^m \hat g_i(\langle \hat a_i,x\rangle),
$$
such that
\begin{equation}\label{mainres2}
\| f- \hat f\|_{L_\infty(B^d_1)} \lesssim \epsilon.
\end{equation}
\rechange{The constant $C>0$ as well as the asymptotic constants in \eqref{mainres1} and \eqref{mainres2} may depend on the smoothness properties
of $g_i$ and the $m^{th}$ singular value of $J[f]$ and $H_2[f]$. Furthermore, the constants in \eqref{mainres1} and \eqref{mainres2}
may depend also polynomially on $m$.}
\end{thm}
\rechange{Let us remark, that Algorithms \ref{alg1}-\ref{alg6} realize the statement of Theorem \ref{mainresult}
only when the vectors $a_1,\dots,a_m$ are close to an orthonormal basis in the sense of Definition \ref{def:S}. If this is not the case,
it might be necessary to perform also the whitening step of Section \ref{whitening}.}

In absence of noise on the point evaluations of $f$ as in Theorem \ref{mainresult}, the usage of more point evaluations
\rechange{improves the accuracy in \eqref{mainres1} and \eqref{mainres2} only up to a limit posed by the numerical differentiation,
but it increases the probability of success arbitrarily close to one.}
The result would need to be significantly modified in case of noise on the {\it active} point evaluations of $f$ in order to deal with stability issues determined by employing finite differences in order to approximate the gradient and the Hessian of $f$. 
Contrary to most results available in the literature \cite{ba17,jasean,memimo19,rova18}, our final estimate \eqref{mainres2} holds in the uniform norm, which is deterministic once the weights are correctly identified. In fact, probabilistic least squares error estimates usually investigated in the literature may appear to successfully circumvent the curse of dimensionality, but they are obtained at the practically relevant price of not ensuring uniform error bounds. Not only we avoid the curse of dimensionality, but we also do not compromise on the uniform bound.

In the setting of passive sampling, we assume that the values $f(x_1),\dots,f(x_{m_{\mathcal X}})$ were sampled in points $x_1,\dots,x_{m_{\mathcal X}}$,
which are chosen independently with respect to a probability distribution $\mu$ on $B_1^d.$ Here, we assume that its density $p(x)$ is smooth and known in advance.
Using Stein's lemma \cite{Stein} or integration by parts in a way similar to \cite{jasean} and \cite{kli92}, we transfer our analysis also to the setting of passive sampling, leading first to the reduction
of dimension from $d$ to $m\le d$ and then to the approximation of ${\mathcal A}=\linspan\{a_1\otimes a_1,\dots,a_m\otimes a_m\}$.
Due to the non-local nature of the sampling process, it is rather easy to incorporate noise. Similarly to the active sampling,
our main tools are the matrix concentration inequalities \cite{ahwi02,ol10,ruve07,tr10}.

We conclude this introduction by mentioning that this paper, besides the specific application on identification of shallow neural networks, contains results
of independent mathematical and computational interest. First of all, we proved that stable whitening of matrices is constructively attainable,
see Section \ref{whitening}. This seems to disprove a common belief in the literature, see, e.g., \cite{jasean}, that the use of third or higher order tensors
can not be avoided.
Moreover, we introduced a new nonconvex optimization \eqref{eq:alg} with no spurious local minimizers,
which allows identifying robustly $1$-rank matrices in linear subspaces of symmetric matrices, see Section \ref{locmaxima} and Section \ref{finapprox}. 
We provide a simple and efficient iterative algorithm to perform such an optimization, Section \ref{simpalg}. 
Similar problems appeared recently in the literature and are of independent interest \cite{yusousXX,qusuwrXX}.



\section{Active subspace identification}\label{dimred}

As previously mentioned in the introduction, as soon as we can produce a subspace $\tilde A \subset \mathbb R^d$ approximating  $A=\operatorname{span}\{a_1,\dots,a_m\}$,
we can eventually reduce the problem of approximating a sum of ridge functions in $\R^d$ to the same problem in $\R^m$, preserving
even the near-orthonormality, cf. Remark \ref{rem:orth}.
In this section we describe two different methods of identification of $A$. {The first one applies to the setting of active sampling. It is
motivated by the results in \cite{FSV} and makes use of
first order differences. The second method is inspired by \cite{jasean}, and implements
passive sampling under the assumption that we dispose of the probability density $p(x)$ of the point distribution of the samples.
}

\subsection{Active sampling}\label{sec:idM}

We observe that the vector
\begin{equation}\label{eq:dif:1}
\nabla f(x)=\sum_{i=1}^m g_i'(\langle a_i,x\rangle) a_i
\end{equation}
lies in $A$ for every $x\in\R^d$. We consider \eqref{eq:dif:1}
for different $x_1,\dots,x_{m_{\mathcal X}}\in\R^d$, where $m_{\mathcal X}\ge m$.
In a generic situation for the points $x_i$'s, $A$ is likely given as the span of $\{\nabla f(x_1),\dots,\nabla f(x_{m_{\mathcal X}})\}$.

As we would like to use only function values of $f$ in our algorithms, we use for every $j=1,\dots,d$ and every $k=1,\dots,m_{\mathcal X}$ the Taylor's expansion
\begin{equation}\label{eq:tay1}
\frac{\partial}{\partial e_j}f(x_k)=\frac{f(x_k+\epsilon e_j)-f(x_k)}{\epsilon}-\Bigl[\frac{\partial}{\partial e_j}f(x_k+\eta_{j,k} e_j)-\frac{\partial}{\partial e_j}f(x_k)\Bigr]
\end{equation}
for some $\eta_{j,k}\in[0,\epsilon]$.
We recast the $d\times m_{\mathcal X}$ instances of \eqref{eq:tay1} into the matrix notation
\begin{equation}\label{eq:tayX}
X=Y-{\mathcal E},
\end{equation}
where
\begin{align}\label{eq:def:Y}
X_{j,k}&=\frac{\partial}{\partial e_j}f(x_k),\qquad Y_{j,k}=\frac{f(x_k+\epsilon e_j)-f(x_k)}{\epsilon},
\intertext{and}
\notag {\mathcal E}_{j,k}&=\frac{\partial}{\partial e_j}f(x_k+\eta_{j,k} e_j)-\frac{\partial}{\partial e_j}f(x_k)
\end{align}
for $j=1,\dots,d$ and $k=1,\dots,m_{\mathcal X}$. It follows from \eqref{eq:dif:1}, that $A$ is the linear span of columns of $X$. Naturally, we define $\tilde A$
using the linear span of the singular vectors of $Y$ corresponding to its $m$ largest singular values. This is formalized in the following algorithm.

\vskip.3cm
\fbox{
\begin{minipage}{13.6cm}
\begin{algorithm}\label{alg1}
\emph{\begin{itemize}
\item Construct $Y$ according to \eqref{eq:def:Y}.
\item Compute the singular value decomposition of
\begin{equation*}
Y^T = \left (\begin{array}{lll}
\tilde U_{1}&\tilde U_{2}\end{array}\right )
\left (\begin{array}{ll}\tilde \Sigma_{1}& 0\\
0& \tilde \Sigma_{2 }\\\end{array}\right )
\left (\begin{array}{l}\tilde V_{1 }^T\\\tilde V_{2}^T\end{array}
\right ),
\end{equation*}
where $\tilde \Sigma_{1}$ contains the $m$ largest singular values.
\item Set $\tilde A$ to be the row space of $\tilde V_1^T$.
\end{itemize}}
\end{algorithm}
\end{minipage}}
\vskip.3cm

The aim of the rest of this section is to show, that $\tilde A$ constructed in Algorithm \ref{alg1} is in some sense close to $A$.
To be more specific, we need to bound $\|P_A-P_{\tilde A}\|$, i.e. the operator or the Frobenius 
norm of the difference between the orthogonal projections onto $A$ and $\tilde A$, respectively.
For this first approximation method we need \rechange{the matrix $J[f]$, which was defined in \eqref{eq:Jf_1} as} 
\begin{equation*}
J[f]:=\int_{{\mathbb S}^{d-1}}\nabla f(x) \nabla f(x)^T d \mu_{{\mathbb S}^{d-1}}(x).
\end{equation*}
\change{In some sense, it captures the contribution of each term in \eqref{sumsridge}.
As we want to recover all ridge profiles, we assume that $J[f]$ \rechange{has the maximal rank $m$ and its non-zero singular values are bounded away from zero.}}

\begin{lem}\label{speclem} Assume the vectors $(a_i)_{i=1}^m$ linearly independent, and $\|a_i\|_2=1$ for all $i=1,\dots,m$. Additionally assume 
\begin{align*}
C_1&:=\max_{i=1,\dots,m}\, \max_{-1\le t\le 1} |g'_i(t)| < \infty.
\end{align*}
Suppose that $\sigma_m(J[f])\ge \alpha>0$, i.e., the $m^{th}$ singular value of the matrix $J[f]$ is bounded away from zero. Then for any $s\in (0, 1)$  we have that
\begin{equation}\label{lowerbnd2}
\sigma_m(X) \geq \sqrt{m_{\mathcal X} \alpha (1-s)}
\end{equation}
with probability at least $1 - m \exp\Bigl(-\frac{m_{\mathcal X}\alpha s^2  }{2 C_1^2 m^2}\Bigr)$, where $X$ is constructed as in \eqref{eq:tayX} for $x_1,\dots,x_{m_{\mathcal X}}\in{\mathbb S}^{d-1}$ drawn uniformly at random.
\end{lem}
\begin{proof} The result will follow by a suitable application of Theorem \ref{chernmat} in the Appendix.
\change{We choose  an orthonormal basis $w_1,\dots,w_m$ of $A$ and denote by $W$ a $d\times m$ matrix with columns $w_1,\dots,w_m.$}
We observe that $\sigma_j(X)=\sigma_j(W^TX)=\sqrt{\sigma_j(W^TX X^TW)}$,
$$
X X^T=\sum_{l=1}^{m_{\mathcal X}} \nabla f(x_l)\nabla f(x_l)^T
$$
and
$$
W^TX X^TW=\sum_{l=1}^{m_{\mathcal X}} W^T\nabla f(x_l)\nabla f(x_l)^T W.
$$
Furthermore, we obtain for every $x\in\R^d$
\begin{align}
\notag\sigma_1(W^T\nabla f(x)\nabla f(x)^TW)&=
\sigma_1(\nabla f(x)\nabla f(x)^T)= \|\nabla f(x)\nabla f(x)^T\|_F\\
\label{eq:est_rem}&=\|\nabla f(x)\|_2^2\rechange{=\left\|\sum_{i=1}^m g_i'(\langle a_i,x\rangle)a_i\right\|_2^2}\\
\notag & \rechange{\le \left(\sum_{i=1}^m |g_i'(\langle a_i,x\rangle)|\cdot\|a_i\|_2\right)^2\le C_1^2m^2.}
\end{align}
Hence $X_j=W^T\nabla f(x_j)\nabla f(x_j)^TW$ is a random $m\times m$ positive-semidefinite matrix, that is almost surely bounded.
Moreover,
$$
{\mathbb E}\,X_j=W^T\int_{{\mathbb S}^{d-1}}\nabla f(x) \nabla f(x)^T d \mu_{{\mathbb S}^{d-1}}(x)W=W^T J[f] W.
$$
We conclude that $\displaystyle \mu_{\min}=\mu_{\min}\Bigl(\sum_{j=1}^{m_{\mathcal X}}{\mathbb E}X_j\Bigr)\ge m_{\mathcal X}\alpha$, and  by Theorem \ref{chernmat} in the Appendix
$$
\sigma_m(X)=\sqrt{\sigma_m(W^TX X^TW)}\ge \sqrt{\mu_{\min}(1-s)}\ge \sqrt{m_{\mathcal X}\alpha(1-s)}
$$
with probability at least
\begin{align*}
1-m \exp\Bigl({-\frac{\mu_{\min}s^2}{2C_1^2m^2}}\Bigr)&\ge 1 - m \exp\Bigl(-\frac{m_{\mathcal X}\alpha s^2  }{2  C_1^2 m^2}\Bigr).\qedhere
\end{align*}
\end{proof}
\begin{rem}
If we further assume that $a_1,\dots,a_m$ are $\varepsilon$-nearly-orthonormal and $w_1,\dots,w_m$ are orthonormal vectors with
$$
{\mathcal S}(a_1,\dots,a_m)=\biggl(\sum_{i=1}^m\|a_i-w_i\|_2^2\biggr)^{1/2}\le\varepsilon,
$$
we can improve \eqref{eq:est_rem} to
\begin{align*}
\sigma_1(W^T\nabla f(x)\nabla f(x)^TW)&\le \|\nabla f(x)\|^2_2=\Bigl\|\sum_{i=1}^m g_i'(\langle a_i,x\rangle)a_i\Bigr\|^2_2\\
&\le \biggl(\Bigl\|\sum_{i=1}^m g_i'(\langle a_i,x\rangle)w_i\Bigr\|_2+\Bigl\|\sum_{i=1}^m g_i'(\langle a_i, x\rangle)(a_i-w_i)\Bigr\|_2\biggr)^{2}\\
&\le \biggl[\Bigl(\sum_{i=1}^m|g_i'(\langle a_i, x\rangle)|^2\Bigr)^{1/2}+\sum_{i=1}^m |g_i'(\langle a_i, x\rangle)|\cdot\|a_i-w_i\|_2\biggr]^2\\
&\le (1+\varepsilon)^2 \sum_{i=1}^m|g_i'(\langle a_i,x\rangle)|^2\le C_1^2(1+\varepsilon)^2m.
\end{align*}
The rest of the proof then follows in the same manner, only the probability changes to
$$
1 - m \exp\Bigl(-\frac{m_{\mathcal X}\alpha s^2  }{2  C_1^2 (1+\varepsilon)^2m}\Bigr).
$$
The same remark applies also to Theorem \ref{thm:projapp} below.
\end{rem}
The following theorem quantifies the distance between the subspace $\tilde A$ constructed in Algorithm \ref{alg1} and $A$.
\begin{thm}\label{thm3} Assume the vectors $(a_i)_{i=1}^m$ linearly independent, and $\|a_i\|_2=1$ for all $i=1,\dots,m$. Additionally assume that
\begin{align*}
C_1&:=\max_{i=1,\dots,m}\, \max_{-1\le t\le 1} |g'_i(t)| < \infty
\end{align*}
and that the Lipschitz constants of all $g_j',j=1,\dots,m$, are bounded by $C_2<\infty.$

Let $\tilde A$ be constructed as described in Algorithm \ref{alg1} by sampling $m_{\mathcal X} (d+1)$ values of $f$. Let $0<s<1$, and assume $\sigma_m(J[f])\ge \alpha>0.$
Then
$$ 
\|P_A-P_{\tilde A}\|_F \le \frac{2C_2\epsilon m}{\sqrt{\alpha(1-s)}-C_2\epsilon m} 
$$
with probability at least $1 - m \exp\Bigl(-\frac{m_{\mathcal X}\alpha s^2  }{2 m^2 C_1^2}\Bigr)$.
\end{thm}
\begin{proof}
We intend to apply the so-called \emph{Wedin's bound}, as recalled in Theorem \ref{wedin} in the Appendix, to estimate the distance between $A$ 
and $\tilde A$. 
If we choose $B=X^T$ and $\tilde B=Y^T$, we get $\Sigma_2=0$ and we observe that \eqref{separa1}
and \eqref{separa2} are satisfied with $\bar\alpha=\sigma_{m}(Y^T)$. Therefore, Theorem \ref{wedin} implies
\begin{align}
\notag\|P_A-P_{\tilde A}\|_F&= \|V_1 V^T_1-\tilde V_1\tilde V_1^T\|_F
\le \frac{2\|X-Y\|_F}{\sigma_{m}(Y^T)}\\
\label{eq:proj1}&\le \frac{2\|X-Y\|_F}{\sigma_{m}(X^T)-\|X-Y\|_F},
\end{align}
where we have used Weyl's inequality $|\sigma_{m}(X^T)-\sigma_m(Y^T)|\le \|X-Y\|_F$ in the last step.
To continue in \eqref{eq:proj1}, we have to estimate $\|X-Y\|_F$ and $\sigma_m(X^T)$.

We use the relation
$$
\Bigl|\frac{\partial}{\partial e_j}f(x_k+\eta_{j,k} e_j)-\frac{\partial}{\partial e_j}f(x_k)\Bigr|=
\Bigl|\sum_{i=1}^m[g'_i(\langle a_i,x_k+\eta_{j,k}e_j\rangle)-g'_i(\langle a_i,x_k\rangle)]a_{i,j}\Bigr|\le C_2\epsilon \sum_{i=1}^m a_{i,j}^2
$$
to obtain the estimate 
\begin{align}\label{eq:proj2}
\|X-Y\|_F&=\|{\mathcal E}\|_F\le C_2\epsilon \Bigl(\sum_{k=1}^{m_{\mathcal X}}\sum_{j=1}^d \Bigl(\sum_{i=1}^m a_{i,j}^2\Bigr)^2 \Bigr)^{1/2}\\
\notag&\le C_2\epsilon \sqrt{m_{\mathcal X}}\sum_{j=1}^d \sum_{i=1}^m a_{i,j}^2= C_2\epsilon\sqrt{m_{\mathcal X}}m.
\end{align}
The statement now follows by a combination of \eqref{eq:proj1} with \eqref{eq:proj2} and \eqref{lowerbnd2}.
\end{proof}
\begin{rem} The same argument as in the proof of Theorem \ref{thm3} allows to show that
$$
\sigma_m(Y)-\sigma_{m+1}(Y)\ge \sqrt{m_{\mathcal X}}(\sqrt{\alpha(1-s)}-2C_2\epsilon m)
$$
with the same probability as before. Hence, for $\epsilon$ small enough and $m_{\mathcal X}$ large,
there is (with high probability) a gap in the spectrum of $Y$ between $\sigma_m(Y)$ and $\sigma_{m+1}(Y)$.
This can be used to detect $m$ if it is unknown.
\end{rem}

{
\subsection{Passive sampling}
In the previous sections we investigated the identification of the subspace $A$ when the sample points of $f(x)$ can be actively chosen.
For that we used classical differentiation and Taylor's residuals and we assumed exact evaluations of the function.
In this section, we discuss the approximation of $A$ in the more realistic scenario where the distribution of the sampling points is known,
but not actively chosen, and the point evaluations are affected by noise.

As in \cite{jasean}, we assume that we are given a probability distribution $\mu$, whose density $p(x)$ is known or has been previously estimated from empirical data \cite{degy85}. For simplicity we assume $\operatorname{supp}(p) \subset B_1^d$.
We also assume that we are given a probability space $(\mathcal V, \pi)$ and a suitable collection of $\mathcal C_c^1$ functions $\varphi_\nu:\mathbb R^d \to \mathbb R$, for $\nu \in \mathcal V$, with the properties 
\begin{equation}\label{nuass}
\operatorname{supp}\varphi_\nu \subset B_1^d \mbox{ for all }  \nu \in \mathcal V, \quad  \max_{\nu \in \mathcal V} \max_{x \in B_1^d} \left \|\frac{\nabla \varphi_\nu(x)}{p(x)} \right \|_2 \leq C_{\mathcal V},
\end{equation}
and for which the matrix
\begin{equation}\label{Jnuass}
J_{\mathcal V}[f]= \int_{\mathcal V} \left ( \int_{\mathbb R^d} \nabla f(x) \varphi_{\nu}(x) dx \right) \left ( \int_{\mathbb R^d} \nabla f(x) \varphi_{\nu}(x) dx \right)^T d \pi(\nu)
\end{equation}
has \change{rank $m$}. 

\begin{rem}
\begin{itemize}
\item[(i)] {The probability space $(\mathcal V, \pi)$, the set of functions $\{\varphi_\nu: \nu \in \mathcal V\}$, and the full-rank condition for $J_{\mathcal V}[f]$
may appear abstract and a bit implicit at the first look. We clarify their role first in the most simple setting when $m=1$, $g(t)=t$
and ${\mathcal V}=\{0\}.$ Then $f(x)=g(\langle a, x\rangle)=\langle a, x\rangle$ and \eqref{Jnuass} becomes
$$
J_{\mathcal V}[f]= a a^T\cdot \left(\int_{\R^d}\varphi_0(x)dx\right)^2.
$$
It will turn out later (cf. Lemma \ref{speclem2} and Theorem \ref{thm3-1}), that we need to choose $\varphi_0:B_1^d\to \R$, such that $\alpha/C_{\mathcal V}^2$ is as large as possible,
where $\alpha=\sigma_1(J_{\mathcal V}[f])$ stands for the spectral norm of $J_{\mathcal V}[f]$. Finally, if $p(x)=\frac{1}{\omega_d}$ for every $x\in B_1^d$ with $\omega_d$ denoting the Lebesgue volume of $B_1^d$
and $\varphi_0$ is radial with $\varphi_0(1)=0$, we get
\begin{align*}
\frac{\alpha}{C_{\mathcal V}^2}&=\frac{\displaystyle \left(\int_{\R^d}\varphi_0(x)dx\right)^2}{\displaystyle \omega_d^2 \max_{x\in B_1^d}\|\nabla \varphi_0(x)\|^2_2}
=\frac{\displaystyle \left(\int_0^1 d\omega_d r^{d-1}\varphi_0(r)dr\right)^2}{\displaystyle \omega_d^2 \max_{0<r<1} |\varphi'_0(r)|^2}
\le d^2 \left(\int_0^1 \int_r^1 1 dsr^{d-1}dr\right)^2\\
&=d^2 \left(\int_0^1 (1-r) r^{d-1}dr\right)^2=d^2\left(\frac{1}{d}-\frac{1}{d+1}\right)^2=\frac{1}{(d+1)^2}.
\end{align*}
We observe that the conditions on $\{\varphi_\nu: \nu \in \mathcal V\}$ and $J_{\mathcal V}[f]$ may include an implicit dependence on $d$.
This is in accordance with the very well-known fact, that even the identification of one neuron (or one ridge function) can suffer the curse of dimension
if we do not pose any additional restrictions on its activation function or its weights, cf. \cite{BR92,FSV,maulvyXX}.}
\item[(ii)]  In fact, one may relate $(\mathcal V, \pi)$ and $\{\varphi_\nu: \nu \in \mathcal V\}$
directly to the density $p$ as follows.
We first consider a bounded resolution of the identity, i.e., a set of nonnegative smooth and compactly supported functions $\psi_\nu \geq 0$ such that
$\int_{\mathcal V} \psi_\nu (x) d\pi(\nu) \equiv 1$ for all $x \in B_1^d$ and
$\max_{\nu \in \mathcal V}\max_{x \in B_1^d} \{| \psi_\nu(x) |, \| \nabla \psi_\nu(x) \|_2 \}\leq C_\Psi$.
In case the set $\mathcal V$ is discrete, then $\{\psi_\nu: \nu \in \mathcal V\}$ is simply a classical bounded partition of the unity.
Additionally we pick yet another bounded and smooth function $q \geq 0$ such that $\max_{x \in B_1^d} \{\frac{ q(x)}{p(x)},  \| \frac{\nabla q(x)}{p(x)}\|_2 \}\leq C_q$.

Then, one can define 
$$
\varphi_\nu(x) = \psi_\nu(x) q(x),
$$ 
and it is not difficult to show that conditions \eqref{nuass} are fulfilled. In fact, for densities $p$ with bounded derivatives, e.g., Gaussian mixtures, one could choose for instance $q(x) =\frac{1}{2} p(x)^2$. In fact, in this case, $q(x)/p(x)= \frac{1}{2} p(x)$ and $\nabla q(x)/p(x) = \nabla p(x)$. Moreover,  the matrix
$$
J_{\mathcal V}[f]= \int_{\mathcal V} \left ( \int_{\mathbb R^d} \nabla f(x) \psi_\nu(x) q(x) dx \right) \left ( \int_{\mathbb R^d} \nabla f(x) \psi_\nu(x) q(x) dx \right)^T d \pi(\nu),
$$
would correspond to the superposition of ``weighted local evaluations'' of $\nabla f \otimes \nabla f$ over the supports of the functions $\psi_\nu$ to build a full-rank matrix.
\end{itemize}
\end{rem}

Now, differently from  \cite{jasean}, we consider the following empirical vectors
\begin{eqnarray}
Y_j=- \frac{1}{N} \sum_{k=1}^N (f(x_k) + n_k) \frac{\nabla   \varphi_{\nu_j}(x_k)}{p(x_k)} &\approx& - \int_{\mathbb R^d}  f(x) \frac{\nabla \varphi_{\nu_j}(x)}{p(x)} p(x) dx \nonumber \\
&=& \int_{\mathbb R^d} \nabla f(x) \varphi_{\nu_j}(x) dx \nonumber \\
&=& \sum_{i=1}^m  \left ( \int_{\mathbb R^d} g'(\langle a_i,x\rangle) \varphi_{\nu_j}(x) d x\right)  a_i, \label{passivelearn2}
\end{eqnarray}
generated at random by sampling i.i.d. $\nu_j \sim \pi$, $j=1,\dots,m_{\mathcal X}$.
Here, the random variables $n_k$ model the noise in the evaluation of the function $f$ in the point $x_k$
and we will assume that $n_k$ are independent bounded centered random variables, i.e., 
\begin{equation}\label{noise}
|n_k|\leq C_{\mathcal N} \mbox{ with probability } 1, \mbox{ and } \mathbb E[n_k]=0.
\end{equation}
The assumption that the noise is bounded can be relaxed to unbounded noise with thin tails (for instance sub-Gaussian noise) at the cost of adding in Theorem \ref{thm3-1} below a negative term to the probability in the statement, which accounts for the probability that the noise realizations are in fact bounded.
We define the matrix $Y_{\mathcal V} \in \mathbb R^{d\times m_{\mathcal X}}$, whose columns are $Y_j$, for $j=1,\dots,m_{\mathcal X}$. We similarly denote $X_{\mathcal V}\in \mathbb R^{d\times m_{\mathcal X}}$ the matrix with columns $X_j =\int_{\mathbb R^d} \nabla f(x) \varphi_{\nu_j}(x) dx$. \change{With similar arguments as Lemma \ref{speclem} and Theorem \ref{thm3} we can show the following result, whose proof is postponed to the Appendix.

\begin{thm}\label{thm3-1} Assume the vectors $(a_i)_{i=1}^m$ linearly independent, and $\|a_i\|_2=1$ for all $i=1,\dots,m$. Additionally assume that
\begin{align*}
C_\ell&:=\max_{i=1,\dots,m}\, \max_{-1\le t\le 1} |g^{(\ell)}_i(t)| < \infty, \quad \ell=0,1.
\end{align*}
Let $\tilde A$ be constructed as described in Algorithm \ref{alg1} by substituting there $Y$ with $Y_{\mathcal V}$, built by sampling $N$ values of $f$ as in \eqref{passivelearn2}. Let $0<s<1$,
and assume $\sigma_m(J_{\mathcal V}[f])\ge \alpha>0$.
Then, \change{for every $0<\eta<\sqrt{\alpha(1-s)}$},
\begin{equation}\label{eq:estnorm1}
\|P_A-P_{\tilde A}\|_F \le \frac{2 \eta}{\sqrt{\alpha(1-s)}-\eta} 
\end{equation}
with probability at least $1 - \change{\exp\Bigl(-\frac{\eta^2N}{8(2mQ)^2}+\frac{1}{4}\Bigr)}
- m \exp\Bigl(-\frac{m_{\mathcal X}\alpha s^2  }{2 (mQ)^2}\Bigr)$,
where $Q= (C_0 +C_{\mathcal N}/m)C_\mathcal{V}$.
As a consequence, for any $\varepsilon>0$ and $\delta>0$,
\begin{equation}\label{simplest1}
\|P_A-P_{\tilde A}\|_F \le \varepsilon
\end{equation}
with probability at least $1- \delta$ as soon as the number $N$ of sampling values of $f$ fulfills
\begin{equation}\label{simplest2}
N\geq \frac{32(2+\varepsilon)^2(mQ)^2\ln (3/\delta)}{\varepsilon^2\alpha (1-s)}.
\end{equation}
\end{thm}
}

\section{Dimensionality reduction}\label{dimred2}

The main aim of this section is Theorem \ref{initthm}, which allows to reduce the general case of identification of a shallow network where the input dimension is larger than the number of neurons, $d\ge m$,
to the case where $d=m$, hence, with a potentially significant dimensionality reduction. Due to the typical range of parameters we have in mind, this step is crucial in reducing the complexity
of the approximation of \eqref{eq:sec2_sum}.

\change{
\begin{thm}[Reduction to $m$ dimensions]\label{initthm}
Let us consider a function 
\begin{equation}\label{eq:sum_ridge}
f(x) = \sum_{i=1}^m g_i(\langle a_i, x\rangle),\quad x \in B_1^d,
\end{equation}
for $m \leq d$ and we denote $A= \operatorname{span}\{a_1,\dots,a_m\}$.
Let us now fix a $m$-dimensional subspace $\tilde A \subset \mathbb R^d$, for which we choose an orthonormal basis $\{\tilde a_1,\dots,\tilde a_m\}$, so that 
$\tilde A= \operatorname{span}\{\tilde a_1,\dots,\tilde a_m\}$.\footnote{With a certain abuse of notation,  we often use in this paper the symbol $A$
also to denote the matrix whose columns are the vectors $\{a_1,\dots,a_m\}$.
Similarly, we arrange the vectors $\tilde a_i$'s as columns of matrix $\tilde A$.}
We denote by $P_A$ and $P_{\tilde A}=\tilde A {\tilde A}^T$ the orthogonal projections onto $A$ and $\tilde A$, respectively.
Then the function 
\begin{equation}\label{fundapprox0}
\tilde f(y) = \sum_{i=1}^m g_i(\langle \alpha_i, y\rangle),\quad y \in  B_1^m \subset \R^m,
\end{equation}
with $\alpha_i = {\tilde A}^T a_i$ satisfies 
for any other function $\hat f:\mathbb R^m \to \mathbb R$ the following estimate 
\begin{equation}\label{fundapprox1}
\| f- \hat f( {\tilde A}^T \cdot) \|_\infty \leq \|f\|_{\operatorname{Lip}} \|P_A - P_{\tilde A}\| + \| \tilde f - \hat f \|_\infty.
\end{equation}
Moreover, for any other set of vectors $\{\hat \alpha_1, \dots, \hat \alpha_m\} \subset \mathbb R^m$,
\begin{equation}\label{fundapprox2}
\| a_i - \tilde A \hat \alpha_i\|_2 \leq  \| P_A - P_{\tilde A} \| + \| \alpha_i - \hat \alpha_i\|_2.
\end{equation}
\end{thm}
\begin{proof}
Let $x\in B_1^d$. We have
\begin{align*}
|f(x)-\hat f({\tilde A}^Tx)|&\le |f(x)-\tilde f({\tilde A}^Tx)|+|\tilde f({\tilde A}^Tx)-\hat f({\tilde A}^Tx)|\\
&\le|f(x)-f({\tilde A}{\tilde A}^Tx)|+\|\tilde f-\hat f\|_\infty=|f(P_Ax)-f(P_{\tilde A}x)|+\|\tilde f-\hat f\|_\infty\\
& \le \|f\|_{\rm Lip}\cdot \|P_A x-P_{\tilde A}x\|_2+\|\tilde f-\hat f\|_\infty.
\end{align*}
If we take the supremum over $x\in B_1^d$, we get \eqref{fundapprox1}.

The proof of \eqref{fundapprox2} follows from
\begin{align*}
\| a_i - \tilde A \hat \alpha_i\|_2 &\leq\| a_i - \tilde A \alpha_i\|_2 + \| \tilde A(\alpha_i - \hat \alpha_i)\|_2
=  \| (P_A - P_{\tilde A}) a_i \|_2 + \|\tilde A( \alpha_i - \hat \alpha_i)\|_2\\
&\leq  \| P_A - P_{\tilde A} \| + \|\alpha_i - \hat \alpha_i\|_2.\qedhere
\end{align*}
\end{proof}

In view of Theorem \ref{initthm}, we start the identification of a sum of $m$ ridge functions \eqref{eq:sum_ridge} on $\R^d$
by first approximating the subspace $A=\operatorname{span}\{a_1,\dots,a_m\}$ by another subspace ${\tilde A}$, such that
the operator norm $\|P_A-P_{\tilde A}\|$ is small. Then we consider the function $\tilde f(y)=f(\tilde Ay)$,
which is a sum of $m$ ridge functions on $\R^m$ with ridge profiles $\alpha_1,\dots,\alpha_m$.
Naturally, we will not be able to recover them exactly and we will only obtain
some good approximation $\{\hat \alpha_1, \dots, \hat \alpha_m\} \subset \mathbb R^m$.
Then \eqref{fundapprox2} shows that the vectors ${\tilde A}\hat \alpha_i$
approximate well the original ridge profiles $a_i$. 
Finally, if $\hat f$ is a uniform approximation of $\tilde f$ on $B_1^m$, then \eqref{fundapprox1} implies that the function $\hat f({\tilde A}^Tx)$
is a uniform approximation of $f$ on $B_1^d$. 

Observe that the sampling of $\tilde f$ can be easily transferred to sampling of $f$
by $\tilde f(y)= f({\tilde A}y).$}

\section{Principal Hessian subspace}\label{princhess}
While first order differentiation provides information about the active subspace $A = \operatorname{span} \left\lbrace a_1, \dots, a_m \right\rbrace$, we need to query higher order derivatives in order to access the individual weights $a_1, \dots, a_m$. 
First of all we construct here an approximation $\widetilde \A$ to the space $\A=\operatorname{span}\{a_i \otimes a_i, i=1,\dots,m\}$.
{As in the previous sections we describe two different methods of identification of $\A$. {The first one is by active sampling  and makes use of
second order differences. The second one implements passive sampling under the assumption that we dispose of the probability density $p(x)$ of the point distribution.
}
 
 \subsection{Active sampling}
We start   by generating again } $m_{\mathcal X} \in \mathbb N$ points $x_l \sim \mu_{\mathbb S^{m-1}}$, $l=1,\dots, m_{\mathcal X}$ uniformly at  random   on the $m-1$ dimensional sphere (remind that now
we assume $m=d$), and we define
\begin{equation*}
(\Delta[f](x_l))_{j,k}=\frac{f(x_l+\epsilon(e_j+e_k))-f(x_l+\epsilon e_j)-f(x_l+\epsilon e_k)+f(x_l)}{\epsilon^2},\quad j,k=1,\dots,m.
\end{equation*}
As $\Delta[f](x)\sim {\nabla^2 f}(x)=\sum_{i=1}^m g_i''(\langle a_i, x\rangle)a_i\otimes a_i\in\A$, we define $\widetilde\A$ as the $m$-dimensional subspace
approximating the points $(\Delta[f](x_l))_{l=1}^{m_{\mathcal X}}$ in the least-square sense.
{For later use we define $Y$ the  $m^2\times m_{\mathcal X}$ matrix with columns $\operatorname{vec}(\Delta[f](x_l))$, i.e.,
\begin{equation}\label{constrY}
Y= (\operatorname{vec}(\Delta[f](x_1))| \dots | \operatorname{vec}(\Delta[f](x_{m_{\mathcal X}})).
\end{equation}
}
We show below that $\widetilde \A$ is indeed a good approximation to $\A$ by showing that the difference of the respective orthogonal projections
$\|P_\A - P_{\widetilde \A}\|_{F \to F}$ in the operator norm associated to the Frobenius norm of matrices  is small  with high probability,
as soon as $m_{\mathcal X}$ is large enough.

We need now to introduce some  notations to facilitate the presentation. We define the vectorization of a matrix $A=(a_{i,j})_{i,j} \in \mathbb R^{m \times m}$
as the column vector in $\mathbb R^{m^2}$
\begin{equation}\label{eq:vec'}
\operatorname{vec}(A)_k := a_{\lfloor \frac{k-1}{m} \rfloor +1, (k-1\mod m)+1}, \quad k=1,\dots,m^2.
\end{equation}
For two matrices $A,B \in  \mathbb R^{m \times m}$ we define their {\it vectorized tensor product} by
\begin{equation}
A \otimes_v B:= \operatorname{vec}(A) \otimes \operatorname{vec}(B) = \operatorname{vec}(A)\operatorname{vec}(B)^T. \label{eq:vec}
\end{equation}
(Note that such a product of matrices does coincide neither with the Hadamard product nor with the Kronecker product.)
Thanks to these definitions and notations we can introduce the matrix
$$
H_2[f] := \int_{\mathbb S^{m-1}} {\nabla^2 f}(x) \otimes_v  {\nabla^2 f}(x) d \mu_{\mathbb S^{m-1}}(x).
$$
This $m^2 \times m^2$ matrix plays exactly the same role as $J[f]$ in Section \ref{sec:idM}. 

\vskip.3cm
\fbox{
\begin{minipage}{13.6cm}
\begin{algorithm}\label{alg3}
\emph{\begin{itemize}
\item {Construct $Y$ as in \eqref{constrY}}.
\item Compute the singular value decomposition of
\begin{equation*}
Y^T = \left (\begin{array}{lll}
\widetilde U_{1}&\widetilde U_{2}\end{array}\right )
\left (\begin{array}{ll}\widetilde \Sigma_{1}& 0\\
0& \widetilde \Sigma_{2 }\\\end{array}\right )
\left (\begin{array}{l}\widetilde V_{1 }^T\\\widetilde V_{2}^T\end{array}
\right ),
\end{equation*}
where $\widetilde \Sigma_{1}$ contains the $m$ largest singular values.
\item Set $\widetilde A$ to be the space of matrices, whose vectorization lies in the row space of $\widetilde V_1^T$.
\end{itemize}
}
\end{algorithm}
\end{minipage}
}
\vskip.3cm

As we follow the same strategy as the one used in Section \ref{sec:idM} to approximate the space $A=\operatorname{span}\{a_i, i=1,\dots,m\}$,
we limit ourselves to reformulate it in the context of the vector space of matrices $\A$.
We start with a technical estimate, which is essentially based on Taylor's theorem. 
\begin{lem}\label{lem6}
Assume the vectors $(a_i)_{i=1}^m$ satisfy $\|a_i\|_2=1$ for all $i=1,\dots,m$ and assume
that $g_j, j=1,\dots,m$, are two times differentiable with the Lipschitz constant of all $g_j'',j=1,\dots,m$ bounded from above by $C_3>0.$
Then, for all $x \in \mathbb{S}^{m-1}$,
$$
\|{\nabla^2 f} (x) - \Delta[f](x) \|_F \leq 2C_3 m \epsilon.
$$
\end{lem}
\begin{proof}
Let $g(t)=f(x+te_j+\epsilon e_k)-f(x+te_j)$, where $0\le t\le\epsilon$. Then by the mean value theorem
\begin{align*}
(\Delta[f](x))_{j,k}&=\frac{g(\epsilon)-g(0)}{\epsilon^2}=\frac{g'(\xi_1)}{\epsilon}
=\frac{\frac{\partial f}{\partial x_j}(x+\xi_1e_j+\epsilon e_k)-\frac{\partial f}{\partial x_j}(x+\xi_1e_j)}{\epsilon}\\
&=\frac{\partial^2 f}{\partial x_k\partial x_j}(x+\xi_1 e_j+\xi_2e_k),
\end{align*}
where $0<\xi_1,\xi_2<\epsilon$. Therefore
\begin{align*}
|(\nabla^2 f(x))_{j,k}-(\Delta[f](x))_{j,k}|&=\Bigl|\frac{\partial^2f}{\partial x_k\partial x_j}(x)-
\frac{\partial^2 f}{\partial x_k\partial x_j}(x+\xi_1 e_j+\xi_2e_k)\Bigr|\\
&\le\sum_{l=1}^m \bigl|g_l''(\langle a_l,x\rangle)- g''_l(\langle a_l,x+\xi_1e_j+\xi_2e_k\rangle)\bigr|\cdot|a_{l,j}|\cdot|a_{l,k}|\\
&\le C_3 \epsilon\sum_{l=1}^m |a_{l,j}|\cdot|a_{l,k}|\cdot(|a_{l,j}|+|a_{l,k}|).
\end{align*}
Using triangle inequality and $\|a_j\|_4\le\|a_j\|_2=1$, we estimate
\begin{align*}
\|{\nabla^2 f}(x)-\Delta[f](x)\|_F &\le 2C_3 \epsilon  \left[\sum_{j,k=1}^m \Bigl(\sum_{i=1}^m a_{i,j}^2 |a_{i,k}|\Bigr)^2\right]^{1/2}
\le 2C_3 \epsilon \sum_{i=1}^m \Bigl(\sum_{j,k=1}^m a_{i,j}^4 a_{i,k}^2\Bigr)^{1/2}\\
&=2C_3 \epsilon \sum_{i=1}^m \biggl[\Bigl(\sum_{j=1}^m a_{i,j}^4\Bigr)\Bigl( \sum_{k=1}^m a_{i,k}^2\Bigr)\biggr]^{1/2}
\le 2C_3 \epsilon m.
\end{align*}
\end{proof}

\begin{thm} \label{thm:projapp}Assume the vectors $(a_i)_{i=1}^m$ linearly independent, and $\|a_i\|_2=1$ for all $i=1,\dots,m$. Additionally assume 
\begin{align*}
C_j&:=\max_{i=1,\dots,m} \max_{-1\le t\le 1} |g^{(j)}_i(t)| < \infty,\quad j=0,1,2.
\end{align*}
 Let $\widetilde \A$ be constructed as described in Algorithm \ref{alg3} by sampling $m_{\mathcal X} \left [ (m+1)(m+2)/2 \right ]$ values of $f$. Let $0<s<1$, and assume $\sigma_m(H_2[f])\ge \alpha_2>0$, i.e., the $m^{th}$ singular value of the matrix $H_2[f]$ is bounded away from zero. 
Then
$$ 
\|P_\A - P_{\widetilde \A}\|_{F \to F} \leq \frac{4 C_3 m \epsilon}{\sqrt{\alpha_2 (1-s)} - 2C_3 m \epsilon}
$$
with probability at least $1- m \exp\left (-\frac{m_{\mathcal X} \alpha_2 s^2}{2 m^2 C_2^2} \right ).$ In particular ${\rm{dim}}(\A)={\rm{dim}}(\widetilde \A)=m$.
\end{thm}
\begin{proof}
We define the matrices $X,Y$ whose columns are given by $\operatorname{vec}({\nabla^2 f}(x_j))$, $j=1,\dots, m_{\mathcal X}$ and 
$\operatorname{vec}(\Delta[f](x_j))$, $j=1,\dots, m_{\mathcal X}$ respectively, namely
\begin{eqnarray*}
X&=&(\operatorname{vec}({\nabla^2 f}(x_1))| \dots | \operatorname{vec}({\nabla^2 f}(x_{m_\mathcal{X}}))), \quad Y= (\operatorname{vec}(\Delta[f](x_1))| \dots | \operatorname{vec} (\Delta[f](x_{m_\mathcal{X}}))).
\end{eqnarray*}
Notice that these matrices have dimension $m^2 \times m_\mathcal{X}$. As done in \eqref{eq:proj1} and by assuming for the moment that $\sigma_m(X) \neq 0$ (but obviously $\sigma_{m+1}(X)=0$ because the ${\nabla^2 f}(x_i)$'s lie all in the 
$m$-dimensional space $\A$), we deduce the  estimate
\begin{equation}
\label{eq:projX}
\|P_\A- P_{\widetilde \A}\|_{F \to F} \leq \frac{2 \|X-Y\|_F}{\sigma_m(X) -\|X-Y\|_F}, 
\end{equation}
as an application of Wedin's bound, Theorem \ref{wedin} in the Appendix. 
From Lemma \ref{lem6} we easily deduce
\begin{equation}\label{pippo}
\|X-Y\|_F = \left ( \sum_{j=1}^{m_{\mathcal X}} \|{\nabla^2 f} (x_j) - \Delta[f](x_j) \|_F^2 \right)^{1/2} \leq 2C_3 \epsilon m\sqrt{m_{\mathcal X}}.
\end{equation}\
In order to apply \eqref{eq:projX} we need finally to estimate $\sigma_m(X)$ from below and we shall do it by using again the 
Chernoff's bound for matrices Theorem \ref{chernmat}. 

Given an orthonormal basis $\{B_1, \dots, B_m\}$ for $\A$ we define the projector from $\mathbb R^{m^2} \to \mathbb R^m$ given by
$P^{\A} v =(\operatorname{vec}(B_1)^T v, \dots, \operatorname{vec}(B_m)^T v)$ for any $v \in \mathbb R^{m^2}$. We additionally define
with some abuse of notation
$$
P^{\A} X := (P^{\A} (\operatorname{vec}({\nabla^2 f}(x_1)))| \dots | P^{\A}( \operatorname{vec}({\nabla^2 f}(x_{m_\mathcal{X}})))).
$$
Notice that now this matrix has dimension $m \times m_{\mathcal X}$. Thanks to the fact that $P^{\A}$ is an orthogonal transformation,
we obtain the following equivalences
$$
\sigma_m(X) = \sqrt{\sigma_m((P^{\A} X)(P^{\A} X)^T)}.
$$
Hence to estimate $\sigma_m(X)$, it is sufficient to do it for $\sigma_m((P^{\A} X)(P^{\A} X)^T)$, whose argument is explicitly expressed as a sum 
$$
(P^{\A} X)(P^{\A} X)^T = \sum_{j=1}^{m_{\mathcal X}} X_j,
$$
where
$$
X_j=P^{\A} \operatorname{vec}({\nabla^2 f}(x_j)) \otimes\operatorname{vec}({\nabla^2 f}(x_j)) (P^{\A})^T.
$$
We wish to apply Theorem \ref{chernmat} for the sequence of positive semidefinite matrices $X_1, \dots, X_{m_{\mathcal X}}$. We notice first
that
$$
\mathbb E X_j = P^{\A} H_2^f (P^{\A})^T,
$$
and therefore 
\begin{equation}
\label{chern1}
\mu_{\min} \left( \mathbb E X_j  \right) \geq m_{\mathcal X} \alpha_2. 
\end{equation}
Additionally, for every $x \in \mathbb S^{m-1}$
\begin{eqnarray}
\nonumber \sigma_1(P^{\A} \operatorname{vec}({\nabla^2 f}(x)) \otimes \operatorname{vec}({\nabla^2 f}(x)) (P^{\A})^T) &=& \sigma_1( \operatorname{vec}({\nabla^2 f}(x)) \otimes \operatorname{vec}({\nabla^2 f}(x)))\\
\nonumber &\leq& \| {\nabla^2 f}(x_j)\|_F^2 \\
\label{chern2}&=& \Bigl\| \sum_{i=1}^m g''(\langle a_i,x\rangle) a_i \otimes a_i \Bigr\|_F^2  \\
\nonumber &\leq& C_2^2 m^2. 
\end{eqnarray}
An application of Theorem \ref{chernmat} under conditions \eqref{chern1} and \eqref{chern2}  yields
\begin{equation}\label{pippo2}
\sigma_m(X) \geq \sqrt{m_{\mathcal X} \alpha_2 (1-s)},
\end{equation}
with probability 
$$
1- m \exp\left (-\frac{m_{\mathcal X} \alpha_2 s^2}{2 m^2 C_2^2} \right ).
$$
We conclude from \eqref{pippo} and \eqref{pippo2} that, with the same probability,
\begin{equation*} 
\|P_{\A} - P_{\widetilde \A}\|_{F \to F} \leq \frac{4 C_3 m \epsilon}{\sqrt{\alpha_2 (1-s)} - 2C_3 m \epsilon}.\qedhere
\end{equation*}
\end{proof}

{
\subsection{Passive sampling}
We again assume that we are given a probability space $(\mathcal V, \pi)$ and a suitable collection of $\mathcal C_c^2$ functions
$\varphi_\nu:\mathbb R^d \to \mathbb R$, for $\nu \in \mathcal V$, with the properties
\begin{equation*}
\operatorname{supp}\varphi_\nu \subset B_1^d \mbox{ for all }  \nu \in \mathcal V, \quad
\max_{\nu \in \mathcal V} \max \left\{ \int_{\mathbb R^d} |\varphi_{\nu}(x)|dx, \max_{x \in B_1^d} \left \|\frac{\nabla^2 \varphi_\nu(x)}{p(x)} \right \|_F\right\} \leq C_{\mathcal V,2},
\end{equation*}
where in the latter bound we consider the \change{Frobeniuns} norm. Furthermore, we also assume that the matrix
$$
H_{\mathcal V}[f]= \int_{\mathcal V} \left ( \int_{\mathbb R^d} \nabla^2 f(x) \varphi_{\nu}(x) dx \right) \otimes_v \left ( \int_{\mathbb R^d} \nabla^2 f(x) \varphi_{\nu}(x) dx \right)^T d \pi(\nu)
$$
has full rank.
We consider the following empirical vectors
\begin{eqnarray}
Y_j= \operatorname{vec} \left ( \frac{1}{N} \sum_{k=1}^N (f(x_k)+n_k) \frac{\nabla^2   \varphi_{\nu_j}(x_k)}{p(x_k)}  \right )&\approx&  \operatorname{vec} \left ( \int_{\mathbb R^d}  f(x) \frac{\nabla^2 \varphi_{\nu_j}(x)}{p(x)} p(x) dx  \right )\nonumber \\
&=& \operatorname{vec} \left (\int_{\mathbb R^d} \nabla^2 f(x) \varphi_{\nu_j}(x) dx\right ) \label{passivelearn3}\\
&=& \sum_{i=1}^m  \operatorname{vec} \left ( \left ( \int_{\mathbb R^d} g''(\langle a_i,x\rangle) \varphi_{\nu_j}(x) d x\right)  a_i\otimes a_i \right ), \nonumber
\end{eqnarray}
generated at random by sampling i.i.d. $\nu_j \sim \pi$, $j=1,\dots,m_{\mathcal X}$, for $n_k$ independent random bounded and centered noise fulfilling \eqref{noise}.
We define the matrix $Y_{\mathcal V,2} \in \mathbb R^{m^2 \times m_{\mathcal X}}$, whose columns are $Y_j$, for $j=1,\dots,m_{\mathcal X}$. We similarly denote $X_{\mathcal V,2}\in \mathbb R^{m^2 \times m_{\mathcal X}}$ the matrix with columns $X_j =\operatorname{vec} \left ( \int_{\mathbb R^d} \nabla^2 f(x) \varphi_{\nu_j}(x) dx\right ) $.

\change{The proof of the following result resembles the proofs of Theorem \ref{thm:projapp} and Theorem \ref{thm3-1} and is postponed to the Appendix.} 

\begin{thm} \label{thm:projapp2}Assume the vectors $(a_i)_{i=1}^m$ linearly independent, and $\|a_i\|_2=1$ for all $i=1,\dots,m$. Additionally assume 
\begin{align*}
C_j&:=\max_{i=1,\dots,m} \max_{-1\le t\le 1} |g^{(j)}_i(t)| < \infty,\quad j=0,1,2.
\end{align*}
 Let $\widetilde \A$ be constructed as described in Algorithm \ref{alg3} by substituting there $Y$ with $Y_{\mathcal V,2}$, built by sampling $N$ values of $f$ as in \eqref{passivelearn3}. Let $0<s<1$, and assume $\sigma_m(H_{\mathcal V}[f])\ge \alpha_2>0$.
Then, \change{for $0<\eta<\sqrt{\alpha_2(1-s)}$},
\begin{equation}\label{eq:estnorm2-1}
\|P_\A-P_{\widetilde \A}\|_{F\to F} \le \frac{2 \eta}{\sqrt{\alpha_2(1-s)}-\eta} 
\end{equation}
with probability at least $1 - \change{\exp\left(-\frac{\eta^2N}{8(2mQ)^2}+\frac{1}{4}\right)} - m \exp\Bigl(-\frac{m_{\mathcal X}\alpha_2 s^2  }{2 (mQ)^2}\Bigr)$,
for $Q= (\max\{ C_0,C_2\} +C_{\mathcal N}/m)C_{\mathcal{V},2}$.
As a consequence, for any $\varepsilon>0$ and $\delta>0$,
\begin{equation}\label{simplest1-1}
\|P_\A-P_{\widetilde \A}\|_{F\to F}\le \varepsilon,
\end{equation}
with probability at least $1- \delta$ as soon as the number $N$ of sampling values of $f$ fulfills
\change{
\begin{equation}\label{simplest2-1}
N \geq \frac{32(2+\varepsilon)^2(mQ)^2\ln(3/\delta)}{\varepsilon^2 \alpha_2(1-s)}.
\end{equation}
}
\end{thm}


\section{\rechange{Near orthonormality}}

As we shall see in Section \ref{locmaxima} and as pointed already in the Introduction,
the problem of the identification of near rank-$1$ matrices in $\widetilde {\mathcal A}
\approx {\mathcal A}=\operatorname{span}\{a_i\otimes a_i, i=1,\dots,m\}\subset \R^{m\times m}$
is greatly simplified if one can assume that $\{a_1, \dots, a_m \}$ are nearly orthonormal vectors.
\rechange{In this section, we introduce the concept of near-orthonormality and its basic properties.}


\begin{dfn}
We define
\begin{equation}\label{def:S}
{\mathcal S}(\alpha_1,\dots,\alpha_m)=\inf\Bigl\{\Bigl(\sum_{i=1}^m\|\alpha_i-w_i\|_2^2\Bigr)^{1/2}:w_1,\dots,w_m\quad\text{orthonormal basis in}\ \R^m\Bigr\}
\end{equation}
for every set $\{\alpha_1,\dots,\alpha_m\}\subset\R^m$. We say that unit vectors $a_1,\dots,a_m\in \R^m$ are $\varepsilon$-nearly-orthonormal, if ${\mathcal S}(a_1,\dots,a_m)\leq \varepsilon$ for $\varepsilon>0$ relatively small.
\end{dfn}

\begin{thm}\label{thm:A1} \begin{enumerate}
\item[(i)] Let $a_1,\dots,a_m\in\R^m$ and let $A\in\R^{m\times m}$ be a matrix with columns $a_1,\dots, a_m$. Then
\begin{equation*}
{\mathcal S}(a_1,\dots,a_m)=\Bigl(\sum_{i=1}^m(\sigma_i-1)^2\Bigr)^{1/2},
\end{equation*}
where $\sigma_1\ge\sigma_2\ge\dots\ge0$ are the singular values of $A$.
\item[(ii)] Furthermore,
$$
{\mathcal S}(a_1,\dots,a_m)\le\|AA^T-I_m\|_F\le (\|A\|+1){\mathcal S}(a_1,\dots,a_m).
$$
\end{enumerate}
\end{thm}
\begin{proof}\begin{enumerate}
\item[(i)] The result is very well known and the proof follows easily by singular value decomposition of $A=U\Sigma V^T.$
The closest orthogonal basis $w_1,\dots,w_m$ is given as the columns of the matrix $W=UV^T.$
\item[(ii)] First observe that if $A=U\Sigma V^T$, then
$$
\|AA^T-I_m\|_F=\|U\Sigma^2 U^T-I_m\|_F=\Bigl(\sum_{i=1}^m [\sigma_i^2-1]^2\Bigr)^{1/2}.
$$
Hence
\begin{align*}
{\mathcal S}(a_1,\dots,a_m)&=\Bigl(\sum_{i=1}^m(\sigma_i-1)^2\Bigr)^{1/2}\le \Bigl(\sum_{i=1}^m(\sigma_i-1)^2(\sigma_i+1)^2\Bigr)^{1/2}\\
&= \Bigl(\sum_{i=1}^m(\sigma_i^2-1)^2\Bigr)^{1/2}=\|AA^T-I_m\|_F\\
&\le \max_{i}(\sigma_i+1)\Bigl(\sum_{i=1}^m [\sigma_i-1]^2\Bigr)^{1/2}=(\|A\|+1){\mathcal S}(a_1,\dots,a_m).
\end{align*}
\end{enumerate}
\end{proof}
}

\begin{lem}\label{lem:A1}
Let $\varepsilon>0$ and let  $a_1,\dots,a_m\in\R^m$ with ${\mathcal S}(a_1,\dots,a_m)\le\varepsilon$ and $\|a_i\|_2=1$ for all $i=1,\dots,m$
and  let $A\in\R^{m\times m}$ be a matrix with columns $a_1,\dots, a_m$.
\begin{enumerate}
\item[(i)] Then
$$
(1-\varepsilon)\|y\|_2\le \|Ay\|_2\le (1+\varepsilon)\|y\|_2
$$
for all $y\in\R^m$. The result holds with identical proof also for $A^T$ instead of $A$ substituted in the inequality.
\item[(ii)] Let $\displaystyle M=\sum_{j=1}^m\sigma_j a_j\otimes a_j$, then $\|M\|\le (1+\varepsilon)^2\|\sigma\|_\infty$.
\item[(iii)] $\sum_{k\not=j}\langle a_k,a_j\rangle^2\le 2\varepsilon^2$ for all $j=1,\dots,m.$
\item[(iv)] ${\mathcal S}(a_1\otimes a_1,\dots,a_m\otimes a_m)\le 2\varepsilon.$
\item[(v)] $\displaystyle (1-2\varepsilon)\|\sigma\|_2\le \Bigl\|\sum_{j=1}^m \sigma_j a_j\otimes a_j\Bigr\|_F\le (1+2\varepsilon)\|\sigma\|_2.$
\item[(vi)] Let $b_j,j=1,\dots,m$ be the dual basis of $a_j,j=1,\dots,m$ (i.e. $\langle b_i, a_j\rangle=\delta_{i,j}$). Then
$\|b_j\|_2\le 1/(1-\varepsilon)$ for all $j=1,\dots,m.$

\end{enumerate}
\end{lem}
\begin{proof}
\change{We denote by $W$ be the optimal orthonormal matrix for $A$ and its columns by $w_1,\dots,w_m$.}
\begin{enumerate}
\item[(i)] Then
$$
\|Ay\|_2=\|(A-W)y+Wy\|_2\le \|A-W\|\cdot \|y\|_2+\|Wy\|_2\le (1+\varepsilon)\|y\|_2.
$$
The estimate from below follows by applying the inverse triangle inequality. The proof can be used similarly also for obtaining the bounds for $A^T$ instead of $A$.
\item[(ii)] We estimate by (i)
\change{
$$
\|M\|=\|A\Sigma A^T\|\le \|A\|\cdot\|\Sigma\|\cdot\|A^T\|\le (1+\varepsilon)^2\|\sigma\|_\infty,
$$
where $\Sigma\in\R^{m\times m}$ is a diagonal matrix with $(\sigma_1,\dots,\sigma_m)$ on its main diagonal.
}
\item[(iii)] \change{Fix $j\in\{1,\dots,m\}$ and let $\overline A\in\R^{m\times (m-1)}$ be $A$ without the $j$-th column. Similarly, we define $\overline W$.
Then we obtain
\begin{align*}
\Bigl(\sum_{k\not=j}\langle a_k,a_j\rangle^2\Bigr)^{1/2}&=\|\overline A^Ta_j\|_2\le \|(\overline A-\overline W)^Ta_j\|_2+\|\overline W^Ta_j\|_2\\
&\le \|\overline A-\overline W\|+\|\overline W^T(a_j-w_j)\|_2\le \|\overline A-\overline W\|_F+\|a_j-w_j\|_2\\
&\le \sqrt{2}\|A-W\|_F.
\end{align*}
\item[(iv)]
We use triangle inequality and obtain
$$
\|a_j\otimes a_j-w_j\otimes w_j\|_F\le \|(a_j-w_j)\otimes a_j\|_F+\|w_j\otimes (a_j-w_j)\|_F\le 2\|a_j-w_j\|_2.
$$
Summing this estimate squared over $j=1,\dots,m$, we obtain the result.
}
\item[(v)] \change{Using (iv) and Cauchy-Schwarz inequality, we estimate}
\begin{align*}
\change{\|A\Sigma A^T\|_F}&\change{\le \|W\Sigma W^T\|_F+\|A\Sigma A^T-W\Sigma W^T\|_F}\\
&\le \|\sigma\|_2+\sum_{j=1}^m |\sigma_j|\cdot \|a_j\otimes a_j-w_j\otimes w_j\|_F\\
&\le \|\sigma\|_2+\Bigl(\sum_{j=1}^m\sigma_j^2\Bigr)^{1/2}\cdot\Bigl(\sum_{j=1}^m\|a_j\otimes a_j-w_j\otimes w_j\|_F^2\Bigr)^{1/2}\\
&\le (1+2\varepsilon)\|\sigma\|_2
\end{align*}
and similarly for the other side.

\item[(vi)] 
\change{The result follows from
\begin{align*}
\|b_j\|_2&=\|W^Tb_j\|_2=\|A^Tb_j+(W^T-A^T)b_j\|_2\le \|A^Tb_j\|_2+\|(W^T-A^T)b_j\|_2\\
&\le 1+\|W^T-A^T\|\cdot\|b_j\|_2\le 1+\varepsilon\|b_j\|_2.\qedhere
\end{align*}
}
\end{enumerate}
\end{proof}

The next lemma shows that normalization of a set of vectors does not influence much
their distance to an orthonormal basis.
\begin{lem}\label{lem:2S}
Let $\{\alpha_1,\dots,\alpha_m\}\subset{\mathbb R}^m$ be arbitrary non-zero vectors in ${\mathbb R}^m$. 
Then
$$
{\mathcal S}\Bigl(\frac{\alpha_1}{\|\alpha_1\|_2},\dots,\frac{\alpha_m}{\|\alpha_m\|_2}\Bigr)
\le \sqrt{2}{\mathcal S}(\alpha_1,\dots,\alpha_m).
$$
\end{lem}
\begin{proof}
Let $\{w_1,\dots,w_m\}\subset{\mathbb R}^m$ be the closest orthonormal basis to ${\alpha_1,\dots,\alpha_m}$.
Then we may assume that $\langle \alpha_i,w_i\rangle \ge 0$, otherwise exchanging $w_i$ for $-w_i$
would decrease the distance to $\{\alpha_1,\dots,\alpha_m\}.$

For every $i=1,\dots,m$, we obtain
\begin{align*}
\|\alpha_i-w_i\|_2^2\ge \Bigl\|\Bigl\langle w_i,\frac{\alpha_i}{\|\alpha_i\|_2}\Bigl\rangle\frac{\alpha_i}{\|\alpha_i\|_2}-w_i\Bigr\|_2^2
=1-\Bigl\langle w_i,\frac{\alpha_i}{\|\alpha_i\|_2}\Bigl\rangle^2
\end{align*}
and therefore
\begin{align*}
\Bigl\|\frac{\alpha_i}{\|\alpha_i\|_2}-w_i\Bigr\|_2^2=2\Bigl(1-\frac{\langle \alpha_i,w_i\rangle}{\|\alpha_i\|_2}\Bigr)
\le 2\Bigl(1-\frac{\langle \alpha_i,w_i\rangle^2}{\|\alpha_i\|^2_2}\Bigr)\le 2\|\alpha_i-w_i\|_2^2.
\end{align*}
To finish the proof, we sum up over $i=1,\dots,m$ and take the square root.
\end{proof}

\begin{lem}\label{lem:5eps}
Let $\{\alpha_1,\dots,\alpha_m\}\subset\R^n$ be arbitrary linearly independent vectors
with unit Euclidean norm and let $\{\omega_1,\dots,\omega_m\}\subset\R^n$ be orthonormal.
Let $\A={\rm span}\{\alpha_1,\dots,\alpha_m\}$, $\hat \A={\rm span}\{\omega_1,\dots,\omega_m\}$, and
$$
\Bigl(\sum_{i=1}^m\|\alpha_i-\omega_i\|_2^2\Bigr)^{1/2}\le\varepsilon<1.
$$
Then
$$
\|P_{\A}-P_{\hat \A}\|\le \change{4\varepsilon},
$$
where $P_{\A}$ and $P_{\hat \A}$ are the orthogonal projections on $\A$ and $\hat \A$ respectively.
\end{lem}
\begin{proof}
Let $A\in \R^{n\times m}$ have columns $\alpha_1,\dots,\alpha_m$ and let $W \in\R^{n\times m}$ have columns $\omega_1,\dots,\omega_m.$
Then $P_{\hat \A}={W}{W}^T$. If $A=U\Sigma V^T$ is the singular value decomposition of ${A}$ with
$U\in\R^{n\times m},\Sigma\in\R^{m\times m}$ and $V\in \R^{m\times m}$, then $P_{\A}=UU^T.$
\change{Further we denote by $B=UV^T$ the closest matrix to $A$ in Frobenius norm with orthonormal columns, see also Theorem \ref{thm:A1}.
Hence, $\|A-B\|_F\le \|A-W\|_F\le \varepsilon$, $\|B-W\|_F\le2\varepsilon$ and
\begin{align*}
\|P_{\A}-P_{\hat {\A}}\|&=\|UU^T-{W}{W}^T\|\le\|UV^TVU^T-WW^T\|_F=\|BB^T-WW^T\|_F\\
&\le \|B(B^T-W^T)\|_F+\|(B-W)W^T\|_F\le 2\varepsilon+2\varepsilon=4\varepsilon.\qedhere
\end{align*}}
\end{proof}

We conclude this subsection with a remark related to the stability of the result of Theorem \ref{initthm}
with respect to $\varepsilon$-nearly orthonormality.
\begin{rem}\label{rem:orth}
Let $\{a_1,\dots,a_m\}$ be $\varepsilon$-nearly orthonormal and let $\{w_1,\dots,w_m\}\subset\R^d$
be an optimal approximating orthonormal basis such that
$$
{\mathcal S}(a_1,\dots,a_m)=\Bigl(\sum_{j=1}^m \|a_j-w_j\|_2^2\Bigr)^{1/2}=\varepsilon.
$$
By Theorem \ref{thm:A1} (and its proof) we can assume that $\{w_1,\dots,w_m\}\subset A.$
Then for the vectors $\alpha_i=\tilde A^Ta_i, i=1,\dots,m$ constructed in Theorem \ref{initthm} it holds
\begin{align*}
{\mathcal S}(\alpha_1,\dots,\alpha_m)&={\mathcal S}({\tilde A}^Ta_1,\dots,{\tilde A}^Ta_m)= {\mathcal S}({\tilde A}{\tilde A}^Ta_1,\dots,{\tilde A}{\tilde A}^Ta_m)\\
&= {\mathcal S}(P_{\tilde A}a_1,\dots,P_{\tilde A}a_m)
\le\Bigl(\sum_{j=1}^m \|P_{\tilde A}a_j-w_j\|_2^2\Bigr)^{1/2}\\
&\le\Bigl(\sum_{j=1}^m \|P_{\tilde A}a_j-P_{\tilde A}w_j\|_2^2\Bigr)^{1/2}+
\Bigl(\sum_{j=1}^m \|P_{\tilde A}w_j-P_Aw_j\|_2^2\Bigr)^{1/2}\\
&\le \varepsilon+\|P_{\tilde A}-P_{A}\|_F.
\end{align*}
Hence, if the vectors $a_1,\dots,a_m$ are orthogonal, or nearly-orthonormal in the sense of Definition \ref{def:S}, the vectors $\alpha_1,\dots,\alpha_m$
behave similarly.
\end{rem}

\section{Identification of weights} \label{locmaxima}
\rechange{In the case of nearly orthonormal weights $\{a_i : i = 1, \dots, m\}$}, the geometry of the space $\widetilde {\mathcal A} \approx {\mathcal A}=\operatorname{span}\{a_i\otimes a_i, i=1,\dots,m\}\subset \R^{m\times m}$ can be described by the  Euclidean representation
of Figure \ref{fig:00}. Inspired by this geometrical description, for the identification of the individual weights, we consider the following nonlinear program
\begin{align}\label{eq:alg}
{\rm arg\ max} \ \|M\|,\quad {\rm s.t.} \quad M\in\widetilde \A, \|M\|_F\le 1
\end{align}
to recover the $a_i$'s - or, more precisely, their approximations $\hat a_i$ (which is of course possible only up to the sign).

The optimization problem \eqref{eq:alg} is not convex and may in general have a large number of local \rechange{maxima}. Nevertheless, we shall prove that
to every local maximizer of \eqref{eq:alg}, there is one of the matrices $a_i\otimes a_i$, which lies very close to it. Moreover, we show also the converse, i.e., that every $a_i\otimes a_i$ is well-approximated by a local maximizer of \eqref{eq:alg}.
In particular, {for $\widetilde \A = \A$ and the vectors $a_1,\dots, a_m\in\R^m$ orthonormal}, we obtain the exact recovery of the $a_i\otimes a_i$'s.

\subsection{Characterization of local maximizers}

In this section we address  the analysis of the optimization program \eqref{eq:alg},
and we derive a characterization of its local maximal solutions. First of all let us observe that every local maximizer of \eqref{eq:alg} will be always found on the sphere ${\mathbb S_{\widetilde \A}}=\{M\in \widetilde \A:\|M\|_F=1\}$. The set ${\mathbb S}_{\widetilde \A}$ is
a unit sphere in a Hilbert space of (symmetric) matrices, intersected with a linear subspace, and therefore everywhere differentiable. Despite the nonsmoothness of the objective function,
i.e. $M\to \|M\|$, the solution of the nonconvex program \eqref{eq:alg} will be tackled by means of differential methods.

\begin{thm}\label{thm:conditions} Let $M$ be any local maximizer of
\begin{equation}\label{eq:conditions}
{\rm arg\ max} \ \|M\|,\quad {\rm s.t.} \quad M\in\widetilde \A,\ \|M\|_F\le 1.
\end{equation}
\change{
Let us denote the eigenvalues of $M$ by $\lambda_1,\lambda_2,\dots,\lambda_m$ and the corresponding eigenvectors 
by $u_1,\dots,u_m.$ We assume that the eigenvalues are ordered by their absolute value, i.e., that $|\lambda_1|\ge |\lambda_2|\ge\dots\ge 0.$
Then
\begin{align}\label{eq:first'}
u_j^TXu_j&=\lambda_j\langle X,M\rangle_F\quad\text{for every}\quad X\in\widetilde \A\\
&\notag\text{and all}\quad j\in\{1,\dots,m\}\quad\text{with}\quad |\lambda_j|=\|M\|.
\end{align}}
If furthermore
\begin{equation}\label{eq:unique1}
{\mathcal S}(a_1,\dots,a_m)\le \varepsilon
\quad \text{and}\quad 
3 m \|P_{\A}-P_{\widetilde \A}\|_{F\to F}< (1-\varepsilon)^2,
\end{equation}
then $|\lambda_1|=\|M\|$, $\lambda_1\not\in\{\lambda_2,\dots,\lambda_m\}$ and\change{
\begin{equation}\label{eq:sec'}
2\sum_{k=2}^m\frac{(u_1^TXu_k)^2}{|\lambda_1-\lambda_k|}\le \|M\|\cdot\|X-\langle X,M\rangle_F M\|_F^2\quad\text{for all}\ X\in\widetilde{\mathcal A}.
\end{equation}}
\end{thm}
\begin{rem}\label{Rem:thm}
\begin{enumerate}
\item[(i)] \change{In the proof we will only show that \eqref{eq:first'} holds for all $X$ orthogonal to $M$, i.e., that
\begin{align}\label{eq:full1}
u_j^TXu_j=0 \quad &\text{for all}\quad X\in{\mathbb S}_{\widetilde \A}\quad\text{with}\quad X\perp M\\
\notag\qquad&\text{and all } j\in\{1,\dots,m\}\ \text{with } |\lambda_j|=\|M\|.
\end{align}
If $X\in\widetilde \A$ is not orthogonal to $M$ and not co-linear with $M$, then \eqref{eq:first'} follows from \eqref{eq:full1} by considering the matrix 
$$
\frac{X-\langle X,M\rangle_F M}{\|X-\langle X,M\rangle_F M\|_F}.
$$ 
Moreover, if $X$ is a multiple of $M$, \eqref{eq:first'} holds trivially.
Furthermore \eqref{eq:sec'} follows in the same way from
\begin{align}\label{eq:full2}
2\sum_{k=2}^m\frac{(u_1^TXu_k)^2}{|\lambda_1-\lambda_k|}\le |\lambda_1|\quad\text{for all}\quad X\in{\mathbb S}_{\widetilde \A}\quad \text{with}\quad X\perp M.
\end{align}}
\item[(ii)] The formulas \eqref{eq:first'} and \eqref{eq:sec'} resemble very much
the so-called first and second Hadamard variation formula, cf. \cite[Chapter 1.3]{Tao_RM}.
At least in the case when the spectrum of $M$ contains only simple eigenvalues, the proof we give
\change{is similar to} the one in \cite{Tao_RM}.
\end{enumerate}
\end{rem}
Before we come to the full proof of Theorem \ref{thm:conditions}, we sketch its main idea in a simplified setting
to give some intuition about the argument.

If $M$ is a local maximizer of \eqref{eq:alg}, then there is a neighborhood ${\mathcal U}\subset {\mathbb S}_{\widetilde \A}$
of $M$, such that $\|X\|\le \|M\|$ for every $X\in{\mathcal U}.$ Hence for every $X\in {\widetilde \A}$, the function
\begin{equation}\label{eq:f_X}
f_X:\gamma\to \frac{\|M+\gamma X\|}{\|M+\gamma X\|_F}
\end{equation}
has a local maximum in $\gamma=0$. Furthermore, it is enough to restrict ourselves to matrices
$X\in{\mathbb S}_{\widetilde \A}$ with $X\perp M$. 
Let now
\begin{equation}\label{eq:svd}
(M+\gamma X)u_j(\gamma)=\lambda_j(\gamma)u_j(\gamma),\quad j=1,\dots,m,
\end{equation}
be the spectral decomposition of $M+\gamma X$ with eigenvalues $\lambda_j(\gamma)$ and eigenvectors $u_j(\gamma)$.
For simplicity, we assume for now that $\lambda_j(\gamma)$ and $u_j(\gamma)$ depend smoothly on $\gamma$.
Observe, that $u_j=u_j(0)$ and $\lambda_j=\lambda_j(0)$. 
Due to $\|M+\gamma X\|^{-1}_F=(1+\gamma^2)^{-1/2}=1-\gamma^2/2+o(\gamma^2)$, we obtain for $\gamma \rightarrow 0$
\begin{align}\label{eq:fXmax}
f_X(\gamma)=(1-\gamma^2/2)\cdot \max_{j:|\lambda_j(0)|=\|M\|}|\lambda_j(0)+\lambda_j'(0)\gamma+\lambda_j''(0)\gamma^2/2|+o(\gamma^2). 
\end{align}
We conclude, that if $f_X$ has a local maximum in $\gamma=0$, then 
\begin{equation}\label{firstordercond}
\lambda_j'(0)=0\quad\text{for all}\quad j\in\{1,\dots,m\}\quad\text{with}\quad |\lambda_j(0)|=\|M\|.
\end{equation}

In order to determine $\lambda_j'(0)$, we differentiate \eqref{eq:svd}
\begin{equation}\label{eq:svd2}
Mu_j'(\gamma)+Xu_j(\gamma)+\gamma X u_j'(\gamma)=\lambda_j'(\gamma)u_j(\gamma)+\lambda_j(\gamma)u_j'(\gamma),
\end{equation}
evaluate \eqref{eq:svd2} in $\gamma=0$ and multiply it with $u_j$. We obtain
\[
(u_j')^TMu_j+u_j^TXu_j=\lambda_j'(0)+\lambda_j u_j^Tu_j'.
\]
We now plug in the relation $Mu_j=\lambda_ju_j$ together with $(u_j')^Tu_j=0$, which follows by
differentiating the orthogonality relation $\langle u_i(\gamma),u_j(\gamma)\rangle=\delta_{i,j}$, and obtain
\[
\lambda_j'(0)=u_j^TXu_j.
\]
Together with \eqref{firstordercond}, this gives the proof of \eqref{eq:full1}. 

To prove \eqref{eq:full2}, we study also second derivatives and distinguish between local minimizers and local maximizers.
For simplicity 
we assume now that $\lambda_1=\|M\|>\max\{|\lambda_2|,\dots,|\lambda_m|\}$.
In this case we reformulate \eqref{eq:fXmax} using $\lambda_1'(0)=0$ and \eqref{eq:fXmax} becomes
$$
f_X(\gamma)=(1-\gamma^2/2)  ( \lambda_1(0)+\lambda_1''(0)\gamma^2/2)+o(\gamma^2)=\lambda_1(0) +\frac{\lambda_1''(0) -\lambda_1(0)}{2}\gamma^2 + o(\gamma^2).
$$
If $f_X$ has a local maximum at $\gamma=0$, again by a simple asymptotic argument for $\gamma \rightarrow 0$, we conclude that 
\begin{equation}\label{eq:svd3}
\lambda_1''(0)\le\lambda_1(0).
\end{equation}

We differentiate \eqref{eq:svd2} with $j=1$ to obtain
\[
Mu_1''(\gamma)+2Xu_1'(\gamma)+\gamma Xu_1''(\gamma)=\lambda''_1(\gamma)u_1(\gamma)+2\lambda'_1(\gamma)u_1'(\gamma)+\lambda_1(\gamma)u_1''(\gamma).
\]
We evaluate this equation at $\gamma=0$ and take again the inner product with $u_1$, yielding
\[
u_1^TMu_1''+2u_1^TXu_1'=\lambda_1''(0)+2\lambda_1'(0)u_1^Tu'_1+\lambda_1u_1^Tu_1''.
\]
Using $u_1^TMu_1''=\lambda_1 u_1^Tu_1''$ and $u_1^Tu'_1=0$, the equation becomes 
\begin{equation}\label{eq:Had1}
\lambda_1''(0)=2u_1^TXu_1'.
\end{equation}
For eliminating $u_1'$, we multiply \eqref{eq:svd2} for $j=1$ with $u_k, k\not=1$ at $\gamma=0$. This gives
\[
u_k^TMu_1'+u_k^TXu_1=\lambda_1'(0)u_k^Tu_1+\lambda_1u_k^Tu_1'.
\]
Using $u_k^TMu_1'=\lambda_k u_k^Tu_1'$ and $u_k^Tu_1=0$, this can be reformulated as 
$u_k^Tu_1'=(u_k^TXu_1)/(\lambda_1-\lambda_k)$ for $\lambda_1\not=\lambda_k.$ 
Hence
\[
u'_1=\sum_{k=1}^m\langle u_1',u_k\rangle u_k=\sum_{k=2}^m \frac{(u_k^TXu_1)}{\lambda_1-\lambda_k}u_k,
\]
which together with \eqref{eq:svd3} and \eqref{eq:Had1} gives \eqref{eq:full2}.


In the argument above we made heavy use of the additional requirement of smooth dependence of the spectral decomposition
of $M+\gamma X$ on the parameter $\gamma$. We will show now, that the same is true even without such an assumption.

\begin{proof}[Proof of Theorem \ref{thm:conditions}]\ \\

\emph{Step 1. Proof of \eqref{eq:first'}}

Let us assume, that $M\in \widetilde \A$ is fixed and that $f_X$ has local maximum at $\gamma=0$ for $X\in \widetilde \A$ with $\|X\|_F=1$ and $X\perp M.$
Hence, for $|\gamma|$ small, we have
\begin{align*}
\|M\|&\ge \frac{\|M+\gamma X\|}{\|M+\gamma X\|_F}\ge \Bigl(1-\gamma^2/2+o(\gamma^2)\Bigr)\cdot\max_{j=1,\dots,m} |u_j^T(M+\gamma X)u_j|\\
&=\max_{j=1,\dots,m}\Bigl|\lambda_j(0)+\gamma u_j^TXu_j\Bigr|+O(\gamma^2).
\end{align*}
Considering $j\in\{1,\dots,m\}$ with $|\lambda_j(0)|=\|M\|$ and $|\gamma|$ small, we \rechange{arrive at} \eqref{eq:first'}.

\emph{Step 2. Proof of \eqref{eq:sec'}}

We derive \eqref{eq:sec'} under the assumption that $\lambda_1=\|M\|$ and $\lambda_1\not\in \{\lambda_2,\dots,\lambda_m\}$.
If $\lambda_1=-\|M\|$, the result follows by considering $-M$ instead of $M$.

Let again $Mu_j=\lambda_ju_j$ be the singular value decomposition of $M$. Then
\begin{align}
\notag \|M+\gamma X\|&=\sup_{\|\sigma\|_2\le 1}\Bigl(\sum_{i=1}^m\sigma_iu_i\Bigr)^T (M+\gamma X)\Bigl(\sum_{j=1}^m\sigma_ju_j\Bigr)\\
\label{eq:Lagr0}&=\sup_{\|\sigma\|_2\le 1}\biggl(\,\sum_{i,j=1}^m\sigma_i\sigma_ju_i^TMu_j+\gamma\sum_{i,j=1}^m\sigma_i\sigma_ju_i^TXu_j\biggr)\\
\notag &=\sup_{\|\sigma\|_2\le 1}\biggl(\sum_{i=1}^m\sigma^2_i\lambda_i+\gamma\sum_{i,j=1}^m\sigma_i\sigma_jA_{ij}\biggr)=:\sup_{\|\sigma\|_{2}\le 1}f(\sigma),
\end{align}
where $A_{ij}=u_i^TXu_j=A_{ji}.$ We will use an approximate solution of the Lagrange's multiplier equations
to estimate $\|M+\gamma X\|$ from below.

We set the constraint condition $g(\sigma)=\|\sigma\|_2^2=1$ and use Lagrange's multiplier theorem on
$$
\theta(\sigma,\nu):=f(\sigma)+\nu(g(\sigma)-1).
$$
This leads to equations
\begin{align}
\notag \frac{\partial \theta}{\partial \nu}&=g(\sigma)-1=0,\\
\label{eq:Lagr1} \frac{\partial \theta}{\partial \sigma_j}&=2\sigma_j\lambda_j + 2\gamma\sum_{i=1}^m\sigma_iA_{ij}+\nu\cdot 2\sigma_j=0,\quad j=1,\dots,m.
\end{align}

For $j=1$ we use $A_{11}=u_1^TXu_1=0$ and \eqref{eq:Lagr1} becomes
\begin{equation}\label{eq:new2}
\sigma_1(\lambda_1+\nu)=-\gamma\sum_{j=2}^m\sigma_jA_{1j}.
\end{equation}
If $j\ge 2$, we reduce \eqref{eq:Lagr1} by the following observation. The optimal value of $\sigma$
in \eqref{eq:Lagr0} for $\gamma=0$ is $\sigma=e_1=(1,0,\dots,0)^T$. We therefore expect that for $|\gamma|$ small,
the optimal value of $\sigma$ in \eqref{eq:Lagr0} will be close to $e_1$, i.e.
$\sigma_2,\dots,\sigma_m$ are expected to be of order $\gamma.$ The values $A_{ij}$ with $i,j\ge 2$
therefore come into the value of $f(\sigma)$ only in the third order in $\gamma$ and may be neglected. Then \eqref{eq:Lagr1}
becomes
$$
\sigma_j(\lambda_j+\nu)=-\gamma\sigma_1A_{1j}.
$$
Finally, \eqref{eq:new2} shows that $\nu$ is close to $-\lambda_1$.
We are then naturally led to chose $\sigma$ according to the equations
\begin{align}
\label{eq:Lagr3} \sum_{j=1}^m\bar\sigma_j^2=1,\quad
\bar\sigma_1\not=0\quad\text{and}\quad
\frac{\bar\sigma_k}{\bar\sigma_1}=\gamma\cdot \frac{u_1^TXu_k}{\lambda_1-\lambda_k},\quad k=2,\dots,m.
\end{align}
Up to the sign of $\bar\sigma$, there is exactly one solution to \eqref{eq:Lagr3}, which we plug into \eqref{eq:Lagr0}.
This leads to
\begin{align}
\notag \|M+\gamma X\|&\ge f(\bar\sigma)=\sum_{j=1}^m\bar\sigma_j^2\lambda_j+\gamma \sum_{i,j=1}^m\bar\sigma_i\bar\sigma_j A_{ij}\\
\notag &=\bar\sigma_1^2\lambda_1+\sum_{k=2}^m\bar\sigma_k^2\lambda_k+2\gamma\sum_{j=2}^m\bar\sigma_1\bar\sigma_j(u_1^TXu_j)+\gamma \sum_{i,j=2}^m\bar\sigma_i\bar\sigma_j(u_i^TXu_j)\\
\label{eq:Lagr1'}&=\bar\sigma_1^2\lambda_1+\sum_{k=2}^m\lambda_k\gamma^2\bar\sigma_1^2\frac{(u_1^TXu_k)^2}{(\lambda_1-\lambda_k)^2}+2\gamma^2\bar\sigma_1^2\sum_{k=2}^m\frac{(u_1^TXu_k)^2}{\lambda_1-\lambda_k}+o(\gamma^2)\\
\notag &=\bar\sigma_1^2\lambda_1+\gamma^2\bar\sigma_1^2\sum_{k=2}^m\frac{(u_1^TXu_k)^2}{\lambda_1-\lambda_k}\Bigl\{\frac{\lambda_k}{\lambda_1-\lambda_k}+2\Bigr\}+o(\gamma^2)\\
\notag &=\bar\sigma_1^2\Bigl\{\lambda_1+\gamma^2 \sum_{k=2}^m\frac{(u_1^TXu_k)^2}{(\lambda_1-\lambda_k)^2} (2\lambda_1-\lambda_k) \Bigr\}+o(\gamma^2).
\end{align}
Furthermore, from $\|\bar\sigma\|_2^2=1$, we derive
\[
\bar\sigma_1^2+\sum_{k=2}^m\gamma^2\bar\sigma_1^2\frac{(u_1^TXu_k)^2}{(\lambda_1-\lambda_k)^2}=\bar\sigma_1^2\Bigl(1+\gamma^2\sum_{k=2}^m\frac{(u_1^TXu_k)^2}{(\lambda_1-\lambda_k)^2}\Bigr)=1,
\]
which, by the Taylor theorem, leads to 
\[
\bar\sigma_1^2=\Bigl(1+\gamma^2\sum_{k=2}^m\frac{(u_1^TXu_k)^2}{(\lambda_1-\lambda_k)^2}\Bigr)^{-1}=
1-\gamma^2\sum_{k=2}^m\frac{(u_1^TXu_k)^2}{(\lambda_1-\lambda_k)^2}+o(\gamma^2).
\]
We plug this estimate into \eqref{eq:Lagr1'} and get
\begin{align*}
f(\bar\sigma)
&=\Bigl(1-\gamma^2\sum_{k=2}^m\frac{(u_1^TXu_k)^2}{(\lambda_1-\lambda_k)^2}\Bigr)\cdot \Bigl\{\lambda_1+\gamma^2 \sum_{k=2}^m\frac{(u_1^TXu_k)^2}{(\lambda_1-\lambda_k)^2} (2\lambda_1-\lambda_k) \Bigr\}+o(\gamma^2)\\
&=\lambda_1+\gamma^2\sum_{k=2}^m(u_1^TXu_k)^2\Bigl\{-\frac{\lambda_1}{(\lambda_1-\lambda_k)^2}+\frac{2\lambda_1-\lambda_k}{(\lambda_1-\lambda_k)^2}\Bigr\}+o(\gamma^2)\\
&=\lambda_1+\gamma^2\sum_{k=2}^m\frac{(u_1^TXu_k)^2}{\lambda_1-\lambda_k}+o(\gamma^2).
\end{align*}
This allows to conclude that
$$
f_X(\gamma)=\frac{\|M+\gamma X\|}{\sqrt{1+\gamma^2}}\ge f(\bar\sigma)(1-\gamma^2/2)=
\lambda_1+\gamma^2\biggl(\sum_{k=2}^m\frac{(u_1^TXu_k)^2}{\lambda_1-\lambda_k}-\frac{\lambda_1}{2}\biggr)+o(\gamma^2). 
$$
If $f_X$ has local maximum at $\gamma=0$, the coefficient at $\gamma^2$ has to be smaller or equal to zero, giving \change{\eqref{eq:sec'}}.

\emph{Step 3. Uniqueness of the largest eigenvalue}

We proceed by contradiction.
Let \eqref{eq:unique1} be fulfilled and let $M\in\widetilde \A$ with $\|M\|_{F}=1$ be a local maximizer of \eqref{eq:conditions} with $\lambda_1=\lambda_2=\|M\|$.
The case $\lambda_1=\lambda_2=-\|M\|$ follows in the same manner.
Taking $X\in{\mathbb S}_{\widetilde \A}$ with $X\perp M$ and considering again the function $f_X$ from \eqref{eq:f_X}, we can write
\begin{align*}
f_X(\gamma)&= \frac{\|M+\gamma X\|}{\|M+\gamma X\|_F}
\ge\sup_{(\sigma_1,\sigma_2):\sigma_1^2+\sigma_2^2=1} (\sigma_1u_1+\sigma_2u_2)^T(M+\gamma X)(\sigma_1u_1+\sigma_2u_2)+o(\gamma)\\
&=\|M\|+\gamma\sup_{(\sigma_1,\sigma_2):\sigma_1^2+\sigma_2^2=1}\{\sigma_1^2u_1^TXu_1+\sigma^2_2u_2^TXu_2+2\sigma_1\sigma_2 u_1^TXu_2\}+o(\gamma).
\end{align*}
If $f_X$ has a local maximum at $\gamma=0$, we choose $(\sigma_1,\sigma_2)$ equal to $(1,0),(0,1),$ or $(1,1)/\sqrt{2}$, respectively.
We conclude that 
\begin{equation}\label{eq:multiple1}
u_1^TXu_1=u_2^TXu_2=u_1^TXu_2=0.
\end{equation}
If $X\in\widetilde \A$ is not orthogonal to $M$, we apply \eqref{eq:multiple1} to $\frac{X-\langle X,M\rangle_F M}{\|X-\langle X,M\rangle_F M\|_2}$,
cf. Remark \ref{Rem:thm}, and obtain
\begin{align*}
u_1^TXu_1&=\langle X,M\rangle_F\cdot \|M\|=u_2^TXu_2,\\
u_1^TXu_2&=0.
\end{align*}

We set $X_j=P_{\widetilde \A}(a_j\otimes a_j)$ and ${\mathcal E}_j=X_j-a_j\otimes a_j$. Then $X_j\in\widetilde \A$ and we derive from these conditions
\begin{align*}
\langle u_1,a_j\rangle^2+u_1^T{\mathcal E}_ju_1&=\langle u_2,a_j\rangle^2+u_2^T{\mathcal E}_ju_2,\qquad j=1,\dots,m,\\ 
\langle u_1,a_j\rangle\cdot\langle u_2,a_j\rangle&=-u_1^T{\mathcal E}_ju_2,\qquad j=1,\dots,m.
\end{align*}
Solving these equations for $\langle u_1,a_j\rangle^2$, we arrive at
\begin{align}
\langle u_1,a_j\rangle^2&\le \frac{|u_2^T{\mathcal E}_ju_2-u_1^T{\mathcal E}_ju_1|+\sqrt{|u_2^T{\mathcal E}_ju_2-u_1^T{\mathcal E}_ju_1|^2+4(u_1^T{\mathcal E}_ju_2)^2}}{2} \nonumber \\
&\le \frac{2|u_2^T{\mathcal E}_ju_2-u_1^T{\mathcal E}_ju_1|+2|u_1^T{\mathcal E}_ju_2|}{2}\nonumber  \\ 
&=|u_2^T{\mathcal E}_ju_2-u_1^T{\mathcal E}_ju_1|+|u_1^T{\mathcal E}_ju_2| \leq 3 \|{\mathcal E}_j\|. \label{carrot}
\end{align}
By assumption ${\mathcal S}(a_1,\dots,a_m)\le\varepsilon$, Lemma \ref{lem:A1}, and \eqref{carrot}, we then obtain
\begin{align*}
(1-\varepsilon)^2&=(1-\varepsilon)^2\|u_1\|_2^2\le \|A^Tu_1\|_2^2=\sum_{j=1}^m\langle u_1,a_j\rangle^2\le 3\sum_{j=1}^m\|{\mathcal E}_j\|\\
&\le 3\sum_{j=1}^m\|(P_{\widetilde \A}-P_{\A})(a_j\otimes a_j)\|_F\le 3m\|P_{\widetilde \A}-P_{\A}\|_{F\to F},
\end{align*}
which leads to a contradiction.
This finishes the proof of Theorem \ref{thm:conditions}.
\end{proof}

\subsection{Approximation of weights}\label{finapprox}

We show how to use Theorem \ref{thm:conditions} to develop approximation schemes for sums of ridge functions.
%
%
We proceed in two steps.
In the first step we identify vectors $\hat a_1,\dots, \hat a_m\in\R^m$, which approximate the true ridge profiles $a_1,\dots,a_m.$
In the second step \rechange{(see Section \ref{sec:funct})} we define with their help a function $\hat f$, which is the uniform approximation of $f$.

We show, how to use the conditions \eqref{eq:first'} and \eqref{eq:sec'} to analyze the minimization problem \eqref{eq:conditions}.
First, 
we summarize the notation
and assumptions used throughout this section. We assume that
\begin{itemize}
\item $a_1,\dots,a_m\in\R^m$ are the unknown weights/ridge profiles,
\item ${\A}={\rm span}\{a_j\otimes a_j,j=1,\dots,m\}\subset\R^{m\times m}$,
\item the vectors $a_1,\dots,a_m$ are $\varepsilon$-nearly-orthonormal, i.e., there is an orthonormal basis $w_1,\dots,w_m$, such that
$\Bigl(\sum_{j=1}^m\|a_j-w_j\|_2^2\Bigr)^{1/2}=\varepsilon>0,$
\item $\hat {\A}={\rm span}\{w_j\otimes w_j,j=1,\dots,m\}$,
\item $\widetilde \A$ is the approximation of ${\A}$ available after the first step with $\|P_{\A}-P_{\widetilde \A}\|_{F\to F}\le \eta$ (Algorithm \ref{alg3} and Theorem \ref{thm:projapp}),
\item by Lemma \ref{lem:5eps} we then have $\|P_{\hat {\A}}-P_{\widetilde \A}\|_{F\to F}\le \|P_{\hat {\A}}-P_{{\A}}\|_{F\to F}+\|P_{\A}-P_{\widetilde \A}\|_{F\to F}\le 4\varepsilon+\eta=:\nu$.
\end{itemize}

We start with several lemmas needed later on. We will use throughout the notation just introduced.
\begin{lem}\label{lem:nu1}
Let $Z\in \hat {\A}$ and $\nu<1$. Then
$$
\|P_{\widetilde \A}(Z)\|_F\le \|Z\|_F\le \frac{1}{1-\nu}\cdot\|P_{\widetilde \A}(Z)\|_F.
$$
In particular, $P_{\widetilde \A}$ is bijective as a map from $\hat \A$ to $\widetilde \A$.
\end{lem}
\begin{proof}
Let $Z\in \hat {\A}$. Then $\|P_{\widetilde \A}Z\|_F\le \|Z\|_F$ and
\begin{align*}
\|P_{\widetilde \A}Z\|_F&\ge \|P_{\hat {\A}}(Z)\|_F-\|(P_{\widetilde \A}-P_{\hat {\A}})(Z)\|_F\ge \|Z\|_F-\nu \|Z\|_F =(1-\nu)\|Z\|_F.
\end{align*}
The inequality implies the injectivity of  $P_{\widetilde \A}$ on $\hat \A$ and from Theorem \ref{thm:projapp}
we know that ${\rm{dim}}(\widetilde \A)={\rm{dim}(\A)}=m$, hence $P_{\widetilde \A}$ is also surjective.
\end{proof}
Lemma \ref{lem:nu1} ensures that for any $M \in \widetilde \A$ with $\|M\|_F=1$ there exists $Z=\sum_{k}\sigma_k w_k\otimes w_k  \in  \hat \A$
such that $M=P_{\widetilde \A}(Z)$ and
\begin{equation}\label{eq:sigma_est}
1\le \Bigl\|\sum_{k=1}^m \sigma_k w_k\otimes w_k\Bigr\|_F=\|\sigma\|_2\le\frac{1}{1-\nu}.
\end{equation}
We will use this property repetitively below, especially for $M$ being a local maximizer of \eqref{eq:alg}.

If $X=w_j\otimes w_j$ and $\|w_j\|_2=\|u\|_2=1$, then
\begin{align*}
\|Xu\|_2^2=\|\langle w_j,u\rangle w_j\|_2^2=|\langle w_j,u\rangle|^2=(u^Tw_j)(w_j^Tu)=u^TXu.
\end{align*}
If $X=P_{\widetilde \A}(w_j\otimes w_j)$ instead, we expect the difference between $\|Xu\|_2^2$ and $u^TXu$ to be small.
This statement is made precise in the following lemma.

\begin{lem}\label{lem:XX}Let $W_j=w_j\otimes w_j$, $X=P_{\widetilde \A}(W_j)$ and $\|u\|_2=1$. Then
$$
\Bigl|\|Xu\|_2^2-u^TXu\Bigr|\le 2\nu.
$$
\end{lem}
\begin{proof}\phantom\qedhere \hskip-.6cm
Indeed, using $W_j=P_{\hat {\A}}(W_j)=W_j^2$ we obtain
\begin{align*}
\Bigl|\|Xu\|_2^2&-u^TXu\Bigr|=\Bigl|\left\langle P_{\widetilde \A}(W_j)u,P_{\widetilde \A}(W_j)u\right\rangle-u^TP_{\widetilde \A}(W_j)u\Bigr|\\
&\rechange{=\Bigl|u^T\left[P_{\widetilde \A}(W_j)P_{\widetilde \A}(W_j)-P_{\widetilde \A}(W_j)\right]u\Bigr|\le \left\|[P_{\widetilde \A}(W_j)]^2-P_{\widetilde \A}(W_j)\right\|_F}\\
&\rechange{\le \left\| [P_{\widetilde \A}(W_j)]^2-P_{\widetilde \A}(W_j) P_{\hat \A}(W_j) \right\|_F + \left\| P_{\widetilde \A}(W_j)P_{\hat \A}(W_j)-P_{\widetilde \A}(W_j) \right\|_F}\\
&\rechange{=\left\|P_{\widetilde \A}(W_j)(P_{\widetilde \A}-P_{\hat {\A}})(W_j)\right\|_F+\left\|P_{\widetilde {\A}}(W_j)W_j-P_{\hat {\A}}(W_j)W_j+
P_{\hat {\A}}(W_j)-P_{\widetilde {\A}}(W_j)\right\|_F}\\
&\rechange{=\left\|P_{\widetilde \A}(W_j)(P_{\widetilde \A}-P_{\hat {\A}})(W_j)\right\|_F+\left\|(P_{\hat {\A}}-P_{\widetilde \A})(W_j)({\rm Id}-W_j)\right\|_F\le 2\nu.\qquad\qquad\ \ \qed}
\end{align*}
\end{proof}

We show that the local maximizers $M$ of \eqref{eq:alg} are (possibly after replacing $M$ by $-M$) nearly positive semi-definite.

\begin{lem} Let $\nu\le 1/4$ and let $M$ be a local maximizer of \eqref{eq:conditions} with $\|M\|=\lambda_1\ge \lambda_2\ge\dots\ge \lambda_m$
being its eigenvalues. Then
\begin{equation}\label{eq:lambdamin}
\lambda_m\ge -\frac{2\nu}{\lambda_1}-4\nu.
\end{equation}
\end{lem}
\begin{proof}
\change{
We denote again $W_j=w_j\otimes w_j$. We use again Lemma \ref{lem:nu1} and the bijectivity of $P_{\widetilde \A}$ as mapping from $\hat\A$ onto $\widetilde\A$,
which allows us to write $M$ as $\displaystyle M=P_{\widetilde \A}\Bigl(\sum_{k=1}^m\sigma_k W_k\Bigr)$ for suitable $\sigma_1,\dots,\sigma_m$.

This allows us to use Cauchy-Schwarz inequality and \eqref{eq:sigma_est} to estimate $\lambda_m$
\begin{align}
\notag\lambda_m&=u_m^TMu_m= \sum_{j=1}^m\sigma_j\langle P_{\widetilde \A}(W_j),u_m\otimes u_m\rangle_F\\
\notag&=\sum_{j=1}^m\sigma_j \langle W_j,u_m\otimes u_m\rangle_F+\sum_{j=1}^m\sigma_j \langle (P_{\widetilde \A}-P_{\hat {\A}})(W_j),u_m\otimes u_m\rangle_F\\
\label{eq:reorder1}&\ge (\min_j\sigma_j)\sum_{j=1}^m\langle w_j,u_m\rangle^2+\Big\langle (P_{\widetilde \A}-P_{\hat {\A}})\Bigl(\sum_{j=1}^m\sigma_j W_j\Bigr),u_m\otimes u_m\Big\rangle_F\\
\notag&\ge (\min_j\sigma_j)-\|P_{\hat {\A}}-P_{\widetilde \A}\|_{F\to F}\cdot\Bigl\|\sum_{j=1}^m\sigma_j W_j\Bigr\|_F\\
\notag&\ge (\min_j\sigma_j)-\frac{\nu}{1-\nu}.
\end{align}

To estimate $\min_j\sigma_j$ from below, we plug $X=P_{\widetilde \A}(W_j)$ into \eqref{eq:first'} to obtain
$$
u_1^TP_{\widetilde \A}(W_j)u_1=\lambda_1\langle P_{\widetilde \A}(W_j),M\rangle_F.
$$
We observe now that from $\|P_{\widetilde \A}-P_{\hat {\A}}\|_F \leq \nu$, we have $|u_1^T(P_{\widetilde \A}-P_{\hat {\A}})(W_j)u_1| \leq \nu$,
implying $u_1^T(P_{\widetilde \A}-P_{\hat {\A}})(W_j)u_1 \geq -\nu$ and
\begin{align}
\notag-\nu&\le \langle w_j,u_1\rangle^2+u_1^T(P_{\widetilde \A}-P_{\hat {\A}})(W_j)u_1=
u_1^TW_ju_1+u_1^T(P_{\widetilde \A}-P_{\hat {\A}})(W_j)u_1\\
\label{eq:ref2_1}&=u_1^TP_{\widetilde \A}(W_j)u_1=\lambda_1\langle P_{\widetilde \A}(W_j),M\rangle_F 
=\lambda_1\Bigl\langle P_{\widetilde \A}(W_j),P_{\widetilde \A}\Bigl(\sum_{k=1}^m \sigma_kW_k\Bigr)\Bigr\rangle_F.
\end{align}}

\rechange{
By Lemma \ref{lem:nu1}, we obtain $\|P_{\widetilde\A}(W_j)\|_F\ge (1-\nu)\|W_j\|_F=1-\nu$. 
If $\sigma_j<0$ for any $j$, we combine this with \eqref{eq:ref2_1} to derive
\begin{align*}
-\nu/\lambda_1&\le \sigma_j \|P_{\widetilde \A}(W_j)\|^2_F+\Bigl\langle P_{\widetilde \A}(W_j),P_{\widetilde \A}\Bigl(\sum_{k\not=j}\sigma_k W_k\Bigr)\Bigr\rangle_F\\
&\le\sigma_j(1-\nu)^2+\Bigl\langle P_{\widetilde \A}(W_j)-W_j,\sum_{k\not=j}\sigma_k W_k\Bigr\rangle_F\\
&=\sigma_j(1-\nu)^2+\Bigl\langle (P_{\widetilde \A}-P_{\hat\A})(W_j),\sum_{k\not=j}\sigma_k W_k\Bigr\rangle_F.
\end{align*}
By Cauchy-Schwarz inequality and \eqref{eq:sigma_est}, we may further estimate
\begin{align*}
-\nu/\lambda_1&\le \sigma_j(1-\nu)^2+\left\|(P_{\widetilde \A}-P_{\hat\A})(W_j)\right\|_F\cdot \left\|\sum_{k\not=j}\sigma_k W_k\right\|_F\\
&\le \sigma_j(1-\nu)^2+\|\sigma\|_2\cdot\left\|(P_{\widetilde \A}-P_{\hat\A})(W_j)\right\|_F
\le\sigma_j(1-\nu)^2+\frac{\nu}{1-\nu}.
\end{align*}}
We conclude that
\begin{equation*}
\min_{j=1,\dots,m}\sigma_j\ge -\Bigl(\frac{\nu}{\lambda_1}+\frac{\nu}{1-\nu}\Bigr)\cdot \frac{1}{(1-\nu)^2},
\end{equation*}
which we insert into \eqref{eq:reorder1}
and the result follows by  simple algebraic computations for $\nu\le 1/4$.
\end{proof}

The recovery algorithm based on the optimization problem \eqref{eq:conditions} is quite straightforward. We show that the
eigenvector corresponding to the largest eigenvalue of any of its local maximizers is actually close
to one of the ridge profiles.
\vskip.5cm
\fbox{
\begin{minipage}{13.6cm}
\begin{algorithm}\label{alg6}
\emph{\begin{itemize}
\item Let $M$ be a local maximizer of \eqref{eq:conditions}.
\item If $\|M\|$ is not an eigenvalue of $M$, replace $M$ by $-M$.
\item Denote by $\lambda_1\ge\lambda_2\ge\dots\ge\lambda_m$ the eigenvalues of $M$ arranged in decreasing order.
\item Take the eigenvalue decomposition of $M$, i.e. $M=\sum_{j=1}^m\lambda_ju_j\otimes u_j$.
\item Put $\hat a:=u_1$.
\end{itemize}}
\end{algorithm}
\end{minipage}
}\vskip.3cm

The performance of Algorithm \ref{alg6} is guaranteed by the following theorem.
 \begin{thm}\label{thm:recoveridge}
If $0<\nu < 1/(c m)$, for a suitable constant 
$c>6$,
then there is $j_0\in\{1,\dots,m\}$, such that
the vector $\hat a$ found by Algorithm \ref{alg6} satisfies $\|\hat a-a_{j_0}\|_2\le 5\nu$.
\end{thm}

The proof of this theorem, which we report below, is fundamentally based on proving the following bound
\begin{equation}\label{bound4specgap}
\|M\| =\lambda_1  \geq 1-c'\nu,
\end{equation}
for some $c'>0$ and for any local maximizers $M$ of \eqref{eq:conditions}. This will allow to ensure a sufficient spectral gap \rechange{to apply Wedin's bound} (Theorem \ref{wedin}) for showing good approximation properties of $\hat a$  as in Algorithm \ref{alg6} to one of the ridge directions $a_1,\dots,a_m$. We shall obtain  
 \eqref{bound4specgap} by a bootstrap argument: first we need to establish a weaker bound
 $$
 \lambda_1 > \frac{\nu}{1-\nu},
  $$ 
and use it for deducing    \eqref{bound4specgap}.
\begin{lem}\label{bootstrap}
Assume $0<\nu < 1/(6m)$. 
Let $M$ be any of the local maximizers of \eqref{eq:conditions} with $\lambda_1=\|M\|$.
\change{Then 
\begin{equation}\label{bootstrap1}
\|Xu_1\|_2^2 \leq \lambda_1^2\cdot\frac{1+\langle X,M\rangle_F^2}{2}+3\nu
\end{equation}
for any $X \in\widetilde \A$ such that $\|X\|_F\le 1$.
Furthermore,
\begin{equation}\label{bootstrap2}
\lambda_1 > \frac{\nu}{1-\nu}.
\end{equation}}
\end{lem}
\begin{proof}
Let $M$ be any of the local maximizers of \eqref{eq:conditions} with $\lambda_1=\|M\|$. Further let $X\in\widetilde \A$.
We estimate the left-hand side of \eqref{eq:sec'} \change{using orthonormality of $\{u_1,\dots,u_m\}$ and \eqref{eq:first'}}
\begin{align*}
2\sum_{k=2}^m\frac{(u_1^TXu_k)^2}{\lambda_1-\lambda_k}&\ge 2\min\Bigl(\frac{1}{\lambda_1-\lambda_k}\Bigr)\sum_{k=2}^m (u_1^TXu_k)^2
=\frac{2}{\lambda_1-\lambda_m}\sum_{k=2}^m \langle Xu_1,u_k\rangle^2\\
&=\frac{2}{\lambda_1-\lambda_m} \Bigl(\|Xu_1\|_2^2-\langle Xu_1,u_1\rangle^2\Bigr)\\
&=\frac{2}{\lambda_1-\lambda_m} \Bigl(\|Xu_1\|_2^2-\lambda_1^2\langle X,M \rangle_F^2\Bigr).
\end{align*}
Together with \eqref{eq:sec'}, this leads to
$$
\frac{2}{\lambda_1-\lambda_m}\Bigl(\|Xu_1\|_2^2-\lambda_1^2\langle X,M \rangle_F^2\Bigr)
\le \lambda_1(\|X\|_F^2-\langle X,M\rangle_F^2).
$$
If moreover $\|X\|_F\le 1$, $\nu\le\frac{1}{4}$, and using \eqref{eq:lambdamin} we conclude that
\begin{align*}
\|Xu_1\|_2^2&\le \frac{\lambda_1(\lambda_1-\lambda_m)}{2}(\|X\|_F^2-\langle X,M\rangle_F^2)+\lambda_1^2\langle X,M\rangle_F^2  \nonumber \\
&\le \frac{\lambda_1(\lambda_1-\lambda_m)}{2}+\frac{\lambda_1(\lambda_1+\lambda_m)}{2}\langle X,M\rangle_F^2  \nonumber \\
&= \lambda_1^2\cdot\frac{1+\langle X,M\rangle_F^2}{2}-\lambda_1\lambda_m\cdot\frac{1-\langle X,M\rangle_F^2}{2}  \nonumber\\
&\le \lambda_1^2\cdot\frac{1+\langle X,M\rangle_F^2}{2}+2\lambda_1\Bigl(\frac{\nu}{\lambda_1}+2\nu\Bigr)\cdot\frac{1-\langle X,M\rangle_F^2}{2} \nonumber \\
&\le \lambda_1^2\cdot\frac{1+\langle X,M\rangle_F^2}{2}+3\nu. \label{eq:lastest}
\end{align*}
\change{The proof of \eqref{bootstrap2} follows easily from $\lambda_1>1/\sqrt{m}$ and $\nu<1/(6m).$}
\end{proof}

\begin{proof}[Proof of  Theorem \ref{thm:recoveridge}]
Let $M$ be any of the local maximizers of \eqref{eq:conditions} with $\lambda_1=\|M\|$.
\change{We denote again $W_j=w_j\otimes w_j$ and assume that there exists
$Z=\sum_{k=1}^m\sigma_k W_k \in \hat \A$ such that $M=P_{\widetilde \A} Z$.
For $j\in\{1,\dots,m\}$ fixed and $X=P_{\widetilde \A}W_j$} we apply by Lemma \ref{lem:XX} and obtain
$$
\|Xu_1\|_2^2\ge u_1^TXu_1-2\nu=\lambda_1\langle X,M\rangle_F -2\nu.
$$
Using \eqref{bootstrap1}, we then arrive at 
\begin{align*}
\lambda_1\langle X,M\rangle_F -2\nu&\le \lambda_1^2\cdot\frac{1+\langle X,M\rangle_F^2}{2}+3\nu,
\end{align*}
which can be further rewritten as
\begin{equation}\label{eq:scalar3}
0\le\lambda_1^2-1+(1-\lambda_1\langle X,M\rangle_F)^2+10\nu.
\end{equation}
\change{Further we use  \eqref{eq:sigma_est} and estimate $\lambda_1$ from above 
\begin{align}
\lambda_1&=u_1^TMu_1=u_1^T(P_{\widetilde \A}-P_{\hat \A})(Z)u_1 + u_1^T Z u_1 
\notag\le \|(P_{\widetilde \A}-P_{\hat \A})(Z)\|_F+\sum_{k=1}^m \sigma_k u_1^TW_ku_1\\
\label{eq:scalar2}&\le\nu 
\|Z\|_F+\sum_{k=1}^m\sigma_k\langle w_k,u_1\rangle^2
\le\frac{\nu}{1-\nu}+\max_{j=1,\dots,m}\sigma_j.
\end{align}}
From Lemma \ref{bootstrap} and in particular by \eqref{bootstrap2} we deduce that
$$
\max_{j=1,\dots,m}\sigma_j \geq \lambda_1 - \frac{\nu}{1-\nu} >0.
$$
Hence there exists certainly some $j$ for which $\sigma_j >0$.
\change{If $\sigma_j> 0$, we put $X=P_{\widetilde \A}W_j$, $Z_j'=\sum_{k\not=j}\sigma_kW_k$ and get}
\change{
\begin{align*}
\langle X,M\rangle_F&=\langle P_{\widetilde \A}W_j,P_{\widetilde \A}Z\rangle_F=
\langle P_{\widetilde \A}W_j,P_{\widetilde \A}(\sigma_jW_j+Z_j')\rangle_F\\
&=\sigma_j\langle P_{\widetilde \A}W_j,W_j\rangle_F+\langle W_j,(P_{\widetilde \A}-P_{\hat\A})(Z_j')\rangle_F\\
&=\sigma_j\langle W_j,W_j\rangle_F+\sigma_j\langle (P_{\widetilde \A}-P_{\hat\A})W_j,W_j\rangle_F+\langle W_j,(P_{\widetilde \A}-P_{\hat\A})(Z_j')\rangle_F\\
&\ge\sigma_j\cdot(1-\nu)-\nu \|\sigma\|_2\ge\sigma_j\cdot(1-\nu)-\frac{\nu}{1-\nu}.
\end{align*}
}
We conclude, that there is $j_0\in\{1,\dots,m\}$ with
\begin{align}\label{eq:scalar1}
\langle P_{\widetilde \A}W_{j_0},M\rangle_F\ge(1-\nu)\max_{j=1,\dots,m}\sigma_j-\frac{\nu}{1-\nu}.
\end{align}

Combining \eqref{eq:scalar1} with \eqref{eq:scalar2}, we obtain for $\nu\le 1/4$
\begin{align*}
\langle P_{\widetilde \A}W_{j_0},M\rangle_F&\ge\Bigl(\lambda_1-\frac{\nu}{1-\nu}\Bigr)\cdot(1-\nu)-\frac{\nu}{1-\nu}\\
&=\lambda_1(1-\nu)-\nu\frac{2-\nu}{1-\nu}\ge\lambda_1(1-\nu)-\frac{5\nu}{2}
\end{align*}
and
\begin{equation}\label{eq:scalar4}
0\le 1-\lambda_1\langle P_{\widetilde \A}W_{j_0},M\rangle_F\le 1-\lambda^2_1(1-\nu)+\frac{5\lambda_1\nu}{2}.
\end{equation}
Finally, \eqref{eq:scalar3} with \eqref{eq:scalar4} give
\begin{align*}
0&\le \lambda_1^2-1+\Bigl(1-\lambda^2_1(1-\nu)+\frac{5\lambda_1\nu}{2}\Bigr)^2+10\nu\\
&=\lambda_1^2(\lambda_1^2-1)+\nu\Bigl\{-2\lambda_1^4+\nu\lambda_1^4+2\lambda_1^2+\frac{25}{4}\lambda_1^2\nu+5\lambda_1-5\lambda_1^3(1-\nu)+10\Bigr\}\\
&\le \lambda_1^2(\lambda_1^2-1)+\tilde c\nu,
\end{align*}
where elementary calculus and the condition $\nu<1/6$ show that we can take $\tilde c=15.$

\change{It follows that there exists an absolute constant $\nu_0>0$ such that
if $0<\nu\le\nu_0$, the latter inequality allows only two possibilities, namely
$\lambda_1 \geq 1- c' \nu$ and $\lambda_1 \leq c'' \sqrt \nu$, for some absolute constants $c',c''>0$.
Finally, choosing $c>(c'')^2$ large, the second option is in contradiction with $\lambda_1\ge 1/\sqrt{m}$ and $\nu<1/(cm).$
Therefore, there exists $c'>0$ such that
\begin{equation*}
\lambda_1 \geq 1-c'\nu.
\end{equation*}
A detailed inspection of the argument above shows that we can take $\nu_0=.01, c''=\sqrt{15}, c'=10$ and $c=16$, but we stress that we did not
try to optimize the numerical values of these constants.}


Finally, \rechange{we apply Wedin's bound}, Theorem \ref{wedin} in the Appendix, to
$$
\tilde B=M=\sum_{j=1}^m |\lambda_j|  {\rm sign}{\lambda_j} u_j\otimes u_j\quad\text{and}\quad  B=\sum_{k=1}^m|\sigma_k | {\rm sign}{\sigma_k}w_k\otimes w_k.
$$
We assume without loss of generality that $\sigma_1 = \max_{k=1,\dots,m} |\sigma_k| $.
We observe that
$$
\|\tilde B-B\|_F=\|P_{\widetilde \A}( B)- B\|_F=\|(P_{\widetilde \A}-P_{\hat {\A}})(B)\|_F\le \nu\| B\|_F=\nu\|\sigma\|_2\le\frac{\nu}{1-\nu}.
$$
\change{Furthermore, elementary calculations show that
\begin{equation}\label{eq:lambdasigma}
\lambda_1>3/4\quad \text{and}\quad\max_{k=2,\dots,m}|\sigma_k|<1/4
\end{equation}
for $\nu\le\nu_0$ sufficiently small.
Indeed, the first inequality in \eqref{eq:lambdasigma} follows from $\|\lambda\|_2\le 1$ and $\lambda_1\ge 1-c'\nu$.
And $\|\sigma\|_2\le 1/(1-\nu)$ together with \eqref{eq:scalar2} and
$$\sigma_1 \geq \lambda_1 - \frac{\nu}{1-\nu} \geq  1- \bigg (\frac{1}{1-\nu} +c'\bigg) \nu$$ 
imply the second inequality in \eqref{eq:lambdasigma}.}

We deduce that we can choose $\bar\alpha\ge \frac{1}{2}$ in Theorem \ref{wedin} for $0<\nu<\nu_0$ small enough, i.e., 
\begin{equation*}
\min_{k=2,\dots,m} | \lambda_1 - \sigma_{k} |  \geq \bar \alpha \geq1/2 \quad\text{and}\quad
| \lambda_{1} | \geq  \bar \alpha \geq1/2,
\end{equation*}
are verified for $0<\nu<\nu_0$ small enough. 
We therefore obtain $\|u_1\otimes u_1-w_1\otimes w_1\|_F\le 4\nu$. After a possible sign change of $w_1$,
we can assume that $\langle u_1,w_1\rangle\ge 0$ and obtain
\begin{align*}
16\nu^2\ge \|u_1\otimes u_1-w_1\otimes w_1\|_F^2=2(1-\langle u_1,w_1\rangle^2)\ge 2(1-\langle u_1,w_1\rangle)=\|u_1-w_1\|_2^2
\end{align*}
and, finally,
\begin{equation*}
\|u_1-a_1\|_2\le \|u_1-w_1\|_2+\|w_1-a_1\|_2\le 4\nu+\varepsilon\le 5\nu.
\end{equation*}
\end{proof}
\change{
Theorem \ref{thm:recoveridge} shows that every local maximizer of \eqref{eq:conditions} lies close to some of the matrices $a_j\otimes a_j$ and, by Algorithm \ref{alg6}, allows for recovery of an approximation of $a_j$.
We conclude this section by proving conversely that every $a_j\otimes a_j$ can be approximated by a local maximizer of \eqref{eq:alg}.
\begin{prop} \label{converse}
Assume $\nu \leq 1/24$. Then for any $a_j\otimes a_j$ there exists a local maximizer $M$ of \eqref{eq:alg} such that 
$$
\| a_j\otimes a_j - M\| \leq 2 \varepsilon +  \sqrt{6\nu}.
$$
\end{prop}
\begin{proof}Let us consider the mapping
$$
\Phi:X\to\frac{P_{\widetilde \A}(X)}{\|P_{\widetilde \A}(X)\|_F}
$$
from $\{X\in\hat\A:\|X\|_F=1\}$ onto $\{Y\in\widetilde \A:\|Y\|_F=1\}$. We denote again $W_j=w_j\otimes w_j$ and, by triangle inequality,
$$
\|\Phi(W_j)\|=\frac{\|P_{\widetilde \A}(W_j)\|}{\|P_{\widetilde \A}(W_j)\|_F}\ge\frac{\|P_{\hat \A}(W_j)\|-\|(P_{\hat\A}-P_{\widetilde \A})(W_j)\|_F}{\|W_j\|_F} \ge 1-\nu
$$ and, similarly,
$$
\|\Phi(X)\|\le \frac{\|X\|+\nu}{1-\nu} \quad \text{for every}\quad X\in\hat \A\quad\text{with}\quad \|X\|_F=1.
$$
Therefore, if $\sqrt{6\nu}\le r\le 1/2$, $Z\in\hat\A$ with $\|Z\|_F=1$ and $\|Z-W_j\|=r$, then $\|Z\|\le 1-3\nu$. The latter inequality follows by the fact that $Z\in\hat\A$ is linear combination of the orthonormal rank-$1$ matrices $W_j$ and rather straightforward estimations of the spectral norm. Therefore, we obtain also 
$$
\|\Phi(Z)\|\le \frac{1-2\nu}{1-\nu}<1-\nu\le\|\Phi(W_j)\|.
$$
We conclude, that (for $\nu$ small enough) the maximum of $\|\cdot \|$ on the compact set $\{\Phi(Z):Z\in\hat \A, \|Z\|_F=1, \|Z-W_j\|\le \sqrt{6\nu}\}$
is not attained on its boundary (with respect to the topology of the unit Frobenius sphere) and lies therefore close to $W_j$. More precisely 
there exists \rechange{a local maximizer of \eqref{eq:conditions} with}
$M \in \widetilde {\mathcal A}$ such that $\|M - W_j\| \leq \sqrt{6\nu}$. By Lemma \ref{lem:A1} (iv) we deduce that
$$
\|a_j\otimes a_j - M \| \leq \|a_j\otimes a_j - W_j \|_F +  \|M - W_j\| \leq 2 \varepsilon +  \sqrt{6\nu}.
$$
\end{proof}
}

\subsection{A gradient-ascent algorithm} \label{simpalg}

Let us describe in this section how to approach practically the solution of the nonlinear program \eqref{eq:alg}. 
Let us introduce first for a given parameter $\gamma>1$ an operator acting on the singular values of a matrix $X=U \Sigma V^T$ as follows.
If $\Sigma\in\R^{m\times m}$ is a diagonal matrix with the singular values of $X$ denoted by $\sigma_1\ge\sigma_2\ge\dots\ge\sigma_m$ on the diagonal, we set
$$
\Pi_\gamma(X) = \frac{1}{\sqrt{\gamma^2\sigma_1^2+\sigma_2^2+\dots+\sigma_m^2}} \ U \left ( 
\begin{array}{llll} 
\gamma \sigma_1 &0 &\dots&0 \\
0&\sigma_2& 0&\dots\\
\dots&\dots&\dots&\dots\\
0&\dots&0&\sigma_m
\end{array}
\right )
V^T.
$$

Notice that $\Pi_\gamma$ maps any matrix $X$ onto a  matrix of unit Frobenius norm, simply exalting the first singular value and
damping the others. It is not a linear operator. Furthermore, the definition of $\Pi_\gamma(X)$ is not well-posed
if $\sigma_1=\sigma_2\ge \sigma_3\ge\dots\ge\sigma_m$ and in this case it is assumed that a choice of ordering is made for just this one application
of $\Pi_\gamma(\cdot).$ Notice that if $\sigma_1>1/\sqrt{2}$ and $\|X\|_F\le 1$, then $\sigma_1>\sigma_2$ and
$\Pi_\gamma(X)$ is well-defined. This is the case for example under the conditions of Lemma \ref{lipcont}.

We propose the following algorithm

\vskip.3cm\fbox{
\begin{minipage}{13.6cm}
\begin{algorithm}\label{alg7}
\emph{\begin{itemize}
\item Fix a suitable parameter $\gamma >1$.
\item Generate an initial guess $X^0  \in \widetilde \A$ and $\|X^0\|_F=1$ at random.
\item For $\ell \geq 0$:
\begin{itemize}
\item[] $X^{\ell+1} := P_{\widetilde \A} \Pi_\gamma(X^\ell)$.
\end{itemize}\end{itemize}}
\end{algorithm}
\end{minipage}
}\vskip.3cm

This algorithm performs essentially an iteratively projected subgradient ascent method as the two operations executed within the loop are respectively a subgradient ascent step
towards the maximization of the spectral norm by means of $\Pi_\gamma$, and a projection back onto $\widetilde \A$ by $P_{\widetilde \A}$.

In the following we analyze some of the convergence properties of this algorithm and its relationship to \eqref{eq:alg}.
Assume for a moment now that $\widetilde \A=\A$ and that $a_1,\dots,a_m$ are orthonormal. In this case the algorithm can be rather trivially analyzed
and performs a straightforward computation of one of the maximizers of \eqref{eq:alg}.
As we shall see later, such maximizer in this case coincides (up to the sign) with one of the matrices $a_j\otimes a_j$.
\begin{prop} Assume that $\widetilde \A=\A$ and that $a_1,\dots,a_m$ are orthonormal. Let $\gamma>\sqrt{2}$ and let $\|X^0\|>1/\sqrt{\gamma^2-1}$.
Then there exists $\mu_0<1$ such that
\begin{equation}\label{eq:alg:1}
\left |1- \|X^{\ell+1}\|\right | \leq \mu_0 \left |1- \|X^{\ell}\|\right |, \quad \mbox{for all } \ell \geq 0.
\end{equation}
Being the sequence $(X^\ell)_\ell$ made of matrices with Frobenius norm bounded by $1$, we conclude that any of its accumulation points has both unit
Frobenius and spectral norm and therefore it has to coincide
with one maximizer of \eqref{eq:alg}.
\end{prop}

\begin{proof}
We can assume now that $X^0$
can already be expressed in terms of its singular value decomposition $X^0=\sum_{j=1}^m \sigma_j(X^0) a_j\otimes a_j$.
Since at each iteration $\|X^\ell\|_F\leq 1$ or $1\geq \sum_{j=1}^m \sigma_j(X^{\ell})^2$, it is a straightforward observation that
\begin{align}\label{iter}
\|X^{\ell+1}\| &= \sigma_1(X^{\ell+1})=
\frac{\gamma \sigma_1(X^{\ell})}{\sqrt{\gamma^2\sigma_1(X^{\ell})^2+\sigma_2(X^{\ell})^2+\dots+\sigma_m(X^{\ell})^2}}\\
\notag&\geq \frac{\gamma \sigma_1(X^{\ell})}{\sqrt{ (\gamma^2-1)\sigma_1(X^{\ell})^2 +1}}.
\end{align}
Using elementary calculations we further estimate
\begin{align}
\notag 1-\sigma_1(X^{\ell+1})&\leq \frac{\sqrt{ (\gamma^2-1)\sigma_1(X^{\ell})^2 +1}-\gamma \sigma_1(X^{\ell})}{\sqrt{ (\gamma^2-1)\sigma_1(X^{\ell})^2 +1}}\\
\label{iter:2}&=\frac{1-\sigma_1(X^{\ell})^2}{[\sqrt{ (\gamma^2-1)\sigma_1(X^{\ell})^2 +1}+\gamma \sigma_1(X^{\ell})]\sqrt{ (\gamma^2-1)\sigma_1(X^{\ell})^2 +1}}\\
\notag&\le \frac{2(1-\sigma_1(X^{\ell}))}{(\gamma^2-1)\sigma_1(X^{\ell})^2 +1}
\end{align}
and we get \eqref{eq:alg:1} with
$$
\mu_0:=\frac{2}{(\gamma^2-1)\|X^0\|^2+1}<1.
$$
\end{proof}

Let us now move away from the ideal case of the $\widetilde \A =\A$ and assume that $\widetilde \A$ is only a good approximation to $\A$, in the sense that $\|P_{\widetilde \A} - P_{\A}\|_{F\to F} \leq \epsilon$.

\begin{rem} As we \rechange{will see in Section \ref{whitening}},
we can retain without loss of generality the assumption of $a_1,\dots,a_m$ being orthonormal to a certain extent.
Indeed, were $\{a_1,\dots,a_m\}$ just $\varepsilon$-near-orthonormal and $\{w_1,\dots,w_m\}$ its approximating orthonormal basis, then we could denote $\alpha_i=a_ia_i^T\in\R^{m\times m}$, $\omega_i=w_iw_i^T\in\R^{m\times m}$, $\A={\rm span}\{\alpha_1,\dots,\alpha_m\}$ and $\hat \A={\rm span}\{\omega_1,\dots,\omega_m\}.$
It is shown in Lemma \ref{lem:A1} (iv), that $\displaystyle\Bigl(\sum_{i=1}^m\|\alpha_i-\omega_i\|_F^2\Bigr)^{1/2}\le 2\varepsilon$. Combining this result with Lemma \ref{lem:5eps} we would obtain $\|P_{\A}-P_{\hat \A}\|_{F\to F} \leq 8 \varepsilon$. Hence, at the price of changing slightly the reference  orthonormal basis and accepting some additional approximation error of order $\varepsilon$, also in the case of a $\varepsilon$-near-orthonormal system of vectors we can reduce  the arguments to the case of an orthonormal system.
\end{rem}

Unfortunately, in the perturbed case $\widetilde \A \neq\A$, there is no direct way of estimating $\|X^{\ell+1}\|$ by some function of $\|X^{\ell}\|$
as it is done in \eqref{iter} as the singular value decompositions of the matrices
$X^{\ell+1}$ and $X^\ell$ are in principle different. However, the singular vectors of both these matrices can be approximated by $\{a_1,\dots,a_m\}$ (we reiterate that here we assume them orthonormal)
and we need to take advantage of this reference orthonormal system.
First, we need to show a certain continuity property of the operator $\Pi_\gamma$.

\begin{lem}\label{lipcont}
Assume $X$, $\tilde X$ to be  two matrices in $\mathbb R^{m \times m}$ with respective singular value decompositions $X=U \Sigma V^T$ and $\tilde X = \tilde U \tilde \Sigma \tilde V^T$.
Let us also assume that $\|X - \tilde X\|_F \leq \epsilon$ for some $0<\epsilon<1$. Assume additionally that $\max \{ \|\tilde X\|_F, \|X\|_F\} \leq 1$
and 
$\sigma_1(X)\ge t_0:=\frac{1}{\sqrt{2}}+\epsilon+\xi$, for $\xi>0$.
Then, for $\gamma>1$
\begin{equation}\label{pippo0}
 \|\Pi_\gamma(X) - \Pi_\gamma(\tilde X) \|_F
\leq 
2^{3/2}\epsilon+
\frac{4 \epsilon}{\xi}
+2\sqrt{1-(t_0-\epsilon)}
:=\mu_1(\gamma,t_0,\epsilon).
\end{equation}
Notice in particular that $ \mu_1(\gamma,t_0,\epsilon) \to 0$ for $(t_0,\epsilon) \to (1,0)$.
\end{lem}
\begin{proof}
As $\sigma_1:=\sigma_1(X) \geq t_0$ and $\sigma_1^2 + \dots + \sigma_m^2 \leq 1$, we have also $\sum_{j=2}^m \sigma_j^2 \leq 1- t^2_0$.
By the assumption $\|X-\tilde X\|_F\le\epsilon$ and by the well known Mirsky's bound we have that $\|\Sigma- \tilde \Sigma\|_F \leq \epsilon$.
 
Hence, $\tilde \sigma_1:=\sigma_1(\tilde X) \geq t_0 - \epsilon$, $\tilde \sigma_j:=\sigma_j(\tilde X) \leq \sqrt{1-t_0^2} + \epsilon$  and
$$
|\tilde \sigma_1 - \sigma_j| \geq t_0 - \epsilon - \sqrt{1- t^2_0}:=\bar \alpha >0, \
$$
for all $j=2,\dots,m$. The positivity of $\bar \alpha>0$ comes from the assumption that $t_0=\frac{1}{\sqrt{2}}+\epsilon+\xi$.
Hence, by applying Wedin's bound, Theorem \ref{wedin} in Appendix, we easily obtain
\begin{equation}\label{pippo1}
\max\{ \|u_1 u_1^T - \tilde u_1  \tilde u_1^T \|_F,  \|v_1 v_1^T - \tilde v_1  \tilde v_1^T \|_F \} \leq \frac{2}{ t_0 - \epsilon - \sqrt{1- t^2_0}} \epsilon\le \frac{2\epsilon}{\xi}.
\end{equation}
The last inequality comes from $1<2(t_0-\epsilon-\xi)^2<(t_0-\epsilon-\xi)^2+t_0^2.$
For later use we notice already that for any unit-norm vectors $x,\tilde x \in \mathbb R^m$
$$
\|xx^T-\tilde x\tilde x^T\|_F^2=\|xx^T\|_F^2+\|\tilde x\tilde x^T\|_F^2-2\langle xx^T,\tilde x\tilde x^T\rangle_F=2(1-\langle x,\tilde x\rangle^2)
$$
and
$$
|\langle x,\tilde x\rangle|=\sqrt{1-\frac{\|xx^T-\tilde x\tilde x^T\|_F^2}{2}}\ge 1-\frac{\|xx^T-\tilde x\tilde x^T\|_F^2}{2}.
$$
If moreover $\langle x,\tilde x\rangle\ge 0$, we get
\begin{equation}\label{pippo2'}
\|x-\tilde x\|_2^2=2(1-\langle x,\tilde x\rangle)\le \|xx^T-\tilde x\tilde x^T\|_F^2.
\end{equation}

We now address \eqref{pippo0} by considering the estimates of different components of the singular value decompositions.
We start by comparing the first singular value components.
To simplify the notation, we set for $s=(s_1,\dots,s_m)$
$$\pi_\gamma(s)=\pi_\gamma(s_1,\dots,s_m) = 
\frac{\gamma s_1}{\sqrt{\gamma^2 s_1^2 + s_2^2 + \dots + s_m^2}}.$$
We first derive a bound for $\left \| u_1 \pi_\gamma(\sigma) v_1^T - 
\tilde u_1 \pi_\gamma(\tilde \sigma)\tilde v_1^T \right \|_F$, where $\sigma=(\sigma_1,\dots,\sigma_m)$ and similarly for $\tilde \sigma =(\tilde \sigma_1,\dots,\tilde \sigma_m)$.

For that, we need first to show the Lipschitz continuity of the function $s\to\pi_\gamma(s)$
\change{on the set $S=\{s\in\R^m:\|s\|_2\le1 \ \text{and}\ s_1>t_0-\epsilon\}$.}
From
\begin{align*}
|\partial_{s_1} \pi_\gamma(s_1,\dots,s_m)| &= \left | \frac{\gamma(s_2^2+\dots+s_m^2)}{(\gamma^2 s_1^2 + s_2^2+\dots+s_m^2)^{3/2}} \right |
\intertext{and}
|\partial_{s_j} \pi_\gamma(s_1,\dots,s_m)| &= \left | \frac{\gamma s_1 s_j}{(\gamma^2 s_1^2 + s_2^2+\dots+s_m^2)^{3/2}} \right |\change{\quad\text{for} \quad j\in\{2,\dots,m\}},
\end{align*}
we obtain for $s\in S$ that $1\geq s_1>t_0 - \epsilon$, $s_2^2+\dots+s_m^2 \leq 1-(t_0-\epsilon)^2\le 1/2$ and
\begin{align*}
\|\nabla \pi_{\gamma}(s)\|^2_2&= \frac{\gamma^2(s_2^2+\dots+s_m^2)^2+\gamma^2s_1^2(s_2^2+\dots+s_m^2)}{(\gamma^2 s_1^2 + s_2^2+\dots+s_m^2)^{3}}\\
&=\frac{\gamma^2(s_2^2+\dots+s_m^2)(s_1^2+s_2^2+\dots+s_m^2)}{(\gamma^2 s_1^2 + s_2^2+\dots+s_m^2)^{3}}
\le \frac{\gamma^2/2}{(\gamma^2 s_1^2 + s_2^2+\dots+s_m^2)^{3}}\le \frac{1}{2\gamma^4s_1^6}.
\end{align*}
\change{As $\sigma,\tilde\sigma\in S$ and $S$ is a convex set, we obtain by the mean value theorem,
\begin{align*}
| \pi_\gamma(\sigma) - \pi_\gamma(\tilde \sigma)|
&\leq \|\nabla \pi_{\gamma}(s)\|_2\cdot\|\sigma-\tilde\sigma\|_2 \le \frac{\epsilon}{\sqrt{2}\gamma^2 s_1^3}\le \frac{\epsilon}{\sqrt{2}\gamma^2(t_0-\varepsilon)^3}. 
\end{align*}}

As the signs of the singular vectors can be chosen arbitrarily, we can assume without loss of generality that $\langle u_1,\tilde u_1\rangle\ge0.$
Together with \eqref{pippo1} and \eqref{pippo2'} we obtain $\|u_1 - \tilde u_1\|_2 \leq \frac{2 \epsilon}{\xi}$
and the same holds also for $\|v_1-\tilde v_1\|_2$. Therefore, we may estimate the difference of the first singular value components by
\begin{align}
\left \| u_1 \pi_\gamma(\sigma) v_1^T -  \tilde u_1 \pi_\gamma(\tilde\sigma) \tilde v_1^T \right \|_F
&\leq \left \| (u_1 -\tilde u_1)\pi_\gamma(\sigma) v_1^T \right \|_F \nonumber
+\left \| \tilde u_1 ( \pi_\gamma(\sigma) - \pi_\gamma(\tilde\sigma)) v_1^T \right \|_F\nonumber \\
&\phantom{XXXX} +\left \| \tilde u_1 \pi_\gamma(\tilde\sigma) (v_1^T-\tilde v_1^T) \right \|_F  \nonumber \\
&\le \| u_1 -\tilde u_1\|_2+|\pi_\gamma(\sigma) - \pi_\gamma(\tilde\sigma)|+\| v_1 -\tilde v_1\|_2  \nonumber \\
&\le \frac{2 \epsilon}{\xi}+ 
\change{\frac{\epsilon}{\gamma^{2} (t_0-\epsilon)^3}}+\frac{2\epsilon }{\xi}  
\le 2^{3/2}\epsilon+\frac{4\epsilon }{\xi}. \label{eq:estim1}
\end{align}
We now need to estimate the difference of the other components of the singular value decomposition. Now notice that
\begin{equation}\label{eq:alg:2}
\left\|\sum_{j=1}^ky_jz_j^T\right\|_F^2=\sum_{j=1}^k\|y_j\|_2^2
\end{equation}
for arbitrary vectors $\{y_1,\dots,y_k\}\subset\R^m$ and orthonormal vectors $\{z_1,\dots,z_k\}\subset\R^m.$ By applying the triangle inequality and \eqref{eq:alg:2}
\begin{eqnarray}
&& \left \|  \sum_{j=2}^m  u_j \frac{\sigma_j}{\sqrt{\gamma^2  \sigma_1^2 +  \sigma_2^2 + \dots + \sigma_m^2}}  v_j^T - 
\sum_{j=2}^m  \tilde u_j \frac{\tilde \sigma_j}{\sqrt{\gamma^2 \tilde \sigma_1^2 + \tilde \sigma_2^2 + \dots + \tilde \sigma_m^2}} \tilde v_j^T \right \|_F \nonumber \\
&\le & \left \|  \sum_{j=2}^m  u_j \frac{\sigma_j}{\sqrt{\gamma^2  \sigma_1^2 +  \sigma_2^2 + \dots + \sigma_m^2}}  v_j^T\right\|_F 
+\left\|\sum_{j=2}^m  \tilde u_j \frac{\tilde \sigma_j}{\sqrt{\gamma^2 \tilde \sigma_1^2 + \tilde \sigma_2^2 + \dots + \tilde \sigma_m^2}} \tilde v_j^T \right \|_F \nonumber \\
&=&\Bigl(\sum_{j=2}^m \frac{\sigma^2_j}{\gamma^2  \sigma_1^2 +  \sigma_2^2 + \dots + \sigma_m^2}\Bigr)^{1/2}
+\Bigl(\sum_{j=2}^m\frac{\tilde \sigma^2_j}{\gamma^2 \tilde \sigma_1^2 + \tilde \sigma_2^2 + \dots + \tilde \sigma_m^2}\Bigr)^{1/2} \nonumber  \\
&\le&\Bigl(\frac{1-t_0^2}{\gamma^2t_0^2}\Bigr)^{1/2}+\Bigl(\frac{1-(t_0-\epsilon)^2}{\gamma^2(t_0-\epsilon)^2}\Bigr)^{1/2}
\le 2\frac{\sqrt{1-(t_0-\epsilon)^2}}{\gamma(t_0-\epsilon)} \leq 2\sqrt{1-(t_0-\epsilon)}, \label{eq:estim2}
\end{eqnarray}
as $t_0-\epsilon>1/\sqrt{2}$ and $\frac{\sqrt{1-u^2}}{u}\le 2\sqrt{1-u}$ for $1>u>1/\sqrt{2}$. The statement now follows by adding \eqref{eq:estim1} and \eqref{eq:estim2}.
\end{proof}

\begin{thm}\label{thm:recovery2}
Assume that $\|P_{\widetilde \A} - P_{\A}\|_{F \to F}< \epsilon <1$ and that $a_1,\dots,a_m$ are orthonormal. 
Let $\|X^0\| > \max \{\frac{1}{\sqrt{\gamma^2-1}},\frac{1}{\sqrt{2}}+\epsilon+\xi \}$ and $\sqrt{2}<\gamma$.
Then for the iterations $(X^\ell)_{\ell \in \mathbb N}$ produced by Algorithm \ref{alg7}, there exists $\mu_0<1$ such that
\begin{equation*}
\limsup_{\ell \rightarrow \infty} |1- \|X^\ell\||  \leq \frac{\mu_1(\gamma,t_0,\epsilon) + 2 \epsilon}{1- \mu_0} +\epsilon,
\end{equation*}
where $\mu_1(\gamma,t_0,\epsilon)$ is as in Lemma \ref{lipcont}. 
The sequence $(X^\ell)_{\ell \in \mathbb N}$ is bounded and its accumulation points $\bar X$ satisfy simultaneously the following properties
$$
\| \bar X \|_F \leq 1 \mbox{ and } \| \bar X \| \geq 1- \frac{\mu_1(\gamma,t_0,\epsilon) + 2 \epsilon}{1- \mu_0} -\epsilon,
$$
and
$$
\| P_{\A} \bar X \|_F \leq 1 \mbox{ and } \| P_{\A} \bar X \| \geq 1- \frac{\mu_1(\gamma,t_0,\epsilon) + 2 \epsilon}{1- \mu_0}-2\epsilon.
$$
\end{thm}

\begin{proof}
We denote the singular value decomposition of $X^\ell$ by $X^\ell = \sum_{j=1}^m \change{\tilde\sigma_j^\ell} u_j^\ell \otimes \tilde v_j^\ell$
and the one of $P_{\A} X^\ell$ by $P_{\A} X^\ell =\sum_{j=1}^m  \sigma_j^\ell a_{i_j} \otimes  a_{i_j}$, where
$i_j$ is a suitable rearrangement of the index set $\{1,\dots,m\}$.
By Lemma \ref{lipcont} we can further develop the following estimates
$$
\change{\|X^{\ell+1}\|\le \|P_{\A}X^{l+1}\|+\|(P_{\widetilde \A}-P_{\A})X^{\ell+1}\|\le \sigma_1^{\ell+1}+\|(P_{\widetilde \A}-P_{\A})X^{\ell+1}\|_F\le \sigma_1^{\ell+1}+\epsilon}
$$
and
\begin{eqnarray*}
\|X^{\ell+1}\| &=& \| P_{\widetilde \A} \Pi_\gamma(X^\ell)\| \geq \| P_{\A} \Pi_\gamma(X^\ell)\| - \|  (P_{\widetilde \A} -P_{\A}) \Pi_\gamma(X^\ell)\|\\
&\geq & \| P_{\A} \Pi_\gamma(X^\ell)\| - \|  (P_{\widetilde \A} -P_{\A}) \Pi_\gamma(X^\ell)\|_F \geq \| P_{\A} \Pi_\gamma(P_{\widetilde \A} X^\ell)\| - \epsilon\\
&=& \| P_{\A} \Pi_\gamma(P_{\A} X^\ell) + P_{\A} \Pi_\gamma(P_{\widetilde \A} X^\ell) -  P_{\A} \Pi_\gamma(P_{\A} X^\ell) \| - \epsilon \\
&\geq& \| P_{\A} \Pi_\gamma(P_{\A} X^\ell) \| - \| \Pi_\gamma(P_{\widetilde \A} X^\ell) -  \Pi_\gamma(P_{\A} X^\ell) \|_F - \epsilon\\
&\geq& \frac{\gamma \sigma_1^\ell}{\sqrt{(\gamma^2-1) {\sigma_1^\ell}^2 + 1}}-\mu_1(\gamma,t_0,\epsilon) - \epsilon.
\end{eqnarray*}
Hence, we obtain
$$
1- \sigma_1^{\ell+1} \leq 1- \frac{\gamma \sigma_1^\ell}{\sqrt{(\gamma^2-1) {\sigma_1^\ell}^2 + 1}} + \mu_1(\gamma,t_0,\epsilon) +2 \epsilon.
$$
By an estimate similar to \eqref{iter:2} and following the arguments given before, we conclude that 
\begin{equation}\label{recursion}
1- \sigma_1^{\ell+1} \leq \mu_0 (1- \sigma_1^\ell) + \eta_0,
\end{equation}
where $\eta_0 = \mu_1(\gamma,t_0,\epsilon) +2 \epsilon$. As $X^\ell = P_{\widetilde \A} \Pi_\gamma(X^{\ell-1})$ and $P_{\A}$ and $P_{\widetilde \A}$ are orthogonal projections we have that
$$
\sigma_1^\ell \leq \| P_\A X^\ell\|_F \leq  \| X^\ell\|_F \leq \|  \Pi_\gamma(X^{\ell-1})\|_F =1.
$$
Hence, actually, the recursion \eqref{recursion} can be rewritten as
\begin{eqnarray*}
|1- \sigma_1^{\ell+1}| &\leq& \mu_0 \left |  1- \sigma_1^\ell \right | + \eta_0\\
&\leq& \mu_0^{\ell+1} \left |  1- \sigma_1^0 \right | + \eta_0 \sum_{k=0}^\ell \mu_0^k.
\end{eqnarray*}
This implies 
\begin{equation}\label{limsup}
\limsup_{\ell \rightarrow \infty} |1- \|X^\ell\|| = \limsup_{\ell \rightarrow \infty} |1- \tilde \sigma_1^\ell| \leq \limsup_{\ell \rightarrow \infty} |1- \sigma_1^{\ell}| +\epsilon \leq \frac{\eta_0}{1- \mu_0} +\epsilon.
\end{equation}
Since the sequence $(X^\ell)_\ell$ is bounded, it has accumulation points $\bar X$, and as a consequence of \eqref{limsup} we obtain that $\bar X$ has simultaneously the following properties
$$
\| \bar X \|_F \leq 1 \mbox{ and } \| \bar X \| \geq 1- \frac{\eta_0}{1- \mu_0} -\epsilon,
$$
and
$$
\| P_\A \bar X \|_F \leq 1 \mbox{ and } \| P_\A \bar X \| \geq 1- \frac{\eta_0}{1- \mu_0}-2\epsilon.
$$
\end{proof}
\begin{rem}
Given the singular value decompositions $\bar X= \sum_{j=1}^m \bar \sigma_j \bar u_j \otimes \bar v_j$ and $ P_\A \bar X = \sum_{j=1}^m \sigma_j^\infty a_{i_j} \otimes a_{i_j}$,
by applying again Wedin's bound we obtain that, for instance
$$
\|\bar v_1 \otimes \bar v_1 - a_{i_1} \otimes  a_{i_1}\|_F \leq \frac{2}{1- \frac{\eta_0}{1- \mu_0} +\epsilon - \sqrt{1- (1- \frac{\eta_0}{1- \mu_0})^2}} \epsilon.
$$
Notice that for $\epsilon \to 0$ we obtain $\eta_0= \mu_1(\gamma,t_0,\epsilon) +2 \epsilon \to 2\sqrt{1-t_0}$.
\end{rem}


\section{Recovering the activation functions}\label{sec:funct}
{The main aim of our work is to identify the structure
of functions, which take the form of \eqref{sumsridge}. 
Nevertheless, once the ridge directions $a_j$ are identified or approximated,
we can produce also a uniform approximation of $f$.

Before we come to that we clarify one technical issue of \eqref{sumsridge}.
It is easy to see, that the representation \eqref{sumsridge} is not unique due to the free choice of additive factors.
Indeed, if we add to the profiles $(g_j)_{j=1}^m$ arbitrary constants which sum up to zero, we obtain the same function $f$.
By simply sampling $f$ at zero and subtracting this value from $f$,
we may assume without loss of generality that $f(0)=0$. 
If $0= f(0) = \sum_{i=1}^m g_i(0)$ then $\{g_i(0):i=1,\dots,m\}$ are indeed constants with zero sum and
we can subtract them term by term $f(x) = \sum_{i=1}^m (g_i(\langle a_i, x\rangle)- g_i(0))$. Consequently, we can assume without loss of generality that
\begin{equation}\label{eq:g0}
g_1(0)=\dots=g_m(0)=0.
\end{equation}
For the uniform approximation of $f$ fulfilling \eqref{sumsridge} and \eqref{eq:g0}, let us assume the we run Algorithm \ref{alg7} with different initial values
and obtain the approximation of the ridge directions $(a_j)_{j=1}^m$ by unit-norm vectors $(\hat a_j)_{j=1}^m$.
We then sample $f$ along the vectors in the dual basis $(\hat b_j)_{j=1}^m$ to obtain an approximation of the
univariate ridge profiles $g_1,\dots,g_m$, which are uniquely determined by \eqref{eq:g0}.
The approximation $\hat f$ of $f$ is then obtained by putting all these
ingredients together.
The resulting algorithm and the analysis of its performance are described below.}
\vskip.3cm
\fbox{
\begin{minipage}{13.6cm}
\begin{algorithm}\label{alg8}
\emph{\begin{itemize}
\item Let $\hat a_j$ be the normalized approximations of $a_j,j=1,\dots,m$.
\item Let $(\hat b_j)_{j=1}^m$ be the dual basis to $(\hat a_j)_{j=1}^m$.
\item Put $\hat g_j(t):=f(t\hat b_j)$, $t\in (-1/\|\hat b_j\|_2,1/\|\hat b_j\|_2)$.
\item Put $\displaystyle\hat f(x):=\sum_{j=1}^m \hat g_j(\langle\hat a_j,x\rangle), \|x\|_2\le 1$.
\end{itemize}}
\end{algorithm}
\end{minipage}
}\vskip.3cm

We first start with an auxiliary result, which can be shown by a simple direct computation.
\begin{lem}\label{lem:taylor}
Let $I\subset \R$ be an interval containing zero and let $G:I\to \R$ be measurable. Then for any $x,y\in I$
$$
\Bigl|\int_0^x (x-u)G(u)du-\int_0^y(y-u)G(u)du\Bigr|\le \max_{u\in I}{|G(u)|}\cdot\Bigl(|x|\cdot|y-x|+|y-x|^2/2\Bigr).
$$
\end{lem}

\change{The performance of Algorithm \ref{alg8} is then described be the following theorem.}

\begin{thm}\label{identact}
Let ${\mathcal S}(a_1,\dots,a_m)\le \varepsilon$, ${\mathcal S}(\hat a_1,\dots,\hat a_m)\le \varepsilon'$,
and $\Bigl(\sum_{j=1}^m\|a_j-\hat a_j\|^2_2\Bigr)^{1/2}\le\eta$. Then
$\hat f$ constructed by Algorithm \ref{alg8} satisfies
$$
\|f-\hat f\|_\infty\le 5C_2(1+\xi(\varepsilon,\varepsilon'))\max(\eta,\eta^2),
$$
where $\xi(\varepsilon,\varepsilon')\to 0$ if $(\varepsilon,\varepsilon')\to (0,0).$
\end{thm}
\begin{proof}
We use that $g_i(0)=0$ for $i=1,\dots,m$, $\langle a_i, x\rangle=\sum_{j=1}^m \langle \hat a_j, x\rangle \cdot \langle\hat b_j,a_i\rangle,$
 Taylor's formula, and estimate for $x\in\R^m$ with $\|x\|_2\le 1$
\begin{align*}
|f(x)-\hat f(x)|&=\Bigl|\sum_{i=1}^mg_i(\langle a_i, x\rangle)-\sum_{j=1}^m\hat g_j(\langle\hat a_j, x\rangle)\Bigr|=
\Bigl|\sum_{i=1}^mg_i(\langle a_i, x\rangle)-\sum_{j=1}^m f(\langle \hat a_j, x\rangle\hat b_j)\Bigr|\\
&=\Bigl|\sum_{i=1}^mg_i(\langle a_i,x\rangle)-\sum_{j=1}^m \sum_{i=1}^m g_i(\langle \hat a_j,x\rangle \cdot\langle\hat b_j, a_i\rangle)\Bigr|\\
&\le\sum_{i=1}^m \Bigl|g_i(\langle a_i,x\rangle)- \sum_{j=1}^m g_i\bigl(\langle\hat a_j,x\rangle\cdot\langle\hat b_j, a_i\rangle\bigr)\Bigr|\\
&=\sum_{i=1}^m\Bigl|g_i'(0)\,\langle a_i,x\rangle-\sum_{j=1}^mg_i'(0)\langle\hat a_j, x\rangle\cdot\langle \hat b_j, a_i\rangle\\
&\qquad+\int_0^{\langle a_i,x\rangle}(\langle a_i,x\rangle-u)g_i''(u)du-\sum_{j=1}^m\int_0^{\langle \hat a_j, x\rangle\cdot\langle\hat b_j, a_i\rangle}
\bigl(\langle\hat a_j, x\rangle\cdot\langle\hat b_j, a_i\rangle-u\bigr)g_i''(u)du\Bigr|\\
&\le \sum_{i=1}^m\Bigl|\int_0^{\langle a_i, x\rangle}(\langle a_i, x\rangle-u)g_i''(u)du-\int_0^{\langle \hat a_i,x\rangle\cdot\langle\hat b_i,a_i\rangle}\bigl(\langle\hat a_i, x\rangle\cdot\langle\hat b_i, a_i\rangle-u\bigr)g_i''(u)du\Bigr|\\
&\qquad +\sum_{i=1}^m\sum_{j\not=i}\Bigl|\int_0^{\langle \hat a_j, x\rangle\cdot\langle \hat b_j,a_i-\hat a_i\rangle}
\bigl(\langle\hat a_j,x\rangle\cdot\langle\hat b_j, a_i-\hat a_i\rangle-u\bigr)g_i''(u)du\Bigr|\\
&=I+II.
\end{align*}
We use Lemma \ref{lem:taylor} and Lemma \ref{lem:A1} to bound the first term by
\begin{align*}
I&\le C_2 \sum_{i=1}^m\Bigl\{ |\langle a_i,x\rangle|\cdot |\langle a_i,x\rangle-\langle\hat a_i,x\rangle\cdot\langle\hat b_i,a_i\rangle|+
|\langle a_i,x\rangle-\langle\hat a_i, x\rangle\cdot\langle \hat b_i, a_i\rangle|^2/2\Bigr\}\\
&\le C_2\Bigl(\sum_{i=1}^m \langle a_i,x\rangle^2\Bigr)^{1/2}\cdot\Bigl(\sum_{i=1}^m|\langle a_i, x\rangle-\langle \hat a_i,x\rangle\cdot\langle \hat b_i,a_i\rangle|^2\Bigr)^{1/2}\\
&\qquad+\frac{C_2}{2}\sum_{i=1}^m|\langle a_i,x\rangle-\langle \hat a_i, x\rangle\cdot\langle \hat b_i, a_i\rangle|^2=I'+I'',
\end{align*}
where
\begin{align*}
I'&\le C_2 (1+\varepsilon)\cdot\Bigl[
\Bigl(\sum_{i=1}^m\langle a_i-\hat a_i, x\rangle^2\Bigr)^{1/2}+
\Bigl(\sum_{i=1}^m|\langle \hat a_i, x\rangle\cdot\langle \hat b_i, \hat a_i-a_i\rangle|^2\Bigr)^{1/2}\Bigr]\\
&\le C_2(1+\varepsilon)\eta+C_2(1+\varepsilon)\max_j \|\hat  b_j\|_2\,\eta
\end{align*}
and
\begin{align*}
I''
&\le C_2\sum_{i=1}^m\Bigl(\langle a_i-\hat a_i, x\rangle^2+|\langle \hat a_i, x\rangle\cdot\langle \hat b_i, \hat a_i-a_i\rangle|^2\Bigr)\\
&\le C_2\,\eta^2+C_2\max_j\|\hat b_j\|_2^2\eta^2.
\end{align*}
Next, we estimate the second term by
\begin{align*}
II&\le C_2\sum_{i=1}^m\sum_{j=1}^m |\langle \hat a_j, x\rangle\cdot\langle \hat b_j,a_i-\hat a_i\rangle|^2
\le C_2\sum_{i,j=1}^m \langle \hat a_j, x\rangle^2\cdot\|\hat b_j\|_2^2\cdot \|a_i-\hat a_i\|_2^2\\
&\le C_2\max_{j}\|\hat b_j\|_2^2\cdot \sum_{j=1}^m \langle\hat a_j, x\rangle^2\cdot\sum_{i=1}^m \|a_i-\hat a_i\|_2^2
\le C_2\max_{j}\|\hat b_j\|_2^2\cdot (1+\varepsilon')^2 \eta^2.
\end{align*}
Using Lemma \ref{lem:A1} (vi) and summing up these estimates we get
$$
\|f-\hat f\|_\infty\le 5C_2(1+\xi(\varepsilon,\varepsilon'))\max(\eta,\eta^2),
$$
where $\xi(\varepsilon,\varepsilon')\to 0$ if $(\varepsilon,\varepsilon')\to (0,0).$
\end{proof}
{
\begin{rem} By triangle inequality, the parameters $\varepsilon,\varepsilon'$ and $\eta$ from Theorem \ref{identact} satisfy $\varepsilon'\le \varepsilon+\eta.$
\end{rem}
}

\section{\rechange{Whitening}}\label{whitening}

\rechange{In Section \ref{locmaxima} we discussed the identification of weights $\{a_i,i=1,\dots,m\}$
under the condition that they are close to an orthonormal system.
In this section, we prove that this assumption is without loss of generality.
As we clarify in this section, if the accuracy of the approximation $\widetilde \A \approx \A$ is high enough, then,
also  for systems of vectors $\{a_i:i=1,\dots,m\}$, which are not $\varepsilon$-nearly-orthonormal, there is a constructive way,
the {\it whitening process} we describe below, to render them $\varepsilon$-nearly-orthonormal.
This procedure is very much inspired by the ones described in \cite{angeja,jasean,kolda} for symmetric tensors.
Again, differently from \cite{jasean}, we will not rely on one instance matrix/tensor, but rather search  within the space $\widetilde \A$
for the right whitening matrix with the necessary stability properties.}


\subsection{Exact whitening}
In this section we explain how we can reduce our analysis to systems $a_1, \dots, a_m \in \mathbb R^m$  of $\varepsilon$-nearly-orthonormal vectors.
Assume for the moment $a_1, \dots, a_m \in \mathbb R^m$ linearly independent unit vectors, but not necessarily orthonormal. We describe below
a quite standard orthonormalization procedure, also called in recent literature  whitening \cite{angeja,kolda} in the context of symmetric tensor decompositions.
It relies on positive definite matrices from the subspace ${\mathcal A}=\linspan\{a_i\otimes a_i:i=1,\dots,m\}$, which can be easily characterized.

\change{
\begin{lem}\label{lem:pos}
Let $A$ be a $m\times m$ matrix with non-zero columns $a_1,\dots,a_m$ and let $D_\lambda$ be a diagonal matrix with
real numbers $\lambda_1,\dots,\lambda_m$ on the diagonal. Then the matrix
$$
G=AD_\lambda A^T=\sum_{i=1}^m\lambda_i a_i\otimes a_i
$$
is positive definite if, and only if, $\{a_1,\dots,a_m\}$ are linearly independent and $\lambda_i>0$ for all $i=1,\dots,m.$
\end{lem}
\begin{proof} Let $\{a_i:i=1,\dots,m\}$ be linearly independent and let $\lambda_i>0$ for all $i=1,\dots,m.$
Then, for all $x\in\R^m$ with $x\not=0$,
$$
x^TGx=\sum_{i=1}^m\lambda_i \langle x,a_i\rangle^2>0. 
$$

If, on the other hand, $G$ is positive definite, then $\operatorname{rank} A=m$ and $a_1,\dots,a_m$ are linearly independent.
Furthermore, $\lambda_i=x^TGx>0$ for $x=A(A^TA)^{-1}e_i$ for all $i=1,\dots,m.$ 
\end{proof}}

\begin{prop}\label{whiten}
Assume we are given a symmetric and positive definite matrix
\begin{equation}\label{keymatrix}
 G = \sum_{i=1}^m \lambda_i a_i \otimes a_i,
\end{equation}
and its singular value decomposition
$$
G= U D U^T,
$$
where $U$ is an orthogonal matrix and $D$ is diagonal matrix with positive diagonal values.
If we denote $W = D^{-\frac{1}{2}} U^T$ the so-called \emph{whitening} matrix, then 
the system of vectors $\{ \sqrt {\lambda_i} W a_i: i=1,\dots,m \}$ defines an orthonormal basis and
$$
I_m=W G W^T = \sum_{i=1}^m \lambda_i W a_i \otimes W a_i
$$
is an orthogonal resolution of the identity.
\end{prop}
\begin{proof}
We know that $G\in {\mathcal A}$ can be written as
\begin{equation}\label{eq:Review1}
G=U DU^T=\sum_{i=1}^m \lambda_i a_i\otimes a_i=AD_{\lambda}A^T,
\end{equation}
where $A\in{\mathbb R}^{m\times m}$ is a matrix with columns $a_1,\dots,a_m$ and $D_{\lambda}$ is a diagonal matrix with
$\lambda_1,\dots,\lambda_m$ on the diagonal. By Lemma \ref{lem:pos}, $\lambda_i>0$ for all $i=1,\dots,m.$
Let $W:=D^{-1/2}U^T$. The matrix with columns $\{\sqrt{\lambda_i}Wa_i:i=1,\dots,m\}$
coincides with $WAD_{\sqrt{\lambda}}$, where $D_{\sqrt{\lambda}}$ is a diagonal matrix with $\sqrt{\lambda_1},\dots,\sqrt{\lambda_m}$ on the diagonal.
Finally, we observe that
\begin{align*}
 (WAD_{\sqrt{\lambda}})(WAD_{\sqrt{\lambda}})^T&=WAD_{\sqrt{\lambda}}D_{\sqrt{\lambda}}A^TW^T=WAD_{\lambda}A^TW^T\\
&=WGW^T=(D^{-1/2}U^T)(UDU^T)(D^{-1/2}U^T)^T=I_m,
\end{align*}
hence $WAD_{\sqrt{\lambda}}$ is an orthonormal matrix.
\end{proof}
\subsection{Perturbed whitening}
Unfortunately in practice we cannot not access directly $\mathcal A = \operatorname{span} \{a_i \otimes a_i:i=1, \dots,m \} \subset \mathbb R^{m \times m}$, and therefore it is not possible in general  to 
construct a matrix $G$ as in \eqref{keymatrix}. However, the results of Section \ref{princhess} allow us to access an approximating space of symmetric matrices $\widetilde \A \subset \mathbb R^{m \times m}$
and in the following we assume that 
\begin{equation}\label{projestt}
\| P_\A - P_{\widetilde \A}\|_{F \to F} \leq \eta.
\end{equation}
We assume that we can construct $\widetilde G\in \widetilde {\mathcal A}$,
which is positive definite and define $G=P_{{\mathcal A}}\widetilde G.$
The existence of a positive definite $\widetilde G\in \widetilde {\mathcal A}$
and algorithmic ways to construct it are discussed in Section \ref{findmatrix} below.
We consider their spectral decompositions
\begin{equation}\label{eq:sing_decomp}
\widetilde G=\widetilde U\widetilde D\widetilde U^T\quad\text{and}\quad G=UDU^T.
\end{equation}
As $G\in {\mathcal A}$, it can be again written as in \eqref{eq:Review1}.
\rechange{If $\eta>0$ is small enough, we show that} $G$ is also positive definite, i.e., that  $\lambda_i>0$ for all $i=1,\dots,m.$
We define again $W:=D^{-1/2} U^T$ and its perturbed version $\widetilde W:=\widetilde D^{-1/2}\widetilde U^T.$
Using this notation together with \eqref{eq:Review1} and \eqref{eq:sing_decomp}, we can
quantify the effect of whitening.
\begin{thm}\label{thm:review1}
Let $\gamma,\eta>0$ be positive real numbers. Let $\|P_{\mathcal A}-P_{\widetilde {\mathcal A}}\|_{F\to F}\le \eta$
and \rechange{let $\widetilde G\in\widetilde {\mathcal A}$ 
be positive definite with $\widetilde G\succcurlyeq \gamma I_m$. If $\eta \|\widetilde G\|_F<\gamma$, then
$G=P_{{\mathcal A}}(\widetilde G)$ is also positive definite},
$$
{\mathcal S}(\sqrt{\lambda_1}\widetilde Wa_1,\dots,\sqrt{\lambda_m}\widetilde Wa_m)\le \frac{\eta \|\widetilde G\|_F}{\gamma}
$$
and $\left \{\frac{\widetilde Wa_1}{\|\widetilde W a_1\|_2},\dots,\frac{\widetilde Wa_m}{\|\widetilde Wa_m\|}\right \}$ are $\varepsilon$-nearly-orthonormal, for $\varepsilon=\frac{\sqrt{2}\eta \|\widetilde G\|_F}{\gamma}$, i.e.,
$$
{\mathcal S}\Bigl(\frac{\widetilde Wa_1}{\|\widetilde W a_1\|_2},\dots,\frac{\widetilde Wa_m}{\|\widetilde Wa_m\|_2}\Bigr)\le \frac{\sqrt{2}\eta \|\widetilde G\|_F}{\gamma}=:\varepsilon.
$$
\end{thm}
\begin{proof}
\rechange{
We use the estimate
\[
\|G-\widetilde G\|\le \|G-\widetilde G\|_{F}=\|(P_{\A}-P_{\widetilde \A})(\widetilde G)\|_F\le\eta\|\widetilde G\|_F
\]
to show that $x^TGx=x^T(G-\widetilde G)x+x^T\widetilde Gx\ge \gamma-\|G-\widetilde G\|\ge \gamma -\eta\|\widetilde G\|_F>0$ for every $x\in\R^m$
with $\|x\|_2=1.$ This gives that $G$ is positive definite.}

Next, we observe that
\begin{align*}
(\widetilde WAD_{\sqrt{\lambda}})\cdot (\widetilde WAD_{\sqrt{\lambda}})^T&=\widetilde WAD_{\lambda} A^T\widetilde W^T=\widetilde WG\widetilde W^T
\end{align*}
and
$$
\widetilde W\widetilde G\widetilde W^T=(\widetilde D^{-1/2}\widetilde U^T)\widetilde U\widetilde D\widetilde U^T(\widetilde D^{-1/2}\widetilde U^T)^T=I_m.
$$
Hence
\begin{align*}
\|(\widetilde WAD_{\sqrt{\lambda}})\cdot (\widetilde WAD_{\sqrt{\lambda}})^T-I_m\|_F&=\|\widetilde WG\widetilde W^T-\widetilde W\widetilde G\widetilde W^T\|_F=
\|\widetilde W(G-\widetilde G)\widetilde W^T\|_F\\
&=\|\widetilde D^{-1/2}\widetilde U^T(G-\widetilde G)\widetilde U\widetilde D^{-1/2}\|_F\le \|\widetilde D^{-1}\|\cdot \|G-\widetilde G\|_F,
\end{align*}
which, by Theorem \ref{thm:A1}, gives the same estimate also for ${\mathcal S}(\sqrt{\lambda_1}\widetilde Wa_1,\dots,\sqrt{\lambda_m}\widetilde W_m)$.
The second assertion then follows simply by Lemma \ref{lem:2S}.
\end{proof}
}

{\subsection{Finding positive definite matrices}\label{findmatrix}

In view of Theorem \ref{thm:review1}, we are interested in the following optimization problem.
Given an $m$-dimensional subspace $\widetilde {\mathcal A}\subset{\mathbb R}^{m\times m}$ of $m\times m$ symmetric matrices,
we would like to answer two questions:
\begin{itemize}
\item[(i)] Does $\widetilde {\mathcal A}$ contain a strictly positive matrix?
\item[(ii)] And, if this is the case, which positive definite matrix in $\widetilde {\mathcal A}$ achieves the smallest ratio between
its Frobenius norm and its smallest eigenvalue?
\end{itemize}

Both these tasks can be solved by the following max-min problem
\begin{equation}\label{eq:maxmin1}
\max_{\substack{\widetilde A\in\widetilde{\mathcal A} \\ \|\widetilde A\|_F=1}}\ \min_{\substack{x\in{\mathbb R}^m \\ \|x\|_2=1}}x^T\widetilde Ax.
\end{equation}
Indeed, the maximizer of \eqref{eq:maxmin1} is the matrix from $\widetilde{\mathcal A}$, which has the largest minimal eigenvalue
among the matrices in $\widetilde{\mathcal A}$, which have unit Frobenius norm. Furthermore, if the value of \eqref{eq:maxmin1}
is zero or negative, there are no positive definite matrices in $\widetilde{\mathcal A}.$

\begin{thm}\label{thm:convex} Let $\widetilde {\mathcal A}\subset{\mathbb R}^{m\times m}$ be a subspace of $m\times m$ symmetric matrices.
Let 
\begin{equation*}
\ell(\widetilde A):=\min_{\substack{x\in{\mathbb R}^m\\\|x\|_2=1}}x^T\widetilde Ax
\end{equation*}
denote the minimal eigenvalue of $\widetilde A\in\widetilde{\mathcal A}.$ Then $\widetilde A\to -\ell(\widetilde A)$
is a convex function. Furthermore, the solution of the convex minimization problem
\begin{equation}\label{eq:maxmin5}
\alpha:=\min_{\substack{\widetilde A\in\widetilde{\mathcal A}\\\|\widetilde A\|_F\le 1}}(-\ell)(\widetilde A)
\end{equation}
satisfies $\alpha\le 0$ with $\alpha=0$ if, and only if, $\widetilde {\mathcal A}$ does not contain any strictly positive definite matrix.
If $\alpha<0$, then the minimizer $\widetilde A_0$ of \eqref{eq:maxmin5} lies on the sphere $\{\widetilde A\in\widetilde{\mathcal A}:\|\widetilde A\|_F=1\}$
and coincides with the solution of \eqref{eq:maxmin1}.
\end{thm}
\begin{proof}
 If $x\in{\mathbb R}^m$ with $\|x\|_2=1$ is fixed, then
\begin{equation*}
x^T\Bigl(\frac{\widetilde A+\widetilde B}{2}\Bigr)x=\frac{x^T\widetilde Ax+x^T\widetilde Bx}{2}\ge \frac{\ell(\widetilde A)+\ell(\widetilde B)}{2}.
\end{equation*}
Taking the infimum over $\|x\|_2=1$, we get $\ell(\widetilde A/2+\widetilde B/2)\ge \ell(\widetilde A)/2+\ell(\widetilde B)/2.$
This implies that the function $\widetilde A\to \ell(\widetilde A)$ is concave. Hence $-\ell$ is convex.
As $(-\ell)(0)=0$, we have $\alpha\le 0.$

If $\alpha=0$, then $(-\ell)(\widetilde A)\ge 0$ or, equivalently, $\ell(\widetilde A)\le 0$ for every $\widetilde A\in\widetilde {\mathcal A}$
and $\widetilde {\mathcal A}$ does not contain any strictly positive definite matrix.

If $\alpha<0$, then the minimizer of \eqref{eq:maxmin5} lies on the boundary of the optimization domain due to
$\ell(t\widetilde A)=t\ell(\widetilde A)$ for every $t>0$. Hence, in this case,
\begin{equation*}
\max_{\substack{\widetilde A\in\widetilde{\mathcal A}\\\|\widetilde A\|_F=1}}\, \min_{\substack{x\in{\mathbb R}^m\\\|x\|_2=1}}x^TAx=
\max_{\substack{\widetilde A\in\widetilde{\mathcal A}\\\|\widetilde A\|_F=1}}\ell(\widetilde A)
=-\min_{\substack{\widetilde A\in\widetilde{\mathcal A}\\\|\widetilde A\|_F=1}}(-\ell)(\widetilde A)
=-\min_{\substack{\widetilde A\in\widetilde{\mathcal A}\\\|\widetilde A\|_F\le 1}}(-\ell)(\widetilde A)=-\alpha.
\end{equation*}
\end{proof}
}

{\begin{rem}
Theorem \ref{thm:convex} translates \eqref{eq:maxmin1} into a convex optimization problem \eqref{eq:maxmin5}, cf. \cite{BEFB,BV}.
It can be solved in two steps. First, we want to decide if ${\widetilde{\mathcal A}}$ contains a strictly positive matrix.
If $\{\widetilde {A}_1,\dots,\widetilde A_m\}\subset {\widetilde{\mathcal A}}$ is any orthonormal basis of ${\widetilde{\mathcal A}}$, we would like to know if there is
a $\xi=(\xi_1,\dots,\xi_m)\in{\mathbb R}^m$, such that
\begin{equation}\label{eq:LMI}
\widetilde A=\xi_1\widetilde A_1+\dots+\xi_m\widetilde A_m \succ 0.
\end{equation}
This question is known as feasibility problem of the \emph{linear matrix inequality} \eqref{eq:LMI} and we refer to \cite[Section 11.4]{BV}
for a detailed discussion of its solution by interior-point methods. If \eqref{eq:LMI} turns out to be feasible, then we can use (for example)
an iterative projected subgradient method very much in the spirit of Section \ref{simpalg} to find the solution of \eqref{eq:maxmin1}.
\end{rem}
}

\vskip.5cm
\fbox{
\begin{minipage}{13.6cm}
\begin{algorithm}\label{alg5}
\rechange{
\emph{\begin{itemize}
\item Fix $\eta>0$, assume $f$ of the form \eqref{eq:sec2_sum}.
\item Denote $f^{(0)}(x):=\sum_{i=1}^m g_i^{(0)}(\langle a_i^{(0)},x\rangle)$ with $g_i^{(0)}:=g_i, a_i^{(0)}:=a_i$ for $i=1,\dots,m$.
\item For $k\ge 0$, compute $\widetilde \A^{(k+1)}$ by using Algorithm \ref{alg3} with accuracy $\eta>0$ from point values of $f^{(k)}$.
\item Define  $\widetilde W^{(k+1)}$ as the whitening matrix of 
$\widetilde \A^{(k+1)}$
using the solution of the optimization problem \eqref{eq:maxmin1}.
\item Denote $f^{(k+1)}(x):=f^{(k)}((\widetilde W^{(k+1)})^T x)$;\\
observe that $f^{(k+1)}(x)=\sum_{i=1}^m g_i^{(k+1)}(\langle a_i^{(k+1)}, x\rangle)$ as in \eqref{whitenedridge}
with $a^{(k+1)}_i :=  \widetilde W^{(k+1)} a_i^{(k)}/\|\widetilde W^{(k+1)} a_i^{(k)}\|_2$, $i=1, \dots,m$.
\end{itemize}}}
\end{algorithm}
\end{minipage}
}\vskip.3cm
{\subsection{Bootstrap whitening}

In view of the simple reformulation
\begin{equation}\label{whitenedridge}
f(\widetilde W^T x)= \sum_{i=1}^m g_i(\langle a_i,\widetilde W^T x\rangle)=\sum_{i=1}^m \tilde g_i(\sqrt{\tilde \lambda_i} \langle \widetilde W a_i,x\rangle)=\tilde f(x),
\end{equation}
for $\tilde g_i(t) = g_i(t/\sqrt{\tilde \lambda_i})$ and Theorem \ref{thm:review1}, we can further assume without loss of generality that the vectors
$\{a_i : i = 1, \dots, m\}$ are $\varepsilon$-nearly-orthonormal in first place. However, \rechange{in Section \ref{locmaxima} we needed} that $\varepsilon$ is indeed quite small
(certainly smaller than $1$ to ensure that our theoretical error estimates are meaningful). In view of Theorem  \ref{thm:review1}, this requires $\eta>0$ in the approximation \eqref{projestt} also rather small
and the identification of a reasonably well-conditioned matrix $\widetilde G$.

In this section we report surprising numerical results, obtained by iterating the whitening procedure (Algorithm \ref{alg5}). So far, we have not been able
to explain this phenomenon analytically, but it is consistently verified in all numerical experiments. It is related to the increasing possibility
over the iterations of finding a well-conditioned matrix $\widetilde G$ for whitening.

By applying whitening, we can assume through \eqref{whitenedridge} that the new function
$$
\tilde f(x) = \sum_{i=1}^m \tilde g_i(\sqrt{\tilde \lambda_i} \langle \widetilde W a_i, x\rangle)
$$ has ridge directions $\sqrt{\tilde \lambda_i} \widetilde W a_i$ which are ``more orthogonal'' than the original ones $a_i$ of
$f(x)= \sum_{i=1}^m g_i(\langle a_i, x\rangle)$. Still, when the distortion parameter $\eta>0$ is not very small (e.g., $\eta=0.1$ for $m=20$),
the level of gained $\varepsilon$-near-orthonormality will become rather mild. 
However, if we apply again the whitening on the previously whitened vectors $\sqrt{\tilde \lambda_i} \widetilde W a_i$ (for fixed accuracy $\eta>0$),
we surprisingly gain further improved $\varepsilon$-near-orthonormality! We implement this bootstrap procedure in Algorithm \ref{alg5}
and we show in 
Figure \ref{espnearorth} corresponding numerical results.


\begin{figure}[h!]
  \centering
 \includegraphics[width=0.5\textwidth]{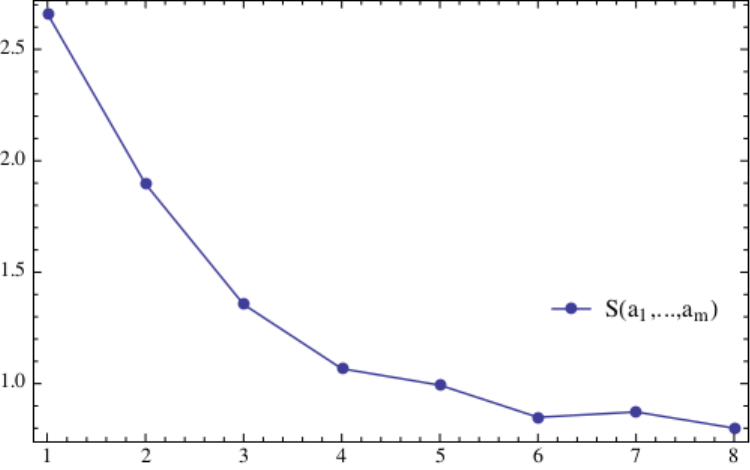}  
\caption{The values of $\mathcal S(a_1^{(k)},\dots,a_m^{(k)})$ for $m=20$ for different iterations $k=1,\dots,8$ of Algorithm \ref{alg5} for fixed $\eta=0.1$. We  observe the improved level of $\varepsilon$-near-orthonormality of the system along the iterations of Algorithm \ref{alg5}. Starting with $\mathcal S(a_1^{(0)},\dots,a_m^{(0)}) \geq \mathcal S(a_1^{(1)},\dots,a_m^{(1)})>1$, one finally obtains $\mathcal S(a_1^{(6)},\dots,a_m^{(6)})<1$ after $k=6$ iterations.}
\label{espnearorth}
\end{figure}

\begin{rem} While high accuracy of the approximation $\widetilde \A \approx \A$ is crucial, as we just pointed out, it is important to stress that in
\rechange{the analysis of Section \ref{locmaxima} the $\varepsilon$-near-orthonormality is merely an useful and technical assumption}
in order to derive in a relatively simple way theoretical error bounds on the identification of the $a_i$'s, see Theorem \ref{thm:recoveridge} and Theorem \ref{thm:recovery2}.
\rechange{Although these bounds provide robust theoretical guarantees, they seem to overestimate
the actual error committed by the reconstruction Algorithms \ref{alg6} or \ref{alg7}
if the vectors $a_i$'s are well-separated (perhaps after just a rough whitening) but not yet $\varepsilon$-near-orthonormal for $\varepsilon$ small. 
So in practice we often observe that}, as soon as the accuracy
$\widetilde \A \approx \A$ is good enough, even a simple whitening is already sufficient for our recovery algorithms Algorithm \ref{alg6} or
Algorithm \ref{alg7} to recover well the $a_i$'s.
We extensively illustrate this  practical efficiency and robustness in the numerical experiments
\rechange{in the next section}.
\end{rem}
}

{
 
\section{Numerical experiments}\label{numerics}


In this section we demonstrate numerically the efficiency of the pipeline of the algorithms we analyzed in the previous sections for the robust
identification of the weights of a network  written in compact form as $f(x) = b^T g(A^T x + \theta)$ of $m$ nodes in dimension $m$,
where $A\in \mathbb{R}^{m \times m}, b,\theta \in \mathbb{R}^m$. Throughout the experiments we use networks with random configurations. 
To generate a random network we specify $m$ and $\varepsilon$, which is the degree of near orthonormality of $A$ and sample the remaining parameters as
\begin{itemize}
    \item $A$: First we draw an orthogonal $m \times m$ matrix from the Haar distribution.
    Then we proceed iterating the following adjustments: we perturb the singular values by Gaussian noise with a given variance, then normalize the columns, and we check whether $A$ has reached the desired  degree $\varepsilon$ of near-orthonormality up to a tolerance of $\pm 0.001$; if this does not occur, then we modify the variance of the Gaussian noise in a bisection fashion to search iteratively the right level of near-orthonormality;
    \item $b=(b_1,\dots,b_m)^T$ with $b_i\sim\mathcal{N}(1, 1/5)$ selected independently;
    \item $\theta=(\theta_1,\dots,\theta_m)$ with $\theta_i \sim \mathcal{N}(0, 1/5)$ selected independently;
    \item $g(t) = \tanh(t)$.
\end{itemize}
This setting corresponds to \eqref{sumsridge} with $g_i(t)=b_i\tanh(t+\theta_i)$.

\subsection{Exact weight identification}
We will consider below trials for $m=20$. The choice of a small dimension is simply due to the necessity of running in a reasonable time a large number of trials to estimate the empirical probability of success, but the algorithms can comfortably be implemented in higher dimensions
$m \approx 10^3$ on a Laptop. In this case the memory needed for storing 64 bit floating point matrices  $Y$ as appearing in Algorithm \ref{alg1}-Algorithm \ref{alg3} is given by
$m^2\times m_{\mathcal X} \times 8$ bytes $\approx 8$GB). For much higher dimensions the use of HPC is needed, see Section \ref{sec:comptime} and Section \ref{sec:memory} below.

Denote by $a_1, \dots, a_m$ the columns of $A$. 
As clarified by Theorem \ref{identact} the fundamental issue is in fact the robust identification of the network weights $a_i$,
while the identification of the rest of the network is its direct consequence.
For the sake of simplicity, we present here results based on active sampling. Accordingly, we denote with $m_{\mathcal X}$ the number of sampled Hessians
of the function $f$, which are computed by finite difference approximations with stepsize $0.001$, cf. \eqref{constrY}.
For each pair $(m_{\mathcal X}, \varepsilon)$ of number of Hessians and near-orthonormality level, we run $60$ trials. In each of the trials we
first construct the subspace $\tilde{\mathcal A}$ by Algorithm \ref{alg3} and then try
to compute all $m = 20$ vectors $a_1, \dots, a_m$ by applying Algorithm \ref{alg7} repeatedly.
One run of Algorithm \ref{alg3} returns at random one of the $a_i$, therefore we need to run the algorithm at least $m$ times to have a chance of recovering all the vectors.
For the hyperparameter of number of repetitions we choose $n_{\text{rep}}=180$. In each of the $n_{\text{rep}}$ repetitions we carried out $100$ steps of the algorithm with $\gamma = 2$ and we used this number of steps as a stopping criterion.
For each of the $60$ trials we get $180$ vectors $V = \left\lbrace v_1, \dots, v_{\text{rep}}\right\rbrace$. In our numerical experiments $V$ always contained only approximations to (some or all) original vectors $a_i$ and no spurious cases seem ever occurring.
For a given tolerance $\delta$ we measure the number of well-approximated vectors as 
\begin{align*}
n_{\text{found}} := \# \left\lbrace i \in [m] :  \min\{ \|v + a_i \|_2 ,\|v - a_i \|_2\} \leq \delta \text{ for any } v \in V \right\rbrace,
\end{align*}
and we set $\delta = 0.05$ in our experiments. 
\begin{figure}[!b]
\begin{center}
\includegraphics[width=1.2\textwidth]{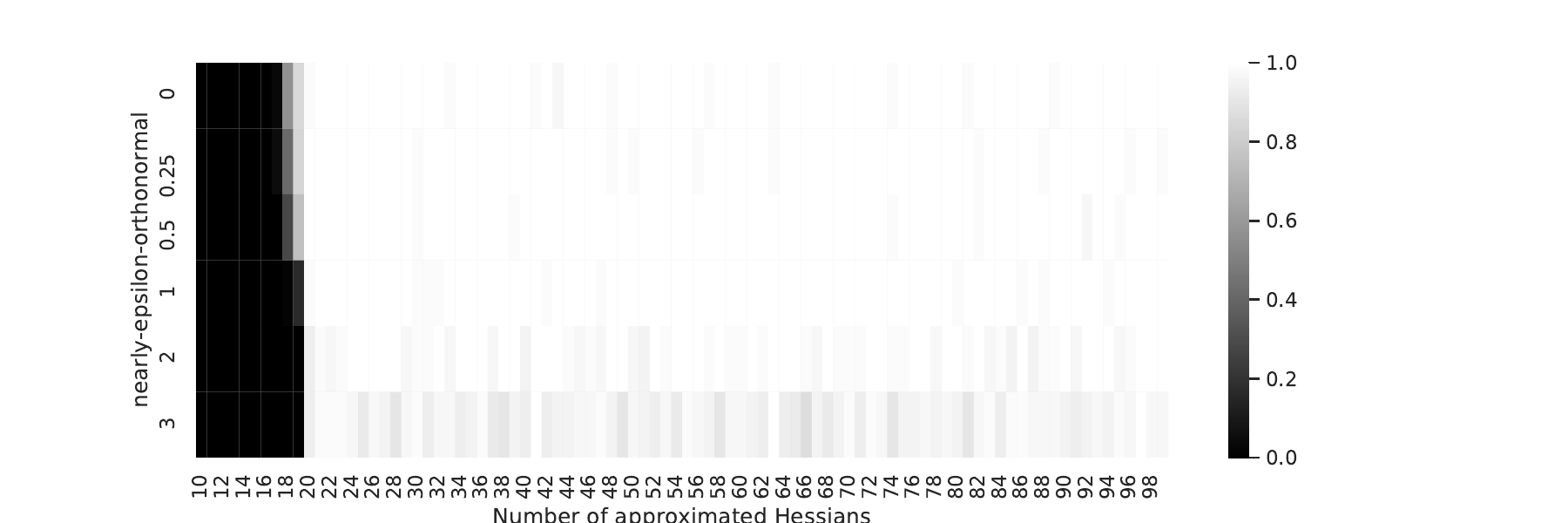}
\end{center}
\caption{Phase transition diagram displaying the empirical success probability of recovering all the vectors $a_1, \dots, a_m$ for a given number
$m_{\mathcal X}$ of sampled Hessians and near-orthonormality level $\varepsilon$ of the searched vectors.}
\label{fig:plot1}
\end{figure}
Figure \ref{fig:plot1} shows the ratio of the number of trials where $n_{\text{found}} = m$ for different degrees of near-orthonormality
and amounts of approximated Hessians. 
From the figure, it is clear that our pipeline of the algorithms is able to reconstruct exactly all $m=20$ network weights $a_1, \dots, a_m$ with high probability with a number $m_{\mathcal{X}}\approx m=20$ of sampled Hessians,
which is the information theoretical lower bound. Moreover, as shown in Figure \ref{fig:plot1}, exact recovery is obtained also for levels $\varepsilon>1$ of near-orthonormality. Even for $\varepsilon \approx 3$
the algorithm recovers all vectors with overwhelming (empirical) probability. This numerical evidence suggests that the $\varepsilon$-nearly-orthonormality is a quite conservative measure of separation of the vectors
$a_1, \dots, a_m$ and that in practice it is sufficient that they are enough separated and a near orthonormality is not necessary.

\subsection{Exploring computational time}\label{sec:comptime}

The previous section indicates that having around $m_{\mathcal{X}}\approx m$ (approximated) Hessians is sufficient for recovery.
We use the parameters above and explore how much computational time different blocks of the algorithm pipeline need for increasing dimensionalities $m=m_{\mathcal{X}}$. Additionally, we keep the deviation from an 
orthonormal system constant at $\varepsilon = 1$ (which should not have much influence on the runtime anyway).
We split the algorithm into three stages: the sampling and approximation of the Hessians, the realization of Algorithm \ref{alg3} (which below we denote PCA), and finally the search of the rank-1 matrices by Algorithm \ref{alg7}. For the last stage we only measure the time needed to find one rank-1 matrix. 
Additionally we also track one version where we replace the SVD in stage 2 (Algorithm \ref{alg3}) by a randomized SVD  \cite{hamatr11}. The computational times of the different phases are plotted in Figure \ref{fig:plotTimeSingle}.
As a clear disclaimer, let us stress that our computational time figures are relative to quite straightforward off the books implementations, with no particular tuning or optimization whatsoever. Hence, they should not be taken as a reference of the absolute performance of our algorithmic pipeline, rather as an illustration of our theoretical findings and an indication of feasibility, even with relatively modest computational resources. In fact, we expect that careful and optimized coding and parallelization will yield tremendous speed-ups and more efficient memory usage over the presented figures.

As mentioned above, we need to apply Algorithm \ref{alg7} repeatedly to find all the weights $a_i$.  An interesting question is how many times $n_{rep}$
do we need to repeat  Algorithm \ref{alg7} with random initial iteration to be able to compute all the weights. 
If we assume a uniform distribution for Algorithm \ref{alg7} to compute at random one of the weights, then our problem is equivalent to the classical
coupon collector's problem. A well known result from probability \cite[Section 8.4]{is95} tells us that we need on average 
\begin{align}\label{norep}
    \mathbb{E}[n_{rep}] = m \ln m + \gamma m + \frac{1}{2} + \mathcal{O}(1 / m)
\end{align}
repetitions to cover all vectors, where $\gamma \approx 0.57721\dots$ is the Euler-Mascheroni constant. Large deviation bounds are also available, see, e.g., \cite{erre61}.
Using this as a baseline to measure the 
cost of the algorithm yields the cumulative results in Figure \ref{fig:plotAccumlative}.
It is important to notice that the showed runtime as well as the memory consumption (next section) 
depend heavily on the implementation, which was not extensively optimized with respect to both. 
However, even with this in mind, it is painfully obvious that finding a way to avoid Algorithm \ref{alg7} picking duplicate vectors would make the algorithm much more
efficient. Nevertheless, it is also clear that the procedure can be easily parallelized, using ca. $\mathbb{E}[N_{rep}]$  \eqref{norep} processors.
\begin{figure}[!hbt]
    \includegraphics[width=1\textwidth]{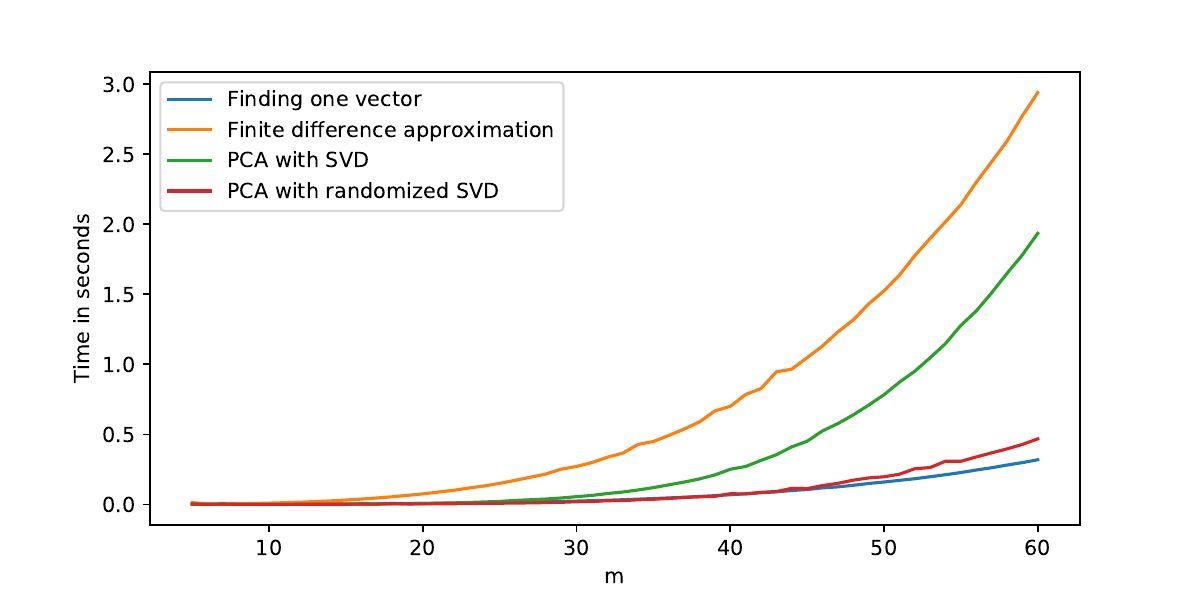}
    \caption{Time spent at the different stages of the pipeline to compute one network weight.}
    \label{fig:plotTimeSingle}
\end{figure}

\begin{figure}[!hbt]
    \includegraphics[width=1\textwidth]{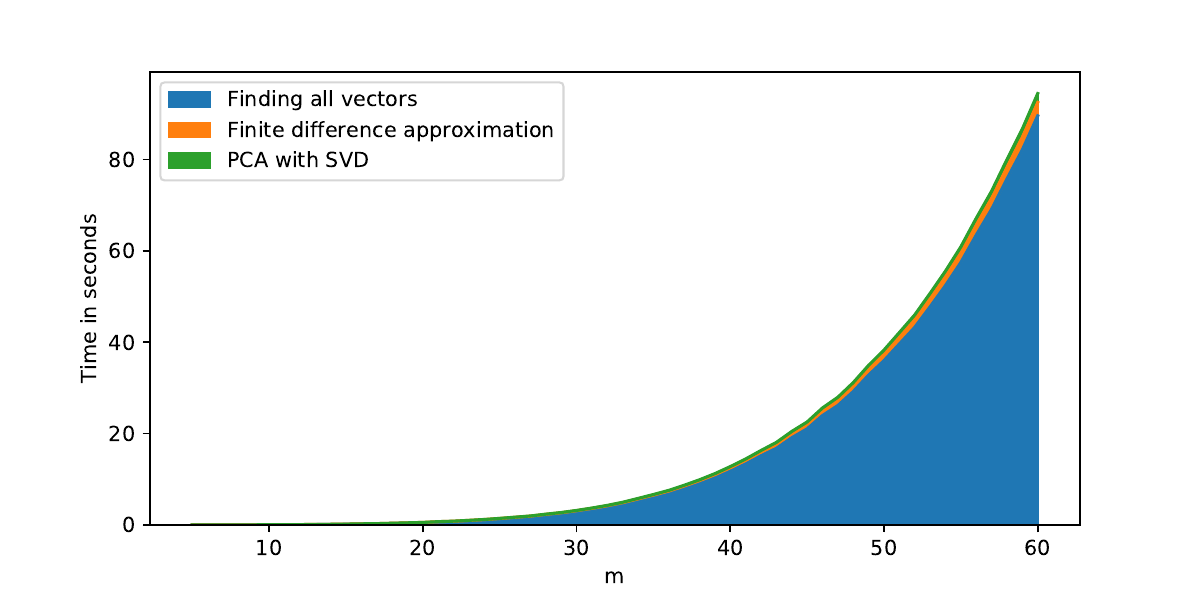}
    \caption{Cumulative time spent at the different stages of the pipeline to obtain full recovery of all vectors in a sequential implementation.}
    \label{fig:plotAccumlative}
    
\end{figure}
\subsection{Exploring memory allocation}\label{sec:memory}
As previously, we split the algorithm into three stages and measure the
 consumed memory of the process for $m=100$ and $m=125$ 
 (cf. Figure \ref{fig:memcon100}-\ref{fig:memcon125}). 
 The most expensive part is clearly the PCA (Algorithm \ref{alg3}). 
 However, choosing a randomized SVD variant required only 50\% of the memory in our example, without diminishing significantly accuracy.
 This stage needs to be further optimized with respect to memory in the future, including considering lower bit encoding etc. 
\begin{figure}[!hbt]
\centering
    \includegraphics[width=1\textwidth]{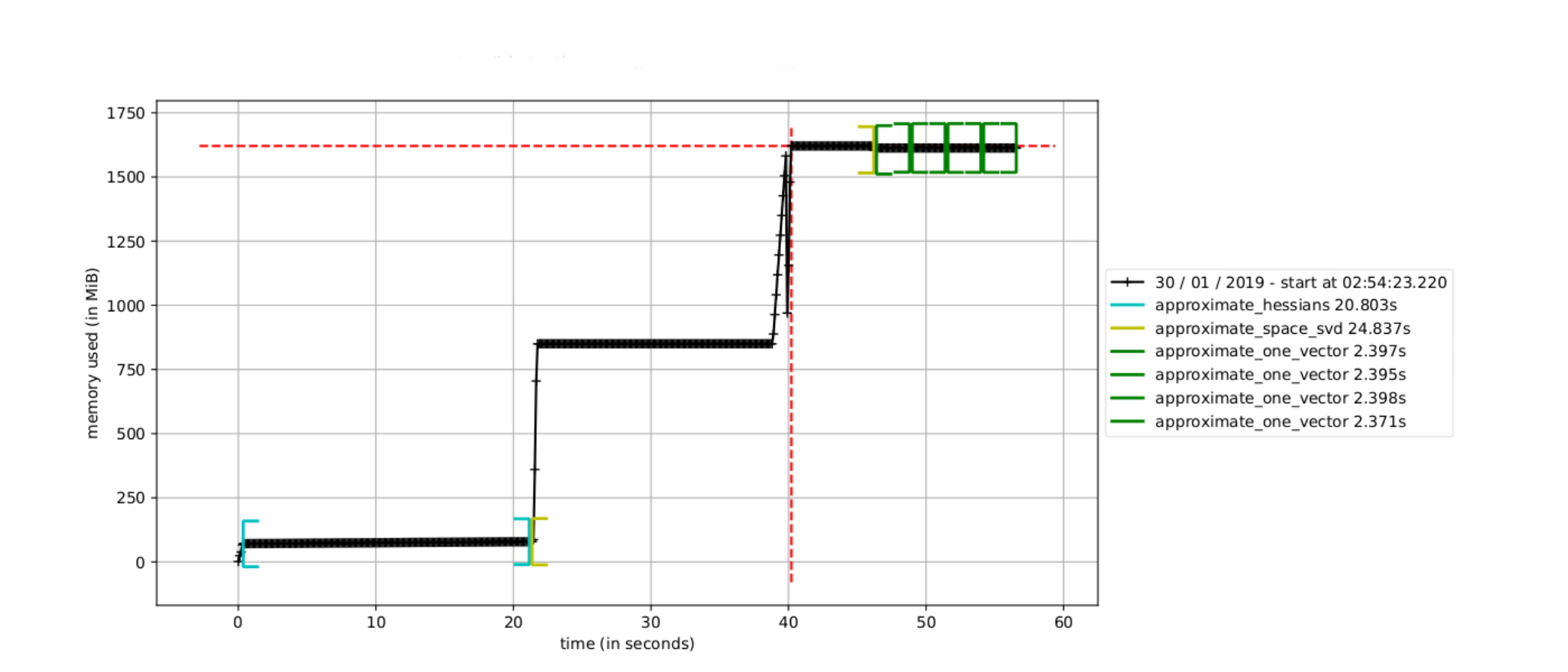}
    \includegraphics[width=1\textwidth]{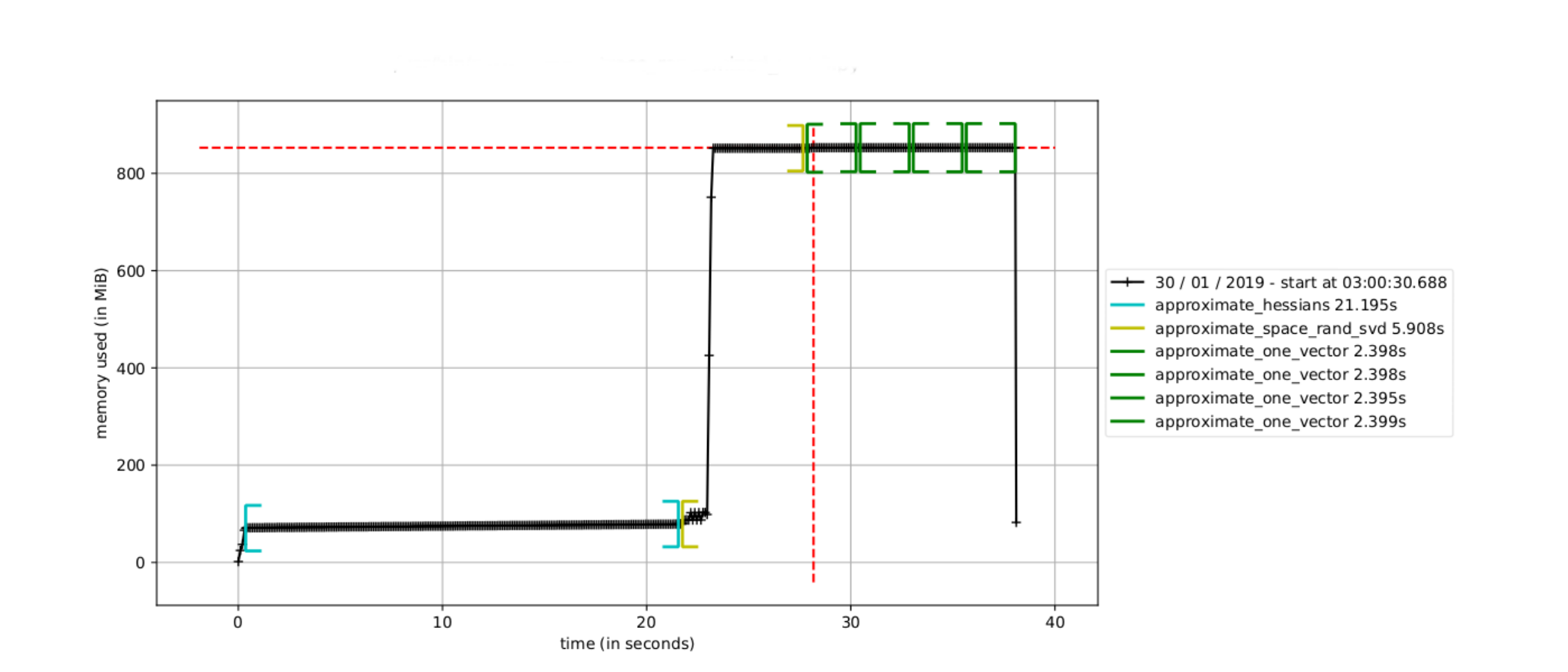}
    \caption{Memory consumption for $m=100$. Classical SVD (top), randomized SVD (bottom). The dashed red lines indicates the cartesian coordinates (time $\times$ memory) of the memory peak.}
    \label{fig:memcon100}
    
\end{figure}
\begin{figure}[!hbt]
\centering
    \includegraphics[width=1\textwidth]{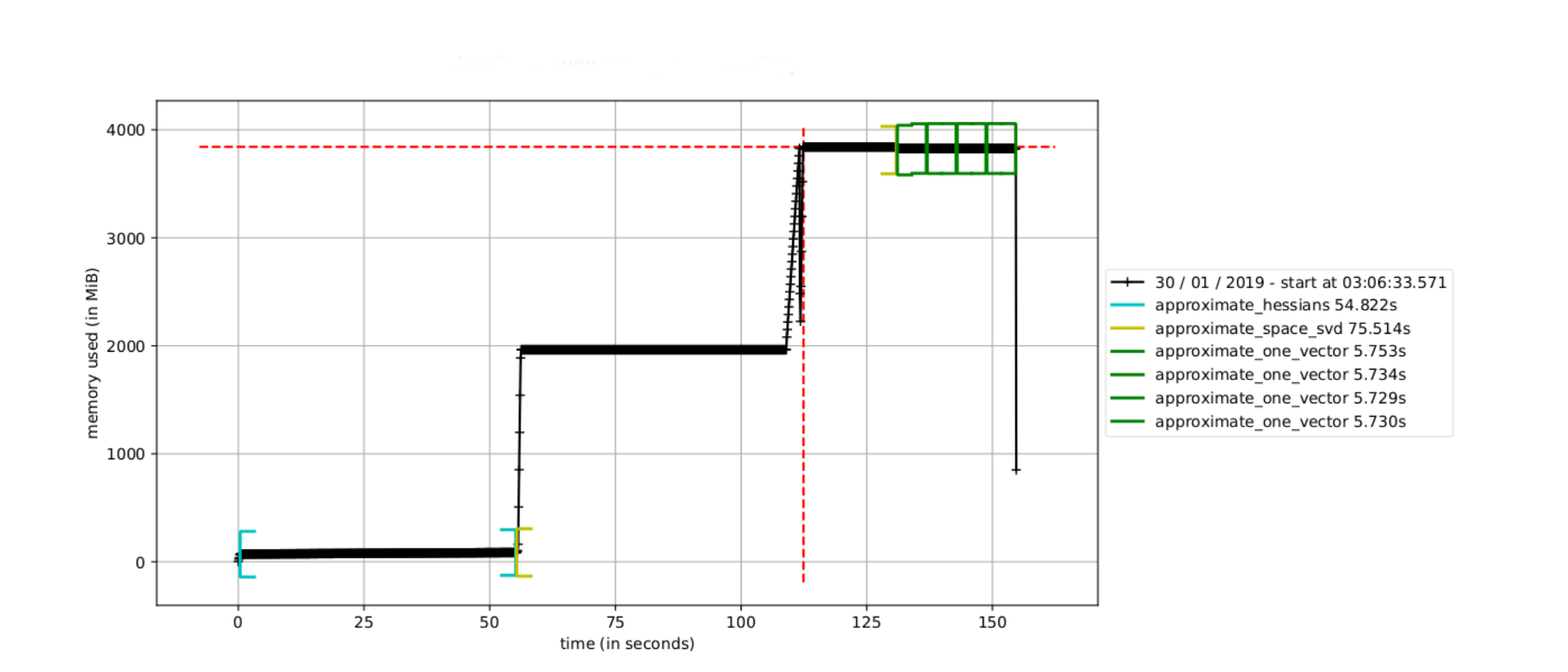}
    \includegraphics[width=1\textwidth]{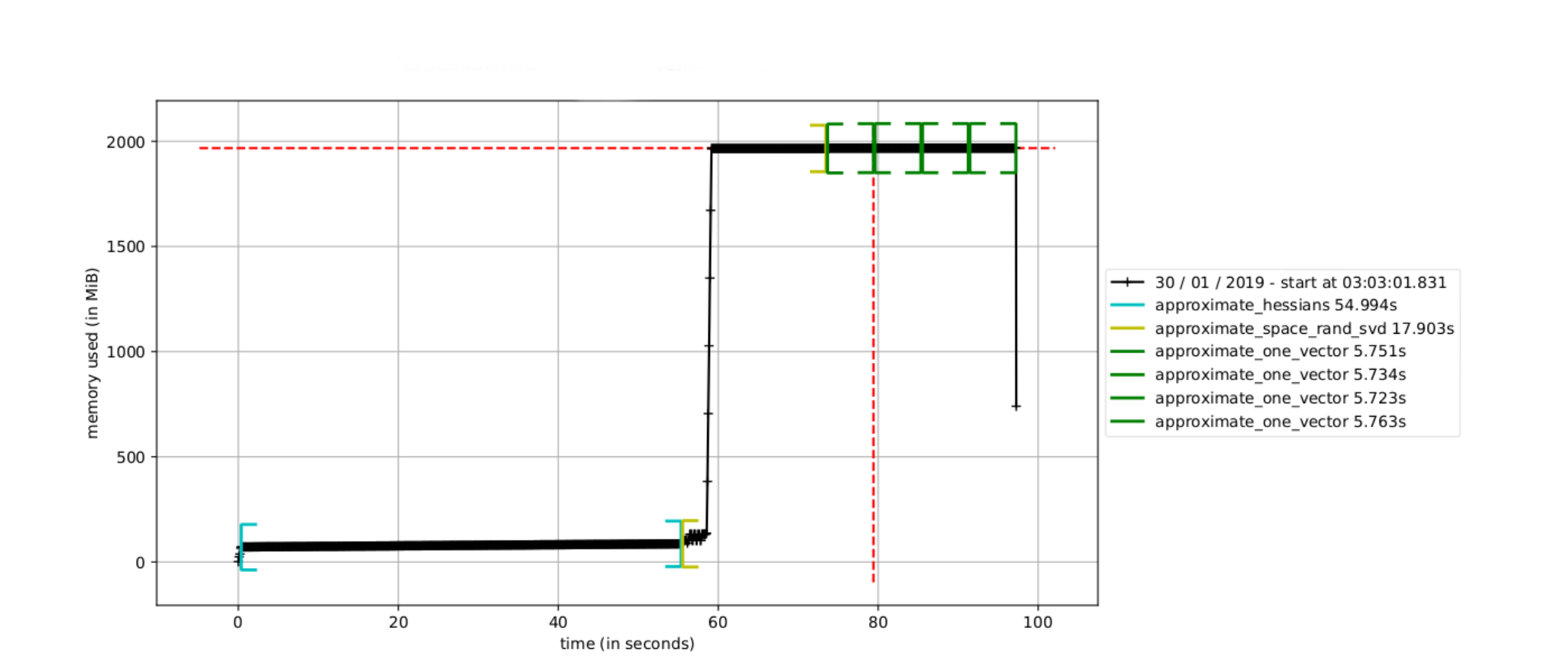}
    \caption{Memory consumption for $m=125$. Classical SVD (top), randomized SVD (bottom). The dashed red lines indicates the cartesian coordinates (time $\times$ memory) of the memory peak.}
    \label{fig:memcon125}
    
\end{figure}
\newpage
\subsection{Comparison to gradient descent}
A popular method for training a neural network is minimizing the misfit (here we consider the mean squared error) on the output of the network over a large number of inputs by means of gradient descent. In this section we would like to compare the behavior of our pipeline of algorithms with gradient descent we choose the same setting for both approaches. In particular, we assume that every other parameter of the network except the (inner) weights $a_1, \dots, a_m$ is known and fixed. To be precise:
\begin{itemize}
\item We assume $b, \theta$ are known and fixed;
\item We set $m = 20$ and $\varepsilon = 1$ (i.e. $a_1, \dots, a_m$ will be $\varepsilon$-nearly-orthonormal);
\item For each random network we run our method for gradually increasing number of sampled approximated Hessians and with $n_{rep}$ sufficiently large. 
From the resulting vectors we compute $\hat a_1, \dots, \hat a_m$ by clustering them by $k$-means. 
Finally we record the error of the new network where $A$ is substituted with our estimate $\hat A$ on a completely new set of $10^5$ datapoints generated at random uniformly on $B_1^d$ that were not used during the previous steps. The error is measured via mean square error (MSE/$L^2$-squared) / uniform norm and we 
record the distance of the estimated $\hat A$ to $A$ in the Frobenius norm.
\item We emulate the same procedure by using gradient descent to minimize the MSE misfit over the same training datapoints.
 Approximating one Hessian by finite-differences requires $1+\frac{d^2 + d}{2}$ samples. The stepsize of gradient descent method remains fixed at $0.1$ and we do exactly $1000$ steps for each trial.
\end{itemize}
We average everything over $10$ random trial networks, the results are collected in Figure \ref{fig:ApproximationError} 
and Figure \ref{fig:LastFigure}.

On the one hand, gradient descent seems to require only a very small amount of samples ($\approx 50$, see the right plot in Figure \ref{fig:LastFigure}) to converge against a very efficient local minimum representing a good approximation of the network in MSE, whereas our algorithm requires at least $m$ Hessians. However, the generated network by gradient descent is not optimal as one can observe from the non-vanishing uniform norm approximation, see the right plot in Figure \ref{fig:ApproximationError}.
On the other hand, if we have enough Hessians, then our method returns the optimal network with vanishing uniform norm discrepancy. Additionally, gradient descent will never come close to the original weight matrix $A$, while our algorithm is able to consistently recover a very good approximation of $A$, see the left plot in Figure \ref{fig:LastFigure}.
\begin{figure}[!hbt]
\centering
    \includegraphics[width=.45\textwidth]{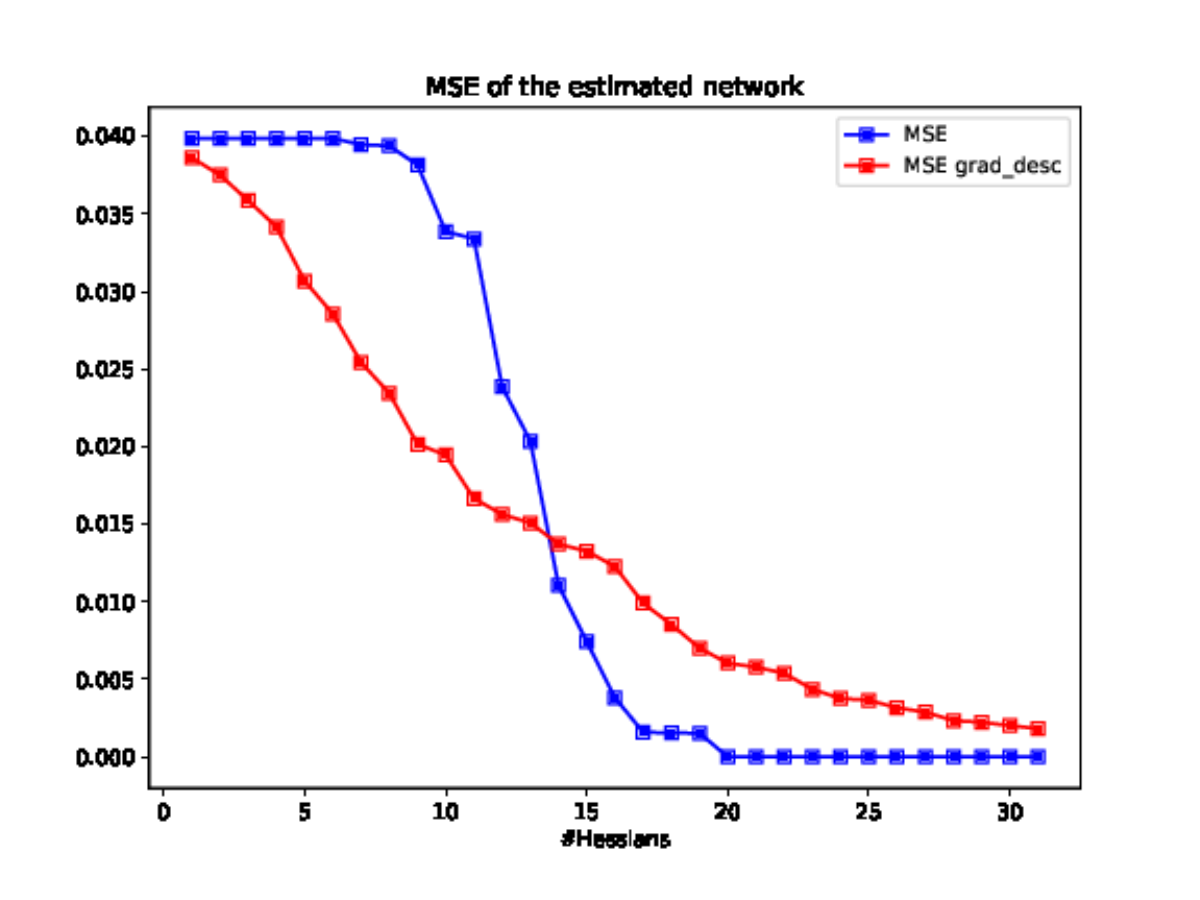}
    \includegraphics[width=.45\textwidth]{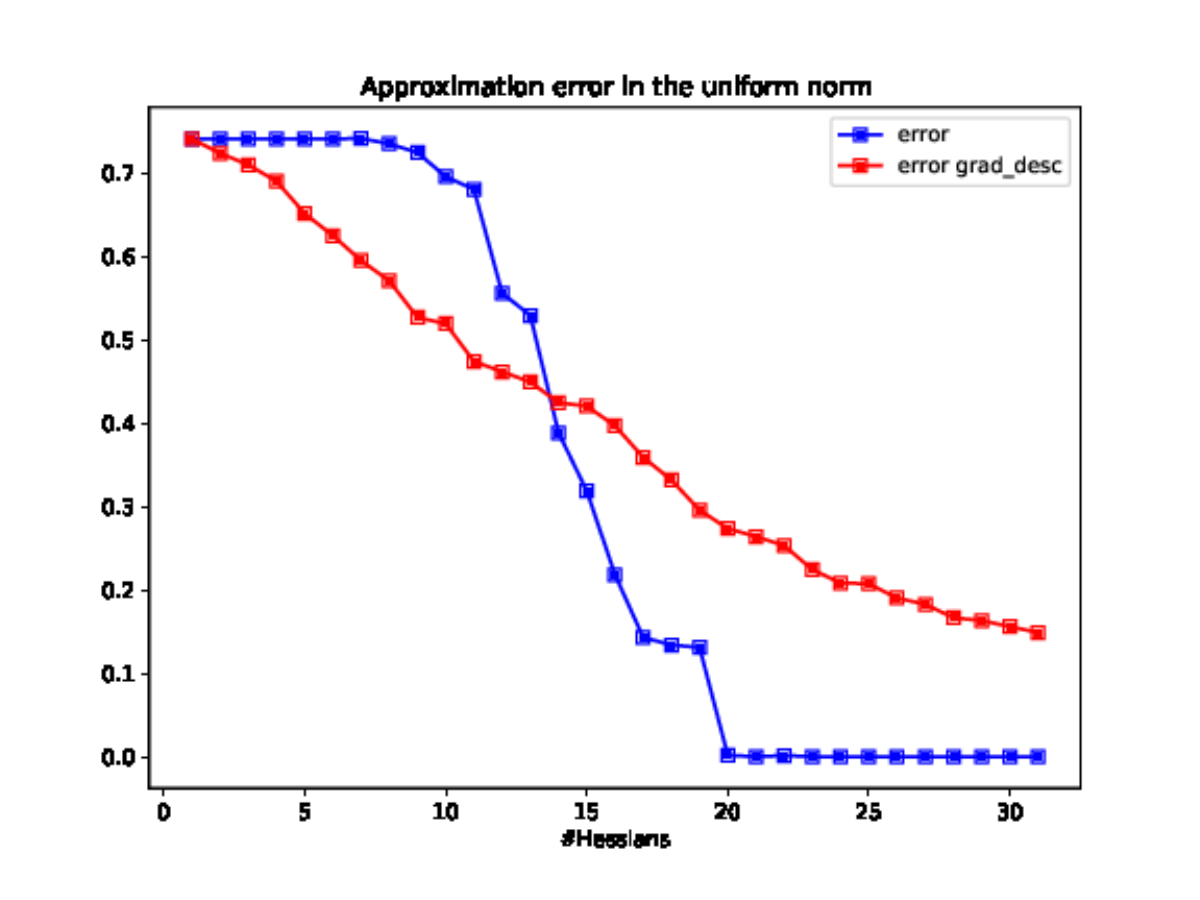}
    \caption{Average approximation error of $10$ random networks with $m=20, \varepsilon=1$ in terms of MSE (left), uniform norm (right). The errors were measured over $10^5$ datapoints generated uniform at random on the ball $B_1^d$.}
    \label{fig:ApproximationError}
\end{figure}
\begin{figure}[!hbt]
\centering
    \includegraphics[width=.45\textwidth]{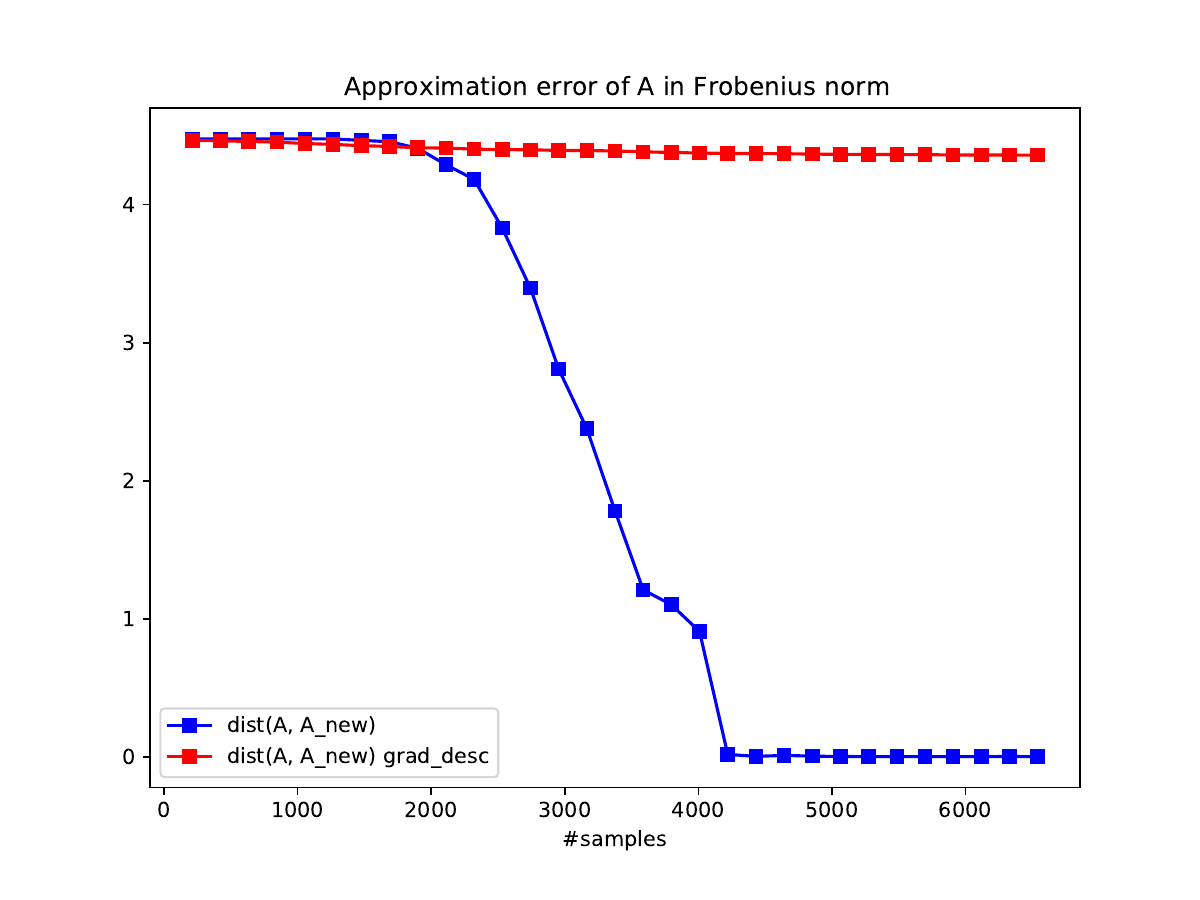}
    \includegraphics[width=.45\textwidth]{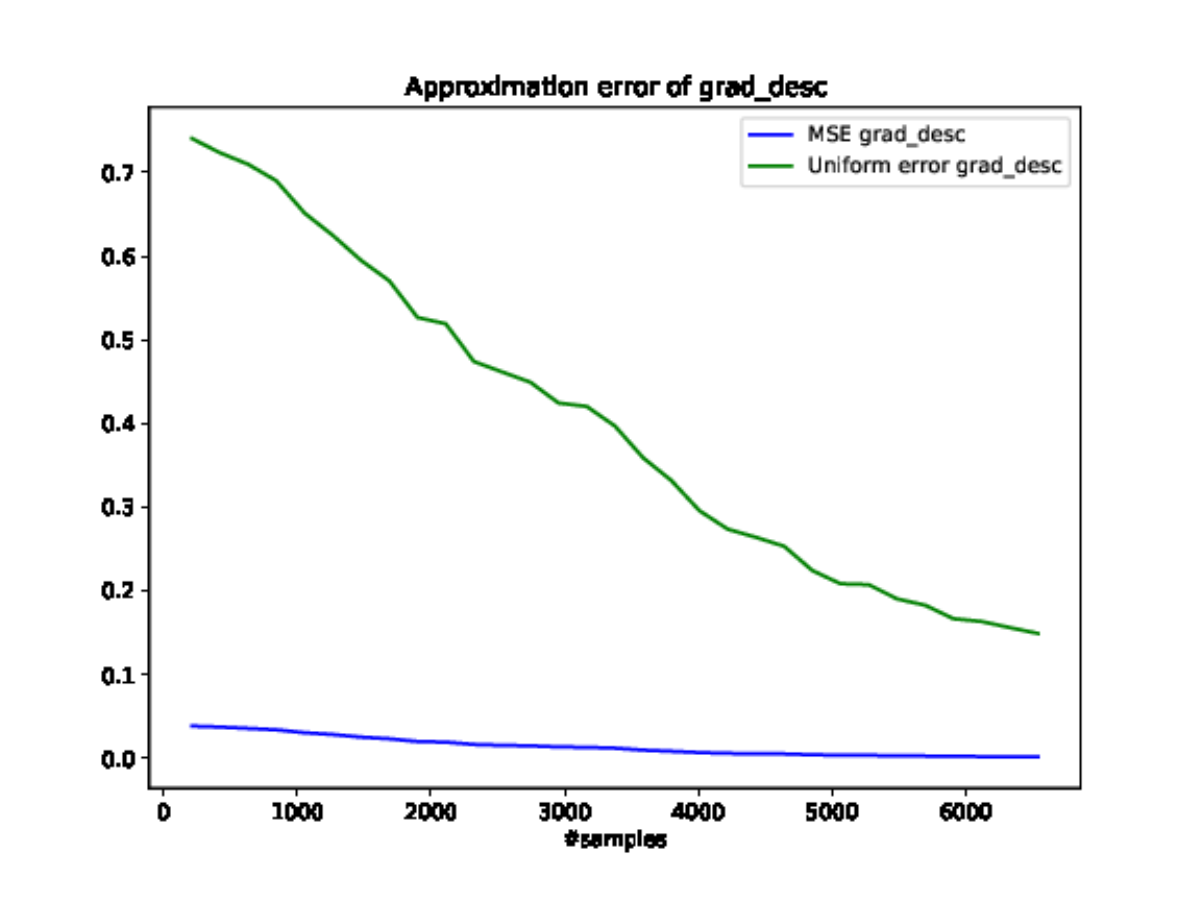}
    \caption{Approximation error of $A$ in the Frobenius norm (left). Approximation error of gradient descent for 1-200 samples (right).}
    \label{fig:LastFigure}
\end{figure}


}

\section{Appendix}

\subsection{Proofs of passive sampling results}

 With the same proof of Lemma \ref{speclem}  we can show the following result.
\begin{lem}\label{speclem2} Assume the vectors $(a_i)_{i=1}^m$ linearly independent, and $\|a_i\|_2=1$ for all $i=1,\dots,m$. Additionally assume 
\begin{align*}
C_1&:=\max_{i=1,\dots,m}\, \max_{-1\le t\le 1} |g'_i(t)| < \infty.
\end{align*}
Suppose that $\sigma_m(J_{\mathcal V}[f])\ge \alpha>0$, i.e., the $m^{th}$ singular value of the matrix $J_{\mathcal V}[f]$ is bounded away from zero. Then for any $s\in (0, 1)$  we have that
\begin{equation*}
\sigma_m(X_{\mathcal{V}}) \geq \sqrt{m_{\mathcal X} \alpha (1-s)}
\end{equation*}
with probability at least $1 - m \exp\Bigl(-\frac{m_{\mathcal X}\alpha s^2  }{2 (C_0 C_{\mathcal V} m)^2}\Bigr)$.
\end{lem}
\begin{proof}
The only difference with respect to the proof of Lemma \ref{speclem} is in how we estimate the term
\begin{align}
& \notag \sigma_1 \left (W^T\left ( \int_{\mathbb R^d} \nabla f(x) \varphi_{\nu}(x) dx \right )\left ( \int_{\mathbb R^d} \nabla f(x) \varphi_{\nu}(x) dx \right )^TW \right )=\left  \|  \int_{\mathbb R^d} \nabla f(x) \varphi_{\nu}(x) dx \right \|_2^2\\ 
\label{eq:sigma_1} 
&=\left  \|  \int_{\mathbb R^d} f(x) \nabla \varphi_{\nu}(x) dx \right \|_2^2\le \biggl(\int_{\mathbb R^d}|f(x)|\cdot \left\|\frac{\nabla\varphi_{\nu}(x)}{p(x)}\right\|_2p(x)dx\biggr)^2
\le (mC_0C_{\mathcal V})^2.
\end{align}
The rest follows similarly.
\end{proof}

{\it Proof of Theorem \ref{thm3-1}.} 
As in the proof of Theorem \ref{thm3}
\begin{align}
\|P_A-P_{\tilde A}\|_F&\le \frac{2\|X_{\mathcal{V}}-Y_{\mathcal{V}}\|_F}{\sigma_{m}(X_{\mathcal{V}}^T)-\|X_{\mathcal{V}}-Y_{\mathcal{V}}\|_F}. \label{1steq}
\end{align}
By Lemma \ref{speclem2} with probability at least $1 - m \exp\Bigl(-\frac{m_{\mathcal X}\alpha s^2  }{2 (C_0 C_{\mathcal V} m)^2}\Bigr)$ we have 
\begin{equation}\label{2ndeq}
\sigma_m(X_{\mathcal{V}}) \geq \sqrt{m_{\mathcal X} \alpha (1-s)}. 
\end{equation}
\change{
In order to estimate $\|X_{\mathcal V}-Y_{\mathcal V}\|_F$ we apply Theorem \ref{thm:vecber}. Let $Z^k\in\R^{d\times m_{\mathcal X}}$ be a matrix with columns
$$
(Z^k)_j=\frac{f(x_k) +n_k}{N}\cdot \frac{\nabla \varphi_{\nu_j}(x_k)}{p(x_k)}- \frac{1}{N}\int_{\mathbb R^d} f(x) \frac{\nabla \varphi_{\nu_j}(x)}{p(x)} p(x) dx
$$
for $j=1,\dots,m_{\mathcal X}.$ Then ${\mathbb E}Z^k=0$,
\begin{align*}
\|Z^k\|_F^2=\sum_{j=1}^{m_{\mathcal X}}\|(Z^k)_j\|_2^2\le \frac{m_{\mathcal X}}{N^2}\cdot[2(mC_0+C_{\mathcal N})C_{\mathcal V}]^2=\frac{m_{\mathcal X}}{N^2}\cdot(2mQ)^2=:B^2
\end{align*}
and
\begin{align*}
\sum_{k=1}^N &{\mathbb E}\|Z^k\|_F^2=\sum_{k=1}^N \sum_{j=1}^{m_{\mathcal X}} {\mathbb E}\|(Z^k)_j\|_2^2\\
&=N m_{\mathcal X}\int_{\mathbb R} \int_{\mathbb R^d}\left \| \frac{f(y)+n}{N}\cdot \frac{\nabla \varphi_{\nu}(y)}{p(y)}- \frac{1}{N}\int_{\mathbb R^d} f(x) \frac{\nabla \varphi_{\nu}(x)}{p(x)} p(x) dx\right  \|_2^2 p(y) dy d\omega(n)\\
&\le \frac{m_{\mathcal X}[2(mC_0+C_{\mathcal N})C_{\mathcal V}]^2}{N}=\frac{m_{\mathcal X}(2mQ)^2}{N}=:\sigma^2,
\end{align*}
where $\omega$ is the probability distribution of the noise.

By Theorem \ref{thm:vecber}, we obtain for $t=\eta \sqrt{m_{\mathcal X}}$ with $0\le \eta\le 2mQ$
\begin{align}
\notag {\mathbb P}(\|X_{\mathcal V}-Y_{\mathcal V}\|_F\ge \eta \sqrt{m_{\mathcal X}})&={\mathbb P}\Bigl(\Bigl\|\sum_{k=1}^N Z^k\Bigr\|_F\ge \eta \sqrt{m_{\mathcal X}}\Bigr)\le \exp\Bigl(-\frac{m_{\mathcal X}\eta^2}{8\sigma^2}+\frac{1}{4}\Bigr)\\
\label{3rdeq'} &\le \exp\Bigl(-\frac{\eta^2N}{8(2mQ)^2}+\frac{1}{4}\Bigr).
\end{align}
The proof of \eqref{eq:estnorm1} now follows by combining \eqref{1steq}, \eqref{2ndeq} and \eqref{3rdeq'} together with
$$
\eta<\sqrt{\alpha(1-s)}\le \sqrt{\alpha}\le\sqrt{\sigma_m(J_{\mathcal V}[f])}\le\sqrt{\sigma_1(J_{\mathcal V}[f])}\le mC_0C_{\mathcal V}\le mQ,
$$
where we used also \eqref{eq:sigma_1}.}

In order to show \eqref{simplest1} and \eqref{simplest2}, let us fix $\eta>0$ such that 
$\varepsilon = \frac{2 \eta}{\sqrt{\alpha(1-s)}-\eta}$, i.e., $\eta = \frac{\varepsilon \sqrt{\alpha(1-s)}}{2+\varepsilon}$. 
We now choose $m_{\mathcal{X}}$ large enough to ensure that
\begin{equation*}
\delta/2 \geq m \exp\Bigl(-\frac{m_{\mathcal X}\alpha s^2  }{2 (mQ)^2}\Bigr),\quad \text{i.e.}\quad m_{\mathcal X} \geq \frac{2 (mQ)^2 \ln(2m/\delta)}{\alpha s^2},
\end{equation*}
and $N$ large enough for 
\begin{equation*}
\delta/2 \geq \exp\Bigl(-\frac{\eta^2N}{8(2mQ)^2}+\frac{1}{4}\Bigr).
\end{equation*}
We observe that for
\change{\begin{equation*}
N \geq \frac{32(2+\varepsilon)^2(mQ)^2\ln (3/\delta)}{\varepsilon^2\alpha (1-s)} 
\end{equation*}}
we can ensure $\|P_A-P_{\tilde A}\|_F \leq \varepsilon,$
with probability at least $ 1-  \delta$.
\qed

{\it Proof of Theorem \ref{thm:projapp2}.}
By Wedin's bound, we obtain as in the proof of Theorem \ref{thm3-1}
\begin{align}
\|P_\A-P_{\widetilde \A}\|_{F\to F}&\le \frac{2\|X_{\mathcal{V},2}-Y_{\mathcal{V},2}\|_F}{\sigma_{m}(X_{\mathcal{V},2}^T)-\|X_{\mathcal{V},2}-Y_{\mathcal{V},2}\|_F}. \label{1steq-1}
\end{align}
The analogue of \eqref{eq:est_rem} and \eqref{eq:sigma_1} now reads as
\begin{align}
& \label{eq:refX1} \sigma_1 \left (P^{\mathcal A}\left ( \int_{\mathbb R^d} f(x) \frac{\nabla^2 \varphi_{\nu}(x)}{p(x)}p(x) dx \right )\otimes_v
\left ( \int_{\mathbb R^d} f(x) \frac{\nabla^2 \varphi_{\nu}(x)}{p(x)}p(x)dx \right )^T(P^{\mathcal A})^T \right )\\
&\notag=\left  \|  \int_{\mathbb R^d} \nabla^2 f(x) \varphi_{\nu}(x) dx \right \|_F^2=\left  \|  \sum_{i=1}^m \left (\int_{\mathbb R^d} g_i''(\langle a_i,x\rangle) \varphi_{\nu}(x) dx\right)  a_i \otimes a_i \right \|_F^2\le (C_2 C_{\mathcal{V},2}m) ^2.
\end{align}
Then for any $s\in (0, 1)$  we have that
\begin{equation}\label{2ndeq-1}
\sigma_m(X_{\mathcal{V},2}) \geq \sqrt{m_{\mathcal X} \alpha_2 (1-s)}
\end{equation}
with probability at least $1 - m \exp\Bigl(-\frac{m_{\mathcal X}\alpha s^2  }{2 (C_2 C_{{\mathcal V},2} m)^2}\Bigr)$.

\change{
To estimate $\|X_{{\mathcal V},2}-Y_{{\mathcal V},2}\|_F$, we use again Theorem \ref{thm:vecber}. Let $Z^k\in\R^{m^2\times m_{\mathcal X}}$
be a matrix with columns 
$$
(Z^k)_j=\operatorname{vec}\left(\frac{f(x_k) +n_k}{N}\cdot \frac{\nabla^2 \varphi_{\nu_j}(x_k)}{p(x_k)}- \frac{1}{N}\int_{\mathbb R^d} f(x) \frac{\nabla^2 \varphi_{\nu_j}(x)}{p(x)} p(x) dx\right)
$$
for $j=1,\dots,m_{\mathcal X}$. Then ${\mathbb E} Z^k=0$,
$$
\|Z^k\|_F^2\le m_{\mathcal X}[2(mC_0+C_{\mathcal N})C_{{\mathcal V},2}]^2/N^2\le m_{{\mathcal X}}(2mQ)^2/N^2=:B^2
$$
and
\begin{align*}
\sum_{k=1}^N {\mathbb E}\|Z^k\|_F^2&=Nm_{\mathcal X}\int_{\mathbb R} \int_{\mathbb R^d}\left \| \frac{f(y)+n}{N} \frac{\nabla^2 \varphi_{\nu}(y)}{p(y)}-\frac{1}{N} \int_{\mathbb R^d} f(x) \frac{\nabla^2 \varphi_{\nu}(x)}{p(x)} p(x) dx\right  \|_F^2 p(y) dy d\omega(n)\\
&\le\frac{m_{\mathcal X}}{N}[(mC_0+C_{{\mathcal N}})C_{{\mathcal V},2}+mC_0C_{{\mathcal V},2}]^2\le m_{\mathcal X}(2mQ)^2/N=:\sigma^2,
\end{align*}
where $\omega$ is the probability distribution of the noise. By Theorem \ref{thm:vecber}, we obtain
$$
{\mathbb P}\left(\|X_{{\mathcal V},2}-Y_{{\mathcal V},2}\|_F>\eta\sqrt{m_{\mathcal X}}\right)={\mathbb P}\left(\left\|\sum_{k=1}^N Z^k\right\|_F\ge \eta\sqrt{m_{\mathcal X}}\right)\le\exp\left(-\frac{\eta^2N}{8(2mQ)^2}+\frac{1}{4}\right)
$$
for $0<\eta<2mQ$. This condition is satisfied for $\eta<\sqrt{\alpha_2(1-s)}$ by \eqref{eq:refX1}.
The proof of \eqref{eq:estnorm2-1} then follows by combining this estimate with \eqref{1steq-1} and \eqref{2ndeq-1}.

The proof of  \eqref{simplest1-1} and \eqref{simplest2-1} proceeds in the same manner as in the proof of Theorem \ref{thm3-1}.
}
\qed 

\subsection{Stability of the singular value decomposition}

Given two matrices $B$ and $\widetilde B$ with corresponding singular value decompositions
$$
B = \left (\begin{array}{lll}U_1&U_2\end{array}\right )
\left (\begin{array}{ll}\Sigma_1& 0\\0& \Sigma_2\\\end{array}\right )
\left (\begin{array}{l}V_1^T\\V_2^T\end{array}\right )
$$
and
$$
\widetilde B = \left (\begin{array}{lll}\widetilde  U_1& \widetilde U_2\end{array}\right )
\left (\begin{array}{ll}\widetilde \Sigma_1& 0\\0& \widetilde \Sigma_2\\\end{array}\right )
\left (\begin{array}{l}\widetilde V_1^T\\\widetilde  V_2^T\end{array}\right ),
$$
where it is understood that two corresponding submatrices, e.g., $U_1,\widetilde U_1$, have the same size,
we would like to bound the difference between $V_1$ and $\widetilde V_1$ by the error $\| B - \widetilde B\|_F$. 
As a consequence of Wedin's perturbation bound \cite{we72}, see also \cite[Section 7]{st90},
we have the following useful result.
\begin{thm}[Stability of subspaces - Wedin's bound]\label{wedin}
If there is an $\bar \alpha >0$ such that
\begin{equation}\label{separa1}
\min_{\ell,\tilde \ell} | \sigma_{\tilde \ell}(\widetilde \Sigma_1) - \sigma_{ \ell}( \Sigma_2) | \geq \bar \alpha,
\end{equation}
and
\begin{equation}\label{separa2}
\min_{\tilde \ell} | \sigma_{\tilde \ell}(\widetilde \Sigma_1) | \geq \bar \alpha,
\end{equation}
then {
\begin{equation*}
\max \{ \|\sin (\Theta(U_1, \widetilde U_1))\|_p , \|\sin (\Theta(V_1, \widetilde V_1))\|_p \} \leq \frac{1}{\bar \alpha} \| B - \widetilde B\|_p,
\end{equation*}
where $\Theta(V, W)$ is the vector of the principal angles between the subspaces $V$ and $W$ and $\| \cdot \|_p$ is an arbitrary $p$-norm or Schatten-$p$-norm for $1\leq p \leq \infty$. The case of $p=2$ corresponds to the Frobenius norm and the bound further specifies as follows:
\begin{equation*}
\max \{ \|U_1 U_1^T - \widetilde U_1 \widetilde U_1^T\|_F , \|V_1 V_1^T - \widetilde V_1 \widetilde V_1^T\|_F\} \leq \frac{\sqrt 2}{\bar \alpha} \| B - \widetilde B\|_F.
\end{equation*}
}
\end{thm}

\subsection{Spectral estimates and sums of random semidefinite matrices}

The value of $\sigma_m(X^T)$ can be estimated by certain matrix Chernoff bounds.
The following theorem generalizes Hoeffding's inequality to sums of random semidefinite matrices and was recently presented by 
Tropp in \cite[Corollary 5.2 and Remark 5.3]{tr10}, improving over results in \cite{ahwi02}, and using techniques from \cite{ruve07} and \cite{ol10}.
\begin{thm}[Matrix Chernoff]\label{chernmat}
Let $X_1, \dots, X_n$ be independent random, positive-semidefinite matrices of dimension $m \times m$. Moreover suppose that
\begin{equation*}
\sigma_1(X_j) \leq C
\end{equation*}
almost surely for all $j=1,\dots,n$. Let 
\begin{equation*}
\mu_{\min} = \sigma_m \Bigl ( \sum_{j=1}^n \mathbb E X_j\Bigr)
\end{equation*}
be the smallest singular value of the sum of the expectations.
Then
\begin{equation*}
\mathbb P \left \{ \sigma_m\left ( \sum_{j=1}^n X_j\right) - \mu_{\min} \leq -s \mu_{\min} \right \} 
\leq m\, \exp\Bigl(-\frac{ \mu_{\min}s^2}{2C}\Bigr),
\end{equation*}
for all $s \in (0,1)$.
\end{thm}
Recall that for some $d_1\times d_2$-matrix $A$ its spectral norm is defined as $\max(\|AA^T\|,\|A^TA\|)^{1/2}$ (i.e. its largest singular value).
For $d_1\times 1$-matrices (i.e. vectors) this gives simply its $\ell_2$-norm.

\begin{cor}\label{chernmat2}
	Let $X_1,\ldots,X_N$ be independent, mean-zero $d_1\times d_2$-random matrices. Assume that
	\[
		\|X_j\|\leq K
	\]
	almost surely for all $1\leq j\leq N$, and denote
	\[
		\sigma^2=\max\Biggl(\biggl\|\sum_{j=1}^N\mathbb{E}(X_jX_j^T)\biggr\|,
			\biggl\|\sum_{j=1}^N\mathbb{E}(X_j^T X_j)\biggr\|\Biggr)\,.
	\]
	Then it holds
	\[
		\rechange{\mathbb P}\Biggl(\biggl\|\sum_{j=1}^NX_j\biggr\|>\eta\Biggr)
			\leq (d_1+d_2)\exp\Bigl(-\frac{\eta^2}{2(\sigma^2+K\eta/3)}\Bigr)\,.
	\]
\end{cor}
We apply this result for random vectors $Y_\ell=X_\ell-X$, where $X=\mathbb{E}X_\ell$, to estimate $\|\frac{1}{N}\sum_{\ell=1}^N X_\ell-X\|$.

\change{
We shall use also the vector valued analogue of Theorem \ref{chernmat} in the form presented in \cite[Proposition 7]{KG14}, see also \cite[Chapter 6.3]{LT}
and \cite[Chapter 8.9]{FR}.
\begin{thm}[Vector Bernstein inequality]\label{thm:vecber} Let $Z_1,\dots,Z_N$ be independent random vectors in $\R^d$ with ${\mathbb E} Z_k=0$ and
$\|Z_k\|_2\le B$ almost surely for all $k=1,\dots,N.$ If $\displaystyle \sigma^2\ge \sum_{k=1}^N {\mathbb E}\|Y_k\|_2^2$,
then for all $0\le t\le \sigma^2/B$
$$
{\mathbb P}\Bigl(\Bigl\|\sum_{k=1}^N Z_k\Bigr\|_2\ge t\Bigr)\le \exp\Bigl(-\frac{t^2}{8\sigma^2}+\frac{1}{4}\Bigr).
$$
\end{thm}}

\change{
\section*{Acknowledgement} 
We authors wish to thank profusely the anonymous Referee for the suggestions, which
greatly improved both results and presentation of the paper.}

\end{document}